\documentclass{article}

\PassOptionsToPackage{numbers, compress, sort}{natbib}

 \usepackage[main, final]{neurips_2026}

\usepackage{microtype}
\usepackage{graphicx}
\usepackage{subcaption}
\usepackage{booktabs} 

\usepackage{hyperref}

\usepackage{amsmath, amsfonts}
\usepackage{epsfig,amsthm}
\usepackage{mathtools}
\usepackage{bm}
\usepackage[framemethod=TikZ]{mdframed}

\usetikzlibrary{arrows.meta, calc, decorations.pathreplacing}

\definecolor{cObsBg}{HTML}{FFE8E0}
\definecolor{cObsBorder}{HTML}{D9583F}
\newmdenv[
  linewidth=0.5pt,
  linecolor=cObsBorder,
  backgroundcolor=cObsBg,
  roundcorner=2pt,
  innerleftmargin=8pt, innerrightmargin=8pt,
  innertopmargin=3pt, innerbottommargin=3pt,
  skipabove=8pt, skipbelow=-2pt
]{obstructionbox}
\newcounter{obstruction}
\newcommand{\obstruction}[1]{\refstepcounter{obstruction}\textbf{Obstruction \theobstruction\ (#1).}}
\definecolor{cSolBg}{HTML}{E0F0FF}
\definecolor{cSolBorder}{HTML}{418DD8}
\newmdenv[
  linewidth=0.5pt,
  linecolor=cSolBorder,
  backgroundcolor=cSolBg,
  roundcorner=2pt,
  innerleftmargin=8pt, innerrightmargin=8pt,
  innertopmargin=3pt, innerbottommargin=3pt,
  skipabove=8pt, skipbelow=-2pt
]{solutionbox}
\newcounter{solution}
\newcommand{\solution}[1]{\refstepcounter{solution}\textbf{Solution \thesolution\ (#1).}}

\DeclareMathOperator{\diag}{diag}

\theoremstyle{definition} 

\theoremstyle{plain} 
\newtheorem{theorem}{Theorem}[section]
\newtheorem*{theorem*}{Theorem}
\newtheorem*{proposition*}{Proposition}
\newtheorem{lemma}{Lemma}[section]

\newtheorem{proposition}{Proposition}[section]

\usepackage{multirow}
\usepackage{multicol}

\usepackage{algorithm}
\usepackage{algpseudocode}

\usepackage{physics}

\let\originalleft\left
\let\originalright\right
\renewcommand{\left}{\mathopen{}\mathclose\bgroup\originalleft}
\renewcommand{\right}{\aftergroup\egroup\originalright}

\usepackage[titles]{tocloft}
\usepackage{amsmath}
\usepackage{amssymb}
\usepackage{mathtools}
\usepackage{amsthm}

\usepackage[capitalize,noabbrev]{cleveref}

\theoremstyle{plain}
\theoremstyle{definition}

\theoremstyle{remark}

\usepackage[textsize=tiny]{todonotes}

\usepackage[titles]{tocloft}
\title{FILO: A Unified Fisher-Geometric Framework for Low-Rank Adaptation}

\newcommand{\authorskip}{\hspace{2.5mm}}
\newcommand{\affiliationskip}{\hspace{2.5mm}}
\author{%
  Ziyun Liu \authorskip
  Fengmiao Bian$^*$ \authorskip
  Jian-Feng Cai$^*$ \authorskip \vspace{.2em}\\
  Department of Mathematics \\
  The Hong Kong University of Science and Technology \vspace{.2em}\\
  \tt {zliueq@connect.ust.hk} \affiliationskip \{mafmbian,jfcai\}@ust.hk
}

\begin{document}

\maketitle

\addtocontents{toc}{\protect\setcounter{tocdepth}{0}}

\begin{abstract}
Low-Rank Adaptation (LoRA) reparameterizes a weight matrix $\bm{W}$ as a product of two low-rank factors, but the Jacobian $J_\mathcal{G}$ of the generator mapping the factors to the weight matrix is rank-deficient, so the factor-space preconditioner $J_\mathcal{G}^* \mathcal{F}_t J_\mathcal{G}$ induced by any $\bm{W}$-space preconditioner $\mathcal{F}_t$ is singular, and consequently the standard chain rule cannot be uniquely inverted to map a preconditioned $\bm{W}$-space direction back to a factor-space update. Therefore, we introduce \textbf{FILO}, a unified
Fisher-geometric framework for low-rank adaptation that enables factor updates without directly inverting this singular operator and casts existing LoRA optimizers by their factor-space inversion strategy and $\bm{W}$-space preconditioner $\mathcal{F}_t$. Existing methods occupy four families along these axes: factor-space adaptive updates, block-diagonal surrogates for $J_\mathcal{G}^* J_\mathcal{G}$, Frobenius-residual pseudoinverse methods, and Riemannian manifold methods. Within this design space, a gradient-statistics-aware $\mathcal{F}_t$ paired with a closed-form factor-space solve at $\mathcal{O}((m+n)r)$ memory remains underexplored. We instantiate FILO as \textbf{FILO-Adafactor} by adopting the Adafactor diagonal Kronecker preconditioner $\mathcal{H}_t$ on $\bm{W}$ and selecting from the resulting factor-space solution family the element minimizing an $\mathcal{H}_t$-weighted imbalance between the two factor contributions. By construction, the induced weight update is the closest LoRA approximation to the preconditioned $\bm{W}$-space direction under the $\mathcal{H}_t$-weighted norm. Across GPT-2 on E2E, Mistral-7B and Qwen2-7B on GLUE, ARC, and GSM8K, and diffusion-model personalization, FILO-Adafactor and its momentum variant are competitive with or improve over a representative set of LoRA optimizers while keeping peak GPU memory at the LoRA optimizer level. An open-source implementation is available at \href{https://github.com/zliueq/FILO}{FILO}.
\end{abstract}

\renewcommand*{\thefootnote}{\fnsymbol{footnote}}
\setcounter{footnote}{1}
\footnotetext{Corresponding authors.}

\section{Introduction}

Fine-tuning large pretrained models~\citep{liu2024deepseek, yang2024qwen2, achiam2023gpt4} for downstream tasks is increasingly bottlenecked by the cost of full-parameter updates, motivating parameter-efficient fine-tuning (PEFT). Low-Rank Adaptation (LoRA)~\citep{hu2022lora} has become the standard PEFT method: it freezes the pretrained weight $\bm{W}_0$ and reparameterizes its update as a product $\bm{B}\bm{A}$ with $\bm{B} \in \mathbb{R}^{m \times r}$, $\bm{A} \in \mathbb{R}^{r \times n}$, $r \ll \min(m,n)$, reducing trainable parameters and optimizer state from $\mathcal{O}(mn)$ to $\mathcal{O}((m+n)r)$. A growing line of work~\citep{hayou2024lora+, zhang2024RiemannianPreconditioned, wang2024lora-ga, zhao2024galore, zhu2024imbalanceLoRA, wang2025lora-pro, mo2025loro, zhang2025lora_one} extends this template with refined optimizers in pursuit of full-fine-tuning quality at LoRA's memory budget.

Despite this progress, optimizing in the factor space $[\bm{B}, \bm{A}]$ rather than directly in $\bm{W}$ raises a fundamental obstruction (\S~\ref{sec:background}): writing $\mathcal{G}: [\bm{B}, \bm{A}] \mapsto \bm{B}\bm{A}$ for the map from the factors to the weight matrix, its Jacobian $J_\mathcal{G}$ is rank-deficient because $\mathcal{G}$ has a built-in redundancy under the gauge reparameterization $(\bm{B}, \bm{A}) \mapsto (\bm{B}\bm{C}, \bm{C}^{-1}\bm{A})$ for any invertible $\bm{C}$. This redundancy also constrains preconditioning in factor space. Practical preconditioners are often motivated by approximations to the Fisher information in the optimized parameterization~\citep{kingma2014adam, martens2015K-FAC}, we therefore consider the empirical Fisher with respect to $[\bm{B}, \bm{A}]$. By the chain rule, the factor-space empirical Fisher must take the form $J_\mathcal{G}^* \mathcal{F}_t J_\mathcal{G}$, where $\mathcal{F}_t$ is the corresponding empirical Fisher in $\bm{W}$-space. Because this gauge redundancy makes $J_\mathcal{G}^* \mathcal{F}_t J_\mathcal{G}$ singular, ordinary inversion cannot be used to obtain a factor-space update from a $\bm{W}$-space preconditioned direction, and different generalized inverses can yield different factor updates. We therefore introduce \textbf{FILO}, a unified \textbf{FI}sher-geometric framework for \textbf{LO}w-rank adaptation that organizes existing LoRA optimizers along two axes: the factor-space inversion strategy
and the choice of $\bm{W}$-space preconditioner $\mathcal{F}_t$. Viewed through FILO, cheap factor-space schemes preserve the $\mathcal{O}((m+n)r)$ memory budget but discard the gradient-statistics structure on $\bm{W}$. Vanilla LoRA~\citep{hu2022lora} and Imbalance-Reg~\citep{zhu2024imbalanceLoRA} bypass the framework altogether by applying per-coordinate adaptive updates directly on the factors. LoRA+~\citep{hayou2024lora+} and Riemannian Preconditioned LoRA~\citep{zhang2024RiemannianPreconditioned} take $\mathcal{F}_t = \bm{I}$ and use block-diagonal approximations of $J_\mathcal{G}^* J_\mathcal{G}$. LoRA-Pro~\citep{wang2025lora-pro} fits the preconditioned weight-space direction $\mathcal{F}_t^{-1}\bm{G}_t$ by minimizing the Frobenius residual $\|J_\mathcal{G}[\bm{\Delta}_{\bm{B}_t}, \bm{\Delta}_{\bm{A}_t}] - \mathcal{F}_t^{-1}\bm{G}_t\|_F^2$, its AdamW variant pairs a non-trivial $\mathcal{F}_t$ on $\bm{W}$ with a Frobenius residual, mismatching the preconditioner's metric, and explicitly maintains $\bm{W}$-space first/second moments at $\mathcal{O}(mn)$ memory. Manifold-based methods~\citep{pmlr-v97-kasai19a}, including Riemannian Muon~\citep{bogachev2025riemannian} and RAdamW~\citep{bian2026finding} take a Riemannian gradient step on $\mathcal{M}_r$ in the ambient $\bm{W}$-space and rely on a retraction back to the manifold. A gradient-statistics-aware $\mathcal{F}_t$ paired with a closed-form factor-space solve and $\mathcal{O}((m+n)r)$ optimizer state remains an underexplored design point.

Within FILO, we target this design point through the consistent factor-space system $J_\mathcal{G}^* \mathcal{F}_t J_\mathcal{G}\,[\bm{P}, \bm{Q}] = J_\mathcal{G}^*(\bm{G}_t)$ for $\mathcal{F}_t \succ 0$: when both factors have rank $r$, its solution set forms an $r^2$-dimensional affine subspace. Our instantiation of FILO therefore makes two coupled choices (\S~\ref{sec:proposed}): (i) which gradient-statistics-aware preconditioner $\mathcal{F}_t$ to use on $\bm{W}$, and (ii) how to select a particular element of the affine solution set. For (i), we adopt the Adafactor~\citep{shazeer2018adafactor} diagonal Kronecker form $\mathcal{F}_t = \mathcal{H}_t = (\bm{L}_t \otimes \bm{R}_t)^{\frac{1}{2}}$, acting as $\mathcal{H}_t \bm{Y} = \bm{L}_t^{1/2} \bm{Y} \bm{R}_t^{1/2}$, the cheapest non-trivial $\bm{W}$-space preconditioner, requiring $\mathcal{O}(m+n)$ memory. For (ii), all elements of the affine solution set induce the same first-order $\bm{W}$-update, namely the $\mathcal{H}_t$-orthogonal projection of $\mathcal{H}_t^{-1} \bm{G}_t$ onto the tangent space $\mathbb{T}_t$, but can lead to different factor updates. We therefore select the element that minimizes the $\mathcal{H}_t$-weighted imbalance $\|\bm{\Delta}_{\bm{B}_t}\bm{A}_t - \bm{B}_t\bm{\Delta}_{\bm{A}_t}\|_{\mathcal{H}_t}^2$ between the two factor contributions to the $\bm{W}$ update. The resulting algorithm, \textbf{FILO-Adafactor}, instantiates FILO with these two choices, admits a closed-form factor update at $\mathcal{O}(r^3)$ extra cost and keeps the optimizer state at $\mathcal{O}((m+n)r)$. Consequently, with current factor gradients, FILO-Adafactor's induced $\bm{W}$-update is the closest point in $\mathbb{T}_t$ to the Adafactor-preconditioned direction $\mathcal{H}_t^{-1} \bm{G}_t$ under the $\mathcal{H}_t$-weighted norm.

Empirically (\S~\ref{sec:experiment}), FILO-Adafactor matches or outperforms vanilla LoRA, Scaled AdamW, LoRA-Pro AdamW, and SOAP across GPT-2 on E2E, Mistral-7B and Qwen2-7B on GLUE, ARC, and GSM8K, and diffusion-model fine-tuning, while matching Scaled AdamW's peak memory and avoiding the $\sim 2\times$ memory overhead of LoRA-Pro AdamW. Our contributions are:
\begin{itemize}
    \item \textbf{FILO}, a unified Fisher-geometric framework for low-rank adaptation that recasts existing LoRA optimizers as instances of the consistent linear system $J_\mathcal{G}^* \mathcal{F}_t J_\mathcal{G}\,[\bm{\Delta}_{\bm{B}_t}, \bm{\Delta}_{\bm{A}_t}] = J_\mathcal{G}^*(\bm{G}_t)$, parameterized by the choice of preconditioner $\mathcal{F}_t$ and the factor-space inversion strategy (\S~\ref{subsec:lora-precond-related}).
    \item \textbf{FILO-Adafactor}, an instantiation of FILO whose $\bm{W}$-update is the closest point in $\mathbb{T}_t$ to the Adafactor-preconditioned direction $\mathcal{H}_t^{-1} \bm{G}_t$ under the $\mathcal{H}_t$-weighted norm, recovered in closed form at $\mathcal{O}((m+n)r)$ optimizer state memory (\S~\ref{sec:proposed}).
    \item Experimental evidence that the resulting update direction is competitive with or improves over both cheap factor-space baselines and baselines based on weight-space preconditioning, including at the 7B parameter scale.
\end{itemize}

\section{The FILO Framework and Related Work}\label{sec:background}

This section sets up the LoRA optimization problem (\S~\ref{subsec:lora-setup}), identifies the singular factor-space operator $J_\mathcal{G}^* \mathcal{F}_t J_\mathcal{G}$ that any $\bm{W}$-space preconditioner $\mathcal{F}_t$ induces and the resulting non-uniqueness of factor update, organizes existing LoRA optimizers in the FILO framework by their choice of $\mathcal{F}_t$ and their factor-space inversion strategy of $J_\mathcal{G}^* \mathcal{F}_t J_\mathcal{G}$ (\S~\ref{subsec:lora-precond-related}), and reviews adaptive preconditioner families on $\bm{W}$ (\S~\ref{subsec:full-weight-precond}). Throughout the paper, calligraphic letters ($\mathcal{F}_t, \mathcal{H}_t, \widetilde{\mathcal{P}}_{\mathbb{T}_t}$) denote linear operators on $\mathbb{R}^{m\times n}$, while bold letters ($\bm{L}_t, \bm{R}_t, \bm{B}_t, \bm{A}_t, \bm{G}_t$) denote matrices, $J_\mathcal{G}$ maps factor-space directions to weight-space directions. A complete notation table is given in Appendix~\ref{sec:notation}.

\subsection{LoRA Setup and Its Singular Jacobian}\label{subsec:lora-setup}

As a representative parameter-efficient fine-tuning method, low-rank fine-tuning freezes the pretrained weight $\bm{W}_0 \in \mathbb{R}^{m \times n}$ and assumes that the weight update $\bm{W}$ admits a low-rank factorization $\bm{W} = \bm{B}\bm{A}$ with $\bm{B} \in \mathbb{R}^{m \times r}$, $\bm{A} \in \mathbb{R}^{r \times n}$, and $r \ll \min\{m, n\}$~\citep{hu2022lora}. The fine-tuning objective is
\[
    \min_{\bm{B} \in \mathbb{R}^{m \times r},\, \bm{A} \in \mathbb{R}^{r \times n}} \mathcal{L}(\bm{W}_0 + \mathcal{G}([\bm{B},\bm{A}])),
    \quad \text{where} \quad \mathcal{G}([\bm{B},\bm{A}]) = \bm{B}\bm{A}.
\]
Under the generator $\mathcal{G}$, the Jacobian operator $J_\mathcal{G}([\bm{B}_t, \bm{A}_t]): [\mathbb{R}^{m \times r}, \mathbb{R}^{r \times n}] \to \mathbb{R}^{m \times n}$ and its adjoint $J_\mathcal{G}^*([\bm{B}_t, \bm{A}_t]): \mathbb{R}^{m \times n} \to [\mathbb{R}^{m \times r}, \mathbb{R}^{r \times n}]$ act as
\begin{equation}\label{eq:jacobian-def}
    J_\mathcal{G}([\bm{B}_t, \bm{A}_t])[\bm{P}, \bm{Q}] = \bm{P}\bm{A}_t + \bm{B}_t\bm{Q}, \qquad
    J_\mathcal{G}^*([\bm{B}_t, \bm{A}_t])(\bm{C}) = [\bm{C}\bm{A}_t^\top,\, \bm{B}_t^\top\bm{C}],
\end{equation}
on factor-space directions $[\bm{P}, \bm{Q}] \in [\mathbb{R}^{m \times r}, \mathbb{R}^{r \times n}]$ and $\bm{W}$-space directions $\bm{C} \in \mathbb{R}^{m \times n}$, respectively. We abbreviate $J_\mathcal{G} := J_\mathcal{G}([\bm{B}_t, \bm{A}_t])$ and $J_\mathcal{G}^* := J_\mathcal{G}^*([\bm{B}_t, \bm{A}_t])$ when the base point is clear; detailed derivations appear in Proposition~\ref{prop:Jacobian-comp}. The chain rule gives the factor gradients $\bm{G}_{\bm{B}_t} = \bm{G}_t \bm{A}_t^\top, \bm{G}_{\bm{A}_t} = \bm{B}_t^\top \bm{G}_t$, with $\bm{G}_t = \nabla_{\bm{W}_t}\mathcal{L}(\bm{W}_0 + \bm{W}_t)$, equivalently
\begin{equation}\label{eq:chain-rule}
    J_\mathcal{G}^*(\bm{G}_t) = [\bm{G}_{\bm{B}_t}, \bm{G}_{\bm{A}_t}].
\end{equation}
Thus $J_\mathcal{G}$ is the central operator linking factor-space and $\bm{W}$-space updates, and its properties determine what factor-space optimizers can achieve.

Unfortunately, $J_\mathcal{G}$ has a non-trivial kernel: \emph{the Jacobian $J_\mathcal{G}$ is rank-deficient.} The Jacobian formula in~\eqref{eq:jacobian-def} immediately produces a family of factor-space directions that $J_\mathcal{G}$ maps to $\bm{0} \in \mathbb{R}^{m \times n}$: for any $\bm{X} \in \mathbb{R}^{r \times r}$,
\begin{equation}\label{eq:ker-direction}
    J_\mathcal{G}[\bm{B}_t \bm{X},\, -\bm{X}\bm{A}_t] = \bm{B}_t \bm{X} \bm{A}_t - \bm{B}_t \bm{X} \bm{A}_t = \bm{0},
\end{equation}
so $[\bm{B}_t \bm{X}, -\bm{X}\bm{A}_t] \in \ker(J_\mathcal{G})$. When $\bm{B}_t$ has column rank $r$ and $\bm{A}_t$ has row rank $r$, this family is the entire kernel: $\ker(J_\mathcal{G}) = \{[\bm{B}_t \bm{X}, -\bm{X}\bm{A}_t] : \bm{X} \in \mathbb{R}^{r \times r}\}$, an $r^2$-dimensional subspace, so $\mathrm{rank}(J_\mathcal{G}) = (m+n)r - r^2$ (Proposition~\ref{prop:gauge-kernel}).

This rank deficiency also constrains the form of preconditioners in factor space. Since practical preconditioners are often motivated by approximations to the Fisher information in the optimized parameterization~\citep{kingma2014adam,martens2015K-FAC}, the natural preconditioner for the factors \([\bm{B}_t,\bm{A}_t]\) is the empirical Fisher formed from the per-sample factor gradients
\(
[\bm{G}_{\bm{B}_t}^{(i)}, \bm{G}_{\bm{A}_t}^{(i)}]
:=
[\nabla_{\bm{B}} \mathcal{L}_i(\bm{W}_t), \nabla_{\bm{A}} \mathcal{L}_i(\bm{W}_t)],
\)
where $i$ indexes the $N$ training samples and $\mathcal{L}_i$ denotes the corresponding loss. That is,
\[
    \mathcal{E}_t
    \;:=\;
    \frac{1}{N}\sum_{i=1}^{N}\langle
    [\bm{G}_{\bm{B}_t}^{(i)}, \bm{G}_{\bm{A}_t}^{(i)}],\,\cdot
    \big\rangle
    [\bm{G}_{\bm{B}_t}^{(i)}, \bm{G}_{\bm{A}_t}^{(i)}]
    \,\big.
\]
Adaptive optimizers use structured preconditioners: diagonal for
Adam~\citep{kingma2014adam}, rank-1 Kronecker for
Adafactor~\citep{shazeer2018adafactor}, full Kronecker for
Shampoo~\citep{gupta2018shampoo}, and layerwise Kronecker for
K-FAC~\citep{martens2015K-FAC}. In practice, these preconditioners are built from accumulated or exponentially averaged mini-batch gradient statistics. However, as our experiments show, AdamW performs poorly in the factorized setting. This motivates a closer look at the structure of $\mathcal{E}_t$.

Because each \(\mathcal{L}_i\) depends on the factors only through \(\bm{W}=\bm{B}\bm{A}\), the chain rule gives
\(
[\bm{G}_{\bm{B}_t}^{(i)}, \bm{G}_{\bm{A}_t}^{(i)}]
=
J_{\mathcal{G}}^*(\bm{G}_t^{(i)})
\),
where
\(
\bm{G}_t^{(i)} := \nabla_{\bm{W}} \mathcal{L}_i(\bm{W}_t).
\)
Hence
\begin{equation}\label{eq:factor-precond-pullback}
    \mathcal{E}_t
    \;=\;
    \frac{1}{N}\sum_{i=1}^{N} \big\langle
    J_{\mathcal{G}}^*(\bm{G}_t^{(i)}),\,\cdot
    \big\rangle
    J_{\mathcal{G}}^*(\bm{G}_t^{(i)})
    \, = J_{\mathcal{G}}^* \Big(\frac{1}{N}\sum_{i=1}^{N} \bm{G}_t^{(i)} \langle {\bm{G}_t^{(i)}}, \cdot \rangle \Big)J_{\mathcal{G}}
    \;=\;
    J_{\mathcal{G}}^* \mathcal{F}_t J_{\mathcal{G}},
\end{equation}
where $\mathcal{F}_t :=\frac{1}{N}\sum_{i=1}^{N}\bm{G}_t^{(i)} \langle \bm{G}_t^{(i)}, \cdot\rangle$ on the right is the empirical Fisher in \(\bm{W}\)-space~\citep{kunstner2019limitations}.
Equivalently, in matrix form,
\begin{equation}\label{eq:empirical-fisher-W}
    \mathcal{F}_t
    \;=\;
    \frac{1}{N}\sum_{i=1}^{N}
    \mathrm{vec}(\bm{G}_t^{(i)})\, \mathrm{vec}(\bm{G}_t^{(i)})^\top
    \;\in\; \mathbb{R}^{mn \times mn}.
\end{equation}

Motivated by \eqref{eq:factor-precond-pullback}, we consider a general positive-definite weight-space preconditioner $\mathcal{F}_t$. Since \(J_{\mathcal{G}}\) is rank-deficient, the induced factor-space operator \(J_{\mathcal{G}}^* \mathcal{F}_t J_{\mathcal{G}}\) is necessarily singular. When $\mathcal{F}_t \succ 0$, $\ker(J_\mathcal{G}^* \mathcal{F}_t J_\mathcal{G}) = \ker(J_\mathcal{G})$ is non-trivial, so $J_\mathcal{G}^* \mathcal{F}_t J_\mathcal{G}$ is singular for any choice of $\mathcal{F}_t$, and the corresponding preconditioned update on $[\bm{B}_t, \bm{A}_t]$,
\begin{equation}\label{eq:factor-precond-update}
    [\bm{B}_{t+1}, \bm{A}_{t+1}] = [\bm{B}_t, \bm{A}_t] - \eta_t\, (J_\mathcal{G}^* \mathcal{F}_t J_\mathcal{G})^{-1}\, [\bm{G}_{\bm{B}_t}, \bm{G}_{\bm{A}_t}],
\end{equation}
is ill-defined.

\begin{obstructionbox}
\obstruction{non-invertibility}\label{obs:non-invertible} The operator $J_\mathcal{G}^* \mathcal{F}_t J_\mathcal{G}$ is non-invertible for any choice of $\mathcal{F}_t$, so the update~\eqref{eq:factor-precond-update} is ill-defined.
\end{obstructionbox}

One natural remedy is to replace the inverse with a generalized inverse, but different choices produce different factor updates, so the ill-definedness shifts from non-existence to non-uniqueness. A canonical choice is the Moore--Penrose pseudoinverse $(J_\mathcal{G}^* \mathcal{F}_t J_\mathcal{G})^\dag$:
\[
    [\bm{\Delta}_{\bm{B}_t}, \bm{\Delta}_{\bm{A}_t}] \;:=\; (J_\mathcal{G}^* \mathcal{F}_t J_\mathcal{G})^\dag\, [\bm{G}_{\bm{B}_t}, \bm{G}_{\bm{A}_t}] \;\in\; [\mathbb{R}^{m \times r}, \mathbb{R}^{r \times n}].
\]
By property of the Moore--Penrose pseudoinverse, $[\bm{\Delta}_{\bm{B}_t}, \bm{\Delta}_{\bm{A}_t}]$ is the unique minimum-Frobenius-norm element of the affine solution set
\[
    \big\{\, [\bm{P}, \bm{Q}] \in [\mathbb{R}^{m \times r}, \mathbb{R}^{r \times n}] \;:\; J_\mathcal{G}^* \mathcal{F}_t J_\mathcal{G}\,[\bm{P}, \bm{Q}] = [\bm{G}_{\bm{B}_t}, \bm{G}_{\bm{A}_t}] \,\big\},
\]
which is consistent since $\mathrm{range}(J_\mathcal{G}^*) = \mathrm{range}(J_\mathcal{G}^* \mathcal{F}_t J_\mathcal{G})$ as $\mathcal{F}_t \succ 0$, and has dimension $\dim \ker(J_\mathcal{G}^* \mathcal{F}_t J_\mathcal{G}) = \dim \ker(J_\mathcal{G}) = r^2$ when both $\bm{B}_t$ and $\bm{A}_t$ have rank $r$. Other generalized inverses of $J_\mathcal{G}^* \mathcal{F}_t J_\mathcal{G}$ correspond to different elements of this affine solution set, differing by an element of $\ker(J_\mathcal{G})$.

\begin{obstructionbox}
\obstruction{non-uniqueness of generalized inverse}\label{obs:non-unique} Generalized inverses of $J_\mathcal{G}^* \mathcal{F}_t J_\mathcal{G}$ are not unique, and different choices can yield factor updates differing by elements of $\ker(J_\mathcal{G})$. An additional selection criterion is needed to determine a unique factor update.
\end{obstructionbox}

These observations motivate FILO, a framework that separates the choice of $\mathcal{F}_t$ on $\bm{W}$ from the inversion strategy for the induced factor-space operator. Existing LoRA optimizers in \S~\ref{subsec:lora-precond-related} differ in the choice of $\mathcal{F}_t$ on $\bm{W}$ and in the rule that selects an element of this affine solution set; \S~\ref{subsec:full-weight-precond} reviews the standard families of $\mathcal{F}_t$. Our method FILO-Adafactor (\S~\ref{sec:proposed}) instantiates this framework with the Adafactor diagonal Kronecker and an $\mathcal{H}_t$-balance criterion that selects a unique element of the affine solution set.

\subsection{FILO: A Unified View of LoRA Optimizers}\label{subsec:lora-precond-related}

Although $J_\mathcal{G}^* \mathcal{F}_t J_\mathcal{G}$ is singular, the linear system
\begin{equation}\label{eq:factor-system}
    J_\mathcal{G}^* \mathcal{F}_t J_\mathcal{G}\,[\bm{\Delta}_{\bm{B}_t}, \bm{\Delta}_{\bm{A}_t}] \;=\; J_\mathcal{G}^*(\bm{G}_t)
\end{equation}
in the factor update $[\bm{\Delta}_{\bm{B}_t}, \bm{\Delta}_{\bm{A}_t}]$ is consistent, with RHS equal to the factor-gradient pair $[\bm{G}_{\bm{B}_t}, \bm{G}_{\bm{A}_t}]$ by~\eqref{eq:chain-rule}. FILO organizes existing LoRA optimizers by (i) the choice of the weight-space preconditioner $\mathcal{F}_t$, which determines the preconditioned direction $\mathcal{F}_t^{-1}\bm{G}_t$ and induces the factor-space operator $J_\mathcal{G}^* \mathcal{F}_t J_\mathcal{G}$, and (ii) which invertible surrogate for this singular operator they use to obtain factor updates. For methods without such a surrogate, we instead record their generalized-inverse construction or their strategy for computing updates without solving~\eqref{eq:factor-system}. Table~\ref{tab:lora-precond-compare} summarizes the resulting design space, and we walk through the main families below.

\textbf{Vanilla LoRA / Imbalance-Reg / LoRA-RITE  (diagonal approximation of $\mathcal{E}_t$, ignoring $J_\mathcal{G}^* \mathcal{F}_t J_\mathcal{G}$).}
Vanilla LoRA~\citep{hu2022lora} ignores $J_\mathcal{G}^* \mathcal{F}_t J_\mathcal{G}$ and approximates the empirical Fisher $\mathcal{E}_t$ on the factor space by a per-coordinate diagonal estimated directly from the factor gradients $[\bm{G}_{\bm{B}_t}, \bm{G}_{\bm{A}_t}]$ via AdamW. Imbalance-Regularized LoRA~\citep{zhu2024imbalanceLoRA} keeps the same diagonal Fisher estimate and adds a penalty $\|\bm{B}_t^\top \bm{B}_t - \bm{A}_t \bm{A}_t^\top\|_F^2$ to align the factor spectra. LoRA-RITE~\citep{yen2025lorarite} replaces the diagonal estimate with a matrix-form second moment $\bm{V}_t \in \mathbb{R}^{r \times r}$ accumulated on the polar/QR-reparameterized factor gradients, yielding a transformation-invariant factor-space update at $\mathcal{O}(r^2)$ extra memory.

\textbf{LoRA+~\citep{hayou2024lora+} and Riemannian Preconditioned LoRA~\citep{zhang2024RiemannianPreconditioned} (block-diagonal surrogates for $J_\mathcal{G}^* J_\mathcal{G}$).}
Specializing~\eqref{eq:factor-system} to $\mathcal{F}_t = \bm{I}$, the operator $J_\mathcal{G}^* J_\mathcal{G}$ (Proposition~\ref{prop:Jacobian-comp}) decomposes into block-diagonal and cross terms. LoRA+ approximates $J_\mathcal{G}^* J_\mathcal{G}$ by the block-scaling identity surrogate $\mathrm{diag}(\bm{I}_m, \lambda \bm{I}_n)$ for a fixed scalar $\lambda > 0$ and inverts it, giving asymmetric per-block scalar rescaling between the $\bm{B}_t$ and $\bm{A}_t$ updates. Riemannian Preconditioned LoRA approximates $J_\mathcal{G}^* J_\mathcal{G}$ by its block-diagonal part $\mathrm{diag}(\bm{A}_t \bm{A}_t^\top, \bm{B}_t^\top \bm{B}_t)$ and inverts it, yielding the explicit factor update $\bm{\Delta}_{\bm{B}_t} = \bm{G}_{\bm{B}_t} (\bm{A}_t \bm{A}_t^\top)^{-1}$, $\bm{\Delta}_{\bm{A}_t} = (\bm{B}_t^\top \bm{B}_t)^{-1} \bm{G}_{\bm{A}_t}$, well-defined whenever $\bm{B}_t, \bm{A}_t$ have full rank. Both updates generally differ from any element of the affine solution set of~\eqref{eq:factor-system}.

\textbf{LoRA-Pro~\citep{wang2025lora-pro}.}
LoRA-Pro solves a different system from~\eqref{eq:factor-system}: it minimizes the Frobenius residual $\|J_\mathcal{G}[\bm{\Delta}_{\bm{B}_t}, \bm{\Delta}_{\bm{A}_t}] - \mathcal{F}_t^{-1}\bm{G}_t\|_F^2$, whose normal equations $J_\mathcal{G}^* J_\mathcal{G}\,[\bm{\Delta}_{\bm{B}_t}, \bm{\Delta}_{\bm{A}_t}] = J_\mathcal{G}^*(\mathcal{F}_t^{-1}\bm{G}_t)$ coincide with~\eqref{eq:factor-system} only when $\mathcal{F}_t = \bm{I}$ but generally differ for other $\mathcal{F}_t$. Its AdamW variant pairs a non-trivial $\mathcal{F}_t$ on $\bm{W}$ with the Frobenius (rather than $\mathcal{F}_t$-weighted) residual, mismatching the preconditioner's metric, and explicitly maintains $\bm{W}$-space first/second moments at $\mathcal{O}(mn)$ memory prohibitive at LLM scale. In contrast, our~\eqref{prob:fit-proj} measures the residual under the $\mathcal{F}_t$-induced $\mathcal{H}_t$-norm consistent with the preconditioner.

\textbf{Manifold-based methods on $\mathcal{M}_r$.}
Rather than solving~\eqref{eq:factor-system} on the factor space, this line of work performs Riemannian gradient descent on the rank-$r$ matrix manifold $\mathcal{M}_r$. Riemannian Muon~\citep{bogachev2025riemannian} uses retraction-based Muon updates on $\mathcal{M}_r$, applying Muon orthogonalization (replacing all singular values by $1$) on the tangent space. The resulting step is equivalent to a per-step spectral $\bm{W}$-space preconditioner $\mathcal{F}_t = (\bm{G}_t\bm{G}_t^\top)^{1/2}$ (no accumulation across steps). RAdaGrad / RAdamW~\citep{bian2026finding} run Riemannian gradient descent on $\mathcal{M}_r$ under a Shampoo $\bm{W}$-space preconditioner $\mathcal{F}_t = (\bm{L}_{\rm Sh} \otimes \bm{R}_{\rm Sh})^{\frac{1}{4}}$ restricted to the manifold tangent space, achieving a similar $\mathcal{F}_t$-aware behaviour to ours but via a retraction step on $\mathcal{M}_r$ instead of a closed-form solution of~\eqref{eq:factor-system} in factor coordinates.

\textbf{Other directions (LoRA-RITE / LoRA-GA).}
LoRA-RITE~\citep{yen2025lorarite} introduces transformation invariance via a polar-decomposition-based reparameterization of the factor coordinates, with $\mathcal{F}_t = \bm{I}$;
LoRA-GA~\citep{wang2024lora-ga} addresses initialization through spectral alignment with full fine-tuning gradients.

These methods reveal a recurring trade-off: cheap factor-space schemes either bypass~\eqref{eq:factor-system} or take $\mathcal{F}_t = \bm{I}$ with block-diagonal approximations of $J_\mathcal{G}^* J_\mathcal{G}$, while methods admitting a non-trivial $\mathcal{F}_t$, such as LoRA-Pro AdamW, pay $\mathcal{O}(mn)$ memory or operate in the ambient $\bm{W}$-space. A gradient-statistics-aware $\mathcal{F}_t$ paired with $\mathcal{O}((m+n)r)$ memory in the LoRA factor space remains an underexplored design point, which our method (\S~\ref{sec:proposed}) targets via the Adafactor diagonal Kronecker $\mathcal{F}_t = \bm{L}_t^{1/2} \otimes \bm{R}_t^{1/2}$ together with a closed-form solution of~\eqref{eq:factor-system} that picks a specific element of the $r^2$-dimensional affine solution set.

\begin{table*}[t]
\centering
\caption{Existing LoRA optimizers as instances of the framework~\eqref{eq:factor-system}, grouped by how they handle the singular operator $J_\mathcal{G}^* \mathcal{F}_t J_\mathcal{G}$. Inversion strategy = how $(J_\mathcal{G}^* \mathcal{F}_t J_\mathcal{G})^{-1}$ is replaced or approximated; $\mathcal{F}_t$ = the gradient-statistics structure on $\bm{W}$ used (``\textemdash'' means the method bypasses~\eqref{eq:factor-system} and does not instantiate $\mathcal{F}_t$). Per-step cost and memory are per-layer beyond forward/backward through $\bm{B}\bm{A}$ and storing $[\bm{B}_t, \bm{A}_t]$. $\bm{g} := \mathrm{vec}(\bm{G}_t) \in \mathbb{R}^{mn}$ denotes the vectorized $\bm{W}$-space gradient.}
\label{tab:lora-precond-compare}
\small
\renewcommand{\arraystretch}{1.15}
\setlength{\tabcolsep}{4pt}
\begin{tabular}{lllll}
\toprule
\textbf{Method} & \textbf{Inversion strategy} & $\mathcal{F}_t$ \textbf{on $\bm{W}$} & \textbf{Per-step cost} & \textbf{Memory} \\
\midrule
\multicolumn{5}{l}{\textit{Bypass~\eqref{eq:factor-system} (factor-space AdamW)}} \\
Vanilla LoRA~\citep{hu2022lora}            & \textemdash              & \textemdash & $\mathcal{O}((m+n)r)$         & $\mathcal{O}((m+n)r)$ \\
LoRA-RITE~\citep{yen2025lorarite}           & \textemdash              & \textemdash & $\mathcal{O}((m+n)r^2)$       & $\mathcal{O}((m+n)r)$ \\
\midrule
\multicolumn{5}{l}{\textit{Block-diagonal approx.\ of $J_\mathcal{G}^*J_\mathcal{G}$}} \\
LoRA+~\citep{hayou2024lora+}                & Block-scaling & $\bm{I}$ & $\mathcal{O}((m+n)r)$         & $\mathcal{O}((m+n)r)$ \\
Riem.\ Precond.~\citep{zhang2024RiemannianPreconditioned} & Block-diag  & $\bm{I}$ & $\mathcal{O}((m+n)r^2)$ & $\mathcal{O}((m+n)r)$ \\
\midrule
\multicolumn{5}{l}{\textit{Pseudoinverse of $J_\mathcal{G}$}} \\
LoRA-Pro SGD~\citep{wang2025lora-pro}   & $J_\mathcal{G}^\dag$  $\mathcal{F}_t^{-1}$ & $\bm{I}$ & $\mathcal{O}((m+n)r^2)$ & N/A  \\
LoRA-Pro AdamW~\citep{wang2025lora-pro}   & $J_\mathcal{G}^\dag$  $\mathcal{F}_t^{-1}$ & $\diag(\bm{g}\odot\bm{g})^{\frac{1}{2}}$ & $\mathcal{O}(mnr)$ & $\mathcal{O}(mn)$ \\

\midrule
\multicolumn{5}{l}{\textit{Riemannian gradient descent on $\mathcal{M}_r$}} \\
Riem.\ Muon~\citep{bogachev2025riemannian} & RGD on $\mathcal{M}_r$ & $(\bm{G}_t\bm{G}_t^\top)^{1/2}$ & $\mathcal{O}((m+n)r^2)$       & $\mathcal{O}((m+n)r)$ \\
RAdamW~\citep{bian2026finding} & RGD on $\mathcal{M}_r$ & $\diag(\bm{L}_{\rm Sh}\otimes\bm{R}_{\rm Sh})^{\frac{1}{4}}$      & $\mathcal{O}((m+n)r^2)$       & $\mathcal{O}((m+n)r)$ \\
\midrule
\textbf{FILO-Adafactor (ours)}  & solve \eqref{eq:factor-system} & Adafactor diag-Kron & $\mathcal{O}((m+n)r^2)$ & $\mathcal{O}((m+n)r)$ \\
\bottomrule
\end{tabular}
\end{table*}

\subsection{\texorpdfstring{Choosing $\mathcal{F}_t$: Adaptive Preconditioner Toolkit on $\bm{W}$}{Choosing F\^{}W: Adaptive Preconditioner Toolkit on W}}\label{subsec:full-weight-precond}

The gap identified above asks for an $\mathcal{F}_t$ that is gradient-statistics-based yet cheap on $\bm{W}$. We review the standard families of $\bm{W}$-space preconditioners, organized by memory cost. All families construct a second-moment-based preconditioner $\mathcal{F}_t : \mathbb{R}^{m \times n} \to \mathbb{R}^{m \times n}$ from gradient outer-product statistics of the form $\bm{G}_t \bm{G}_t^\top$ or $\bm{G}_t \odot \bm{G}_t$, and produce the preconditioned update $\bm{W}_{t+1} = \bm{W}_t - \eta_t \mathcal{F}_t^{-1} \bm{G}_t$; they differ in the structure imposed on $\mathcal{F}_t$, which trades off expressiveness against cost. AdaGrad~\citep{duchi2011AdaGrad} and Adam~\citep{kingma2014adam} approximate $\mathcal{F}_t$ by its diagonal $\bm{h}_t \in \mathbb{R}^{mn}$ as an exponential moving average of $\bm{G}_t \odot \bm{G}_t$, yielding per-coordinate rescaling that ignores the matrix structure of $\bm{G}_t$ at memory $\mathcal{O}(mn)$. Adafactor~\citep{shazeer2018adafactor} compresses this further into a rank-1 Kronecker form by maintaining only the row sums $\bm{l}_t \in \mathbb{R}^m$ and column sums $\bm{r}_t \in \mathbb{R}^n$ of $\bm{G}_t \odot \bm{G}_t$ (the elementwise Hadamard product), dropping the memory cost to $\mathcal{O}(m+n)$. Shampoo~\citep{gupta2018shampoo} maintains $\bm{L}_{{\rm Sh},t} = \bm{L}_{{\rm Sh},t-1} + \bm{G}_t \bm{G}_t^\top \in \mathbb{R}^{m \times m}$ and $\bm{R}_{{\rm Sh},t} = \bm{R}_{{\rm Sh},t-1} + \bm{G}_t^\top \bm{G}_t \in \mathbb{R}^{n \times n}$ and updates by $\bm{W}_{t+1} = \bm{W}_t - \bm{L}_{{\rm Sh},t}^{-\frac{1}{4}} \bm{G}_t \bm{R}_{{\rm Sh},t}^{-\frac{1}{4}}$; SOAP~\citep{morwani2024new_shampoo, vyas2025soap} runs Adam in the eigenbasis of the Shampoo preconditioner; and K-FAC~\citep{martens2015K-FAC, eschenhagen2023NeurIPS-K-FAC} factorizes $\mathcal{F}_t$ as the Kronecker product of activation and gradient covariances. All three impose $\mathcal{O}(m^2 + n^2)$ memory and $\mathcal{O}(m^3 + n^3)$ per-step inverse cost, which dominates LoRA's budgets. Among these candidates, the Adafactor diagonal Kronecker form $\mathcal{F}_t = (\bm{L}_t \otimes \bm{R}_t)^{\frac{1}{2}}$ is the only one that is simultaneously gradient-statistics-based and cheap, requiring $\mathcal{O}(m+n)$ memory. Our method (\S~\ref{sec:proposed}) adopts this candidate and pairs it with a closed-form solution of the linear system~\eqref{eq:factor-system} that respects the LoRA factorization.

\section{FILO-Adafactor: An Instantiation of FILO}\label{sec:proposed}

We instantiate the framework FILO~\eqref{eq:factor-system} as \textbf{FILO-Adafactor} with two specific choices: (i) for the $\bm{W}$-space preconditioner, the Adafactor diagonal Kronecker form $\mathcal{H}_t = \bm{L}_t^{1/2} \otimes \bm{R}_t^{1/2}$ on $\bm{W}$ (\S~\ref{subsec:weighted-metric}); (ii) for the element of the affine solution set of~\eqref{eq:factor-system}, the unique minimizer of the $\mathcal{H}_t$-imbalance criterion (Solution~\ref{sol:balance}). Choice (i) solves the linear system \eqref{eq:factor-system} through an equivalent weighted least-squares problem (Obstruction~\ref{obs:non-invertible}); choice (ii) resolves the $r^2$-dimensional ambiguity over $\ker(J_\mathcal{G})$ (Obstruction~\ref{obs:non-unique}). The closed-form factor update is given in Theorem~\ref{thm:find_delta_B_A}. Figure~\ref{fig:schematic} contrasts the resulting $\bm{W}$-update geometry against LoRA-Pro and Riemannian Preconditioned LoRA~\citep{zhang2024RiemannianPreconditioned} under the $\mathcal{H}_t$-weighted inner product.

\subsection{\texorpdfstring{The Adafactor Preconditioner $\mathcal{H}_t = \bm{L}_t^{1/2} \otimes \bm{R}_t^{1/2}$}{The Adafactor Preconditioner H = L^{1/2} kron R^{1/2}}}\label{subsec:weighted-metric}

To distinguish our choice from those of existing methods, we denote the preconditioner used in FILO-Adafactor by $\mathcal{H}_t$. Specifically, we adopt the diagonal Kronecker preconditioner $\mathcal{F}_t=\mathcal{H}_t=(\bm{L}_t \otimes \bm{R}_t)^{1/2}$ on $\bm{W}$, where $\bm{L}_t, \bm{R}_t$ are the Adafactor~\citep{shazeer2018adafactor} rank-$1$ second-moment estimate of $\bm{G}_t \odot \bm{G}_t$:
\begin{equation}\label{eq:def_LR}
\begin{gathered}
    \bm{L}_t = \diag(\bm{l}_t / \sqrt{\|\bm{l}_t\|_1}), \quad \bm{R}_{t} = \diag(\bm{r}_t / \sqrt{\|\bm{r}_t\|_1}), \quad \textmd{with} \\
    \bm{l}_t = \beta_1\bm{l}_{t-1} + (1-\beta_1)\sum_{j=1}^{n} (\bm{G}_t \odot \bm{G}_t)_{i,j}, \quad \bm{r}_t = \beta_2 \bm{r}_{t-1} + (1-\beta_2)\sum_{i=1}^m (\bm{G}_t \odot \bm{G}_t)_{i,j},
\end{gathered}
\end{equation}
where $\odot$ denotes the Hadamard product, $\|\cdot\|_1$ denotes the $\ell_1$-norm, and $\beta_1, \beta_2 \in [0,1]$ are decay rates. The vectors $\bm{l}_t \in \mathbb{R}^m$ and $\bm{r}_t \in \mathbb{R}^n$ are the diagonals of the moving averages of $\bm{G}_t \bm{G}_t^\top$ and $\bm{G}_t^\top \bm{G}_t$, respectively, and together define the normalized rank-$1$ second-moment estimate
$\bm{l}_t \bm{r}_t^\top /
\sqrt{\|\bm{l}_t\|_1 \|\bm{r}_t\|_1}$
following Adafactor~\citep{shazeer2018adafactor}. The memory cost is $\mathcal{O}(m+n)$.

We treat $\mathcal{H}_t$ as an operator on $\mathbb{R}^{m \times n}$, defined by $\mathcal{H}_t \bm{Y} := \bm{L}_t^{1/2} \bm{Y} \bm{R}_t^{1/2}$ for any $\bm{Y} \in \mathbb{R}^{m \times n}$, with inverse $\mathcal{H}_t^{-1} \bm{K} = \bm{L}_t^{-1/2} \bm{K} \bm{R}_t^{-1/2}$. The $1/2$-power form yields the preconditioned direction
$\mathcal{H}_t^{-1} \bm{G}_t = \bm{L}_t^{-1/2} \bm{G}_t \bm{R}_t^{-1/2}$, following Adafactor's inverse-square-root second-moment scaling~\citep{shazeer2018adafactor}. This choice of exponent is also consistent with the $1/2$-power Shampoo variant studied in SOAP~\citep{vyas2025soap}. For full-matrix Shampoo, the Kronecker second-moment estimate corresponds, up to scaling, to one power-iteration step toward the Frobenius-optimal Kronecker approximation of the accumulated outer products of vectorized gradients~\citep{morwani2024new_shampoo}. We define the $\mathcal{H}_t$-weighted inner product by
\begin{equation}\label{eq:weighted-inner-product}
    \langle \bm{Y}, \bm{Z} \rangle_{\mathcal{H}_t} := \langle \mathcal{H}_t \bm{Y}, \bm{Z} \rangle = \langle \bm{L}_t^{1/2} \bm{Y} \bm{R}_t^{1/2}, \bm{Z} \rangle,
\end{equation}
where $\langle \cdot, \cdot \rangle$ is the Frobenius inner product.

\subsection{Solving the Linear System on Factor Space}\label{subsec:ouralgorithm}

With $\mathcal{H}_t = \bm{L}_t^{1/2} \otimes \bm{R}_t^{1/2}$ from \S~\ref{subsec:weighted-metric}, the factor-space linear system~\eqref{eq:factor-system} becomes
\begin{equation}\label{eq:factor-system-instantiated}
    J_\mathcal{G}^* \mathcal{H}_t J_\mathcal{G}\,[\bm{\Delta}_{\bm{B}_t}, \bm{\Delta}_{\bm{A}_t}] \;=\; J_\mathcal{G}^*(\bm{G}_t),
\end{equation}
in the candidate factor update $[\bm{\Delta}_{\bm{B}_t}, \bm{\Delta}_{\bm{A}_t}] \in \mathbb{R}^{m \times r} \times \mathbb{R}^{r \times n}$. The operator $J_\mathcal{G}^* \mathcal{H}_t J_\mathcal{G}$ is singular (Obstruction~\ref{obs:non-invertible}), so we cannot invert it.

\begin{solutionbox}
\solution{Resolve Obstruction~\ref{obs:non-invertible}: solve the equivalent least-squares problem}\label{sol:transport}
Equation~\eqref{eq:factor-system-instantiated} is the normal equation of
\begin{equation}\label{prob:fit-proj}
    \min_{\bm{\Delta}_{\bm{B}_t}, \bm{\Delta}_{\bm{A}_t}} \big\| J_\mathcal{G}[\bm{\Delta}_{\bm{B}_t}, \bm{\Delta}_{\bm{A}_t}] - \mathcal{H}_t^{-1}\bm{G}_t \big\|_{\mathcal{H}_t}^2,
\end{equation}
so solving~\eqref{prob:fit-proj} replaces inverting $J_\mathcal{G}^* \mathcal{H}_t J_\mathcal{G}$.
\end{solutionbox}

The following theorem characterizes the solution set of~\eqref{prob:fit-proj}.

\definecolor{cVanilla}{RGB}{50,135,190}   
\definecolor{cLoRAPro}{RGB}{195,108,30}   
\definecolor{cFILO}{RGB}{155,110,196}   
\definecolor{cPlane}{RGB}{120,120,120}    

\begin{figure}[!h]
\begin{minipage}[c]{0.42\textwidth}
\captionof{figure}{Geometric contrast of LoRA optimizers under the $\mathcal{F}_t$-weighted inner product on $\mathbb{R}^{m \times n}$; $\mathcal{F}_t^{-1}\bm{G}_t$ is the gradient under this inner product, and all updates land in $\mathbb{T}_t = \mathrm{range}(J_\mathcal{G})$. From $\mathcal{F}_t^{-1}\bm{G}_t$, \textcolor{cFILO}{\textbf{FILO-Adafactor}} drops $\mathcal{F}_t$-orthogonally; \textcolor{cLoRAPro}{\textbf{LoRA-Pro}} does not realize an orthogonal projection under the $\mathcal{F}_t$-weighted inner product. \textcolor{cVanilla}{\textbf{Riem.\ Precond.}} lies in $\mathbb{T}_t$ but is neither a Frobenius nor an $\mathcal{F}_t$-weighted orthogonal projection. LoRA-Pro and FILO-Adafactor coincide when $\mathcal{F}_t = \bm{I}$.}
\label{fig:schematic}
\end{minipage}\hfill
\begin{minipage}[c]{0.55\textwidth}
\centering
\begin{tikzpicture}[
  xscale=0.85, yscale=1.0,
  >=Stealth,
  every node/.style={font=\small},
]
  \fill[cPlane!12] (-0.4,-0.8) -- (5.2,-0.8) -- (6.6,1.2) -- (1.0,1.2) -- cycle;
  \draw[cPlane!70, thick] (-0.4,-0.8) -- (5.2,-0.8) -- (6.6,1.2) -- (1.0,1.2) -- cycle;
  \node[cPlane!90!black] at (2.9,-1.2) {$\mathbb{T}_t \subset \mathbb{R}^{m\times n}$, $\mathcal{F}_t$-weighted view};

  \coordinate (W) at (0.9,-0.3);
  \fill[black] (W) circle (2pt);
  \node[black, left=2pt of W] {$\bm{W}_t = \bm{B}_t \bm{A}_t$};

  \coordinate (FG) at (4.87,2.4);
  \draw[->, thick, black] (W) -- (FG);
  \node[black, above right=1pt and -10pt of FG] {$\mathcal{F}_t^{-1} \bm{G}_t$};

  \coordinate (G) at (3.2,2.4);
  \draw[->, thick, black] (W) -- (G);
  \node[black, above=1pt of G] {$\bm{G}_t$};

  \coordinate (S)  at (4.87,0.0);    
  \coordinate (LP) at (4.0,0.4);     
  \coordinate (V)  at (3.6,0.4);     

  \fill[cFILO] (S) circle (2.2pt);
  \draw[dashed, cFILO, very thick] (FG) -- (S);
  \node[cFILO, below=2pt of S, align=center, xshift=10pt] {\footnotesize \textbf{FILO-Adafactor}\\[-1pt] \footnotesize $\widetilde{\mathcal{P}}_{\mathbb{T}_t}(\mathcal{F}_t^{-1}\bm{G}_t)$};
  \draw[cFILO, line width=1.0pt] ($(S)+(0.0,0.22)$) -- ($(S)+(0.22,0.22)$) -- ($(S)+(0.22,0.0)$);

  \fill[cLoRAPro] (LP) circle (1.5pt);
  \draw[dashed, cLoRAPro, thick] (FG) -- (LP);
  \node[cLoRAPro, below left=2pt and -20pt of LP] {\footnotesize \textcolor{cLoRAPro}{LoRA-Pro}};

  \fill[cVanilla] (V) circle (1.5pt);
  \draw[dashed, cVanilla, thick] (G) -- (V);
  \node[cVanilla, above left=2pt and -2pt of V] {\footnotesize \textcolor{cVanilla}{Riem.\ Precond.}};
\end{tikzpicture}
\end{minipage}
\end{figure}

\begin{theorem}[Solution set of~\eqref{prob:fit-proj}]\label{thm:proj-solution}
Let $\widetilde{\bm{G}}_t := \mathcal{H}_t^{-1}\bm{G}_t$. Since $J_\mathcal{G}[\bm{\Delta}_{\bm{B}_t}, \bm{\Delta}_{\bm{A}_t}] \in \mathbb{T}_t = \mathrm{range}(J_\mathcal{G})$, the minimum of~\eqref{prob:fit-proj} is attained iff
\begin{equation}\label{eq:proj-image}
    \bm{\Delta}_{\bm{B}_t}\bm{A}_t + \bm{B}_t\bm{\Delta}_{\bm{A}_t} \;=\; \widetilde{\mathcal{P}}_{\mathbb{T}_t}(\widetilde{\bm{G}}_t),
\end{equation}
the $\mathcal{H}_t$-orthogonal projection of $\widetilde{\bm{G}}_t$ onto $\mathbb{T}_t$ (closed form in Proposition~\ref{prop:proj_tangent_weight}). The minimizers form an $r^2$-parameter family (Appendix~\ref{sec:proof_find_delta_B_A}, Lemma~\ref{lem:find_delta_B_A_with_X})
\begin{equation}\label{eq:X-family}
\begin{aligned}
    \bm{\Delta}_{\bm{B}_t}(\bm{X}_t) &= \big[\bm{L}_t^{-1/2} - \bm{B}_t (\bm{B}_t^\top \bm{L}_t^{1/2} \bm{B}_t)^{-1} \bm{B}_t^\top\big]\, \bm{G}_{\bm{B}_t} (\bm{A}_t \bm{R}_t^{1/2} \bm{A}_t^\top)^{-1} + \bm{B}_t \bm{X}_t, \\
    \bm{\Delta}_{\bm{A}_t}(\bm{X}_t) &= (\bm{B}_t^\top \bm{L}_t^{1/2} \bm{B}_t)^{-1} \bm{G}_{\bm{A}_t} \bm{R}_t^{-1/2} - \bm{X}_t \bm{A}_t,
\end{aligned}
\qquad \bm{X}_t \in \mathbb{R}^{r \times r},
\end{equation}
where the offsets $(\bm{B}_t \bm{X}_t, -\bm{X}_t \bm{A}_t)$ parameterize $\ker(J_\mathcal{G})$.
\end{theorem}

By Theorem~\ref{thm:proj-solution}, every factor pair in~\eqref{eq:X-family} induces the common $\bm{W}$-update $\widetilde{\mathcal{P}}_{\mathbb{T}_t}(\widetilde{\bm{G}}_t)$. The following solution selects a specific $\bm{X}_t$ to resolve this $r^2$-dimensional ambiguity (Obstruction~\ref{obs:non-unique}).

\begin{solutionbox}
\solution{$\mathcal{H}_t$-balance fixes the $\ker(J_\mathcal{G})$ ambiguity}\label{sol:balance}
Among the $\bm{X}_t$-family~\eqref{eq:X-family}, we select the unique element by choosing $\bm{X}_t$ to minimize the $\mathcal{H}_t$-imbalance $\| \bm{\Delta}_{\bm{B}_t}\bm{A}_t - \bm{B}_t\bm{\Delta}_{\bm{A}_t} \|_{\mathcal{H}_t}^2$ between the two factor contributions to the $\bm{W}$-update.
\end{solutionbox}

This criterion balances the magnitudes of the two factor contributions, in the same spirit as the regularizer in Imbalance-Regularized LoRA~\citep{zhu2024imbalanceLoRA} and the standard balance term used in nonconvex low-rank matrix recovery~\citep{tu2016low}. Combining Theorem~\ref{thm:proj-solution} with Solutions~\ref{sol:transport} and~\ref{sol:balance} fixes $\bm{X}_t$ in closed form and yields the full FILO-Adafactor update.
\begin{theorem}[FILO-Adafactor closed-form factor update]\label{thm:find_delta_B_A}
The unique factor update solving~\eqref{prob:fit-proj} together with the $\mathcal{H}_t$-balance criterion (Solution~\ref{sol:balance}) is
\[
    \bm{\Delta}_{\bm{B}_t}^{\rm opt} = \big(\bm{I} - \tfrac{1}{2}\widetilde{\bm{P}}_{B_t}\big) \bm{L}_t^{-1/2} \bm{G}_{\bm{B}_t} (\bm{A}_t \bm{R}_t^{1/2} \bm{A}_t^\top)^{-1}, \quad
    \bm{\Delta}_{\bm{A}_t}^{\rm opt} = (\bm{B}_t^\top \bm{L}_t^{1/2} \bm{B}_t)^{-1} \bm{G}_{\bm{A}_t} \bm{R}_t^{-1/2} \big(\bm{I} - \tfrac{1}{2}\widetilde{\bm{Q}}_{A_t}\big),
\]
where the $\mathcal{H}_t$-weighted projector matrices are
\begin{equation}\label{eq:proj-mats}
    \widetilde{\bm{P}}_{B_t} := \bm{B}_t (\bm{B}_t^\top \bm{L}_t^{1/2} \bm{B}_t)^{-1} \bm{B}_t^\top \bm{L}_t^{1/2}, \qquad
    \widetilde{\bm{Q}}_{A_t} := \bm{R}_t^{1/2} \bm{A}_t^\top (\bm{A}_t \bm{R}_t^{1/2} \bm{A}_t^\top)^{-1} \bm{A}_t.
\end{equation}
\end{theorem}
Under the $\mathcal{H}_t$-balance criterion (Solution~\ref{sol:balance}), two features are worth highlighting: (i) given $\mathcal{H}_t$, the update depends on the gradient only through the low-rank factor gradients $\bm{G}_{\bm{A}_t}, \bm{G}_{\bm{B}_t}$, keeping per-step optimizer state at $\mathcal{O}((m+n)r)$; (ii) the $\frac{1}{2}$ coefficients on the projectors are the signature of the $\mathcal{H}_t$-balance choice. The base algorithm, \textbf{FILO-Adafactor}, is summarized as Algorithm~\ref{alg:RAdaGrad} in the appendix; its computational complexity is analyzed in Appendix~\ref{sec:alg-anal}. For practical use we also provide \textbf{FILO-AdafactorM} (Algorithm~\ref{alg:RAdaGradM}), a momentum variant that maintains first-moment EMAs of the factor gradients. It also includes bias correction and decoupled weight
decay. In both variants, the second-moment statistics are computed from the weight-space gradient captured by a backward hook.

\section{Experimental Results}\label{sec:experiment}

We evaluate FILO-Adafactor and FILO-AdafactorM against representatives of the design families identified in Table~\ref{tab:lora-precond-compare}: LoRA trained with SGD / AdamW, Scaled GD / Scaled AdamW~\citep{zhang2024RiemannianPreconditioned}, LoRA-Pro SGD / AdamW~\citep{wang2025lora-pro}, and SOAP~\citep{vyas2025soap}. The Scaled baselines are the SGD and AdamW implementations of Riemannian Preconditioned LoRA. Three axes test the trade-off identified in \S~\ref{subsec:lora-precond-related}: model scale, from 124M--355M GPT-2 to 7B Mistral/Qwen2; task family, covering NLU, reasoning, math, generation, and image tasks; and resource cost, measured by peak GPU memory and per-step time. Following~\citep{zhang2024RiemannianPreconditioned, wang2025lora-pro}, learning rates are independently tuned per optimizer via grid search; full hyperparameters are in Appendix~\ref{sec:gpt2-setting}. Unless otherwise noted, all runs use PyTorch~\citep{paszke2019pytorch} on NVIDIA A100 GPUs.

\subsection{Controlled study: GPT-2}\label{sec:GPT2}

We start with controlled fine-tuning of GPT-2~\citep{radford2019gpt2} (small, 124M; medium, 355M) on the E2E natural language generation challenge~\citep{novikova2017e2e}, sweeping rank $r \in \{4, 16, 64\}$ to evaluate the effectiveness of our preconditioning approach at small scale.

\begin{table*}[t]
\caption{GPT-2 fine-tuning on E2E at $r=4$, across model size (small / medium). Bold/underline = best/second-best per metric per (model, optimizer family). DART and cross-rank ablation ($r \in \{16, 64\}$) results are reported in Table~\ref{tab:gpt2-dart} and Table~\ref{tab:score_GPT2_rank_ablation}.}\label{tab:score_GPT2}
\centering
\small
\begin{tabular}{cccccccc}
\toprule
\multirow{2}{*}{Model} & \multirow{2}{*}{$r$} & \multirow{2}{*}{Method} & \multicolumn{5}{c}{E2E} \\
\cmidrule(lr){4-8}
& & & BLEU & NIST & MET & ROUGE-L & CIDEr \\
\midrule
\multirow{8}{*}{\shortstack{GPT-2\\small}} & \multirow{8}{*}{4} & SGD & 54.8 & 4.56 & 34.0 & 63.3 & 1.29 \\
& & Scaled GD & \underline{68.5} & \underline{8.72} & \underline{45.5} & 69.4 & 2.40 \\
& & LoRA-Pro SGD & 68.4 & \underline{8.72} & \underline{45.5} & \underline{69.6} & \underline{2.43} \\
& & \textbf{FILO-Adafactor (ours)} & \textbf{69.5} & \textbf{8.77} & \textbf{46.5} & \textbf{71.5} & \textbf{2.50} \\
\cmidrule{3-8}
& & AdamW & 69.1 & 8.75 & 46.0 & 70.5 & 2.47 \\
& & Scaled AdamW & \underline{69.5} & \underline{8.80} & \underline{46.2} & \underline{70.9} & \underline{2.48} \\
& & LoRA-Pro AdamW & 69.2 & 8.73 & 45.9 & 70.8 & 2.47 \\
& & \textbf{FILO-AdafactorM (ours) } & \textbf{70.0} & \textbf{8.84} & \textbf{46.3} & \textbf{71.3} & \textbf{2.50} \\
\cmidrule{1-8}
\multirow{8}{*}{\shortstack{GPT-2\\medium}} & \multirow{8}{*}{4} & SGD & 66.6 & 8.54 & 44.2 & 68.2 & 2.32 \\
& & Scaled GD & 69.2 & 8.71 & 46.3 & \underline{70.9} & 2.48 \\
& & LoRA-Pro SGD & \underline{69.7} & \underline{8.77} & \underline{46.5} & \underline{70.9} & \underline{2.50} \\
& & \textbf{FILO-Adafactor (ours)} & \textbf{70.3} & \textbf{8.84} & \textbf{46.9} & \textbf{71.7} & \textbf{2.54} \\
\cmidrule{3-8}
& & AdamW & 68.9 & 8.69 & 46.5 & 71.3 & 2.51 \\
& & Scaled AdamW & 69.6 & 8.77 & \underline{46.6} & \textbf{71.8} & \underline{2.52} \\
& & LoRA-Pro AdamW & \underline{69.8} & \underline{8.78} & 46.5 & \underline{71.7} & \underline{2.52} \\
& & \textbf{FILO-AdafactorM (ours)} & \textbf{70.3} & \textbf{8.84} & \textbf{46.7} & \textbf{71.8} & \textbf{2.53} \\
\bottomrule
\end{tabular}
\end{table*}

Table~\ref{tab:score_GPT2} reports E2E scores at $r=4$ on both GPT-2 small and medium. FILO-Adafactor and FILO-AdafactorM achieve the best or tied-best score on every metric within their respective optimizer groups for both model sizes. To further validate the effectiveness of both variants, we also conducted GPT-2 fine-tuning experiments on DART~\citep{nan2021dart}. Both variants achieve the highest BLEU scores within their respective groups, as shown in Table~\ref{tab:gpt2-dart}. These results show that the gains extend to another generation benchmark. The cross-rank results also show that both variants retain the highest BLEU scores within their respective groups at $r \in \{16, 64\}$ on E2E, as reported in Table~\ref{tab:score_GPT2_rank_ablation} in Appendix~\ref{sec:gpt2-rank-ablation}.

Figure~\ref{fig:gpt2-loss-time} compares training loss against cumulative training time on GPT-2 medium / E2E at $r=4$. At a budget of $6000$~s, FILO-Adafactor and FILO-AdafactorM attain lower training loss than the baselines in their respective panels.

\begin{table}[t]
\centering
\caption{Scores of GPT-2 small model (rank=4) fine-tuned using different optimizers. Evaluation is conducted on DART dataset.}
\label{tab:gpt2-dart}
\begin{tabular}{lccccc}
\toprule
Methods
 & BLEU$\uparrow$  & METEOR$\uparrow$  & chrF++$\uparrow$  & TER$\downarrow$ & BLEURT$\uparrow$ \\
\midrule
SGD     & 41.2  &   0.63 &  0.59 & 0.52 & 0.33 \\
Scaled GD    &  43.8 & \textbf{0.66}  & 0.61 &  0.50 & 0.38 \\
LoRA-Pro SGD    &  44.1 &  \textbf{0.66} & 0.61  & 0.50 & 0.38 \\
\textbf{FILO-Adafactor (ours)}  & \textbf{44.6}  &  \textbf{0.66}  & \textbf{0.62}  & \textbf{0.49} & \textbf{0.39} \\
\midrule
AdamW     & 43.9  & 0.66 &  0.60 & 0.50 & 0.38 \\
Scaled AdamW    & 44.8  & \textbf{0.67} & \textbf{0.62} & 0.49 & \textbf{0.40 } \\ 
LoRA-Pro AdamW   & 44.9 & 0.66 & \textbf{0.62} & 0.50 & 0.39 \\
\textbf{FILO-AdafactorM (ours)}   & \textbf{45.4}  & \textbf{0.67}  & 0.60  &  \textbf{0.49} & \textbf{0.40} \\
\bottomrule
\end{tabular}
\end{table}

\begin{figure}[htbp]
\centering
\begin{subfigure}[t]{0.49\textwidth}
    \centering
    \includegraphics[width=\linewidth]{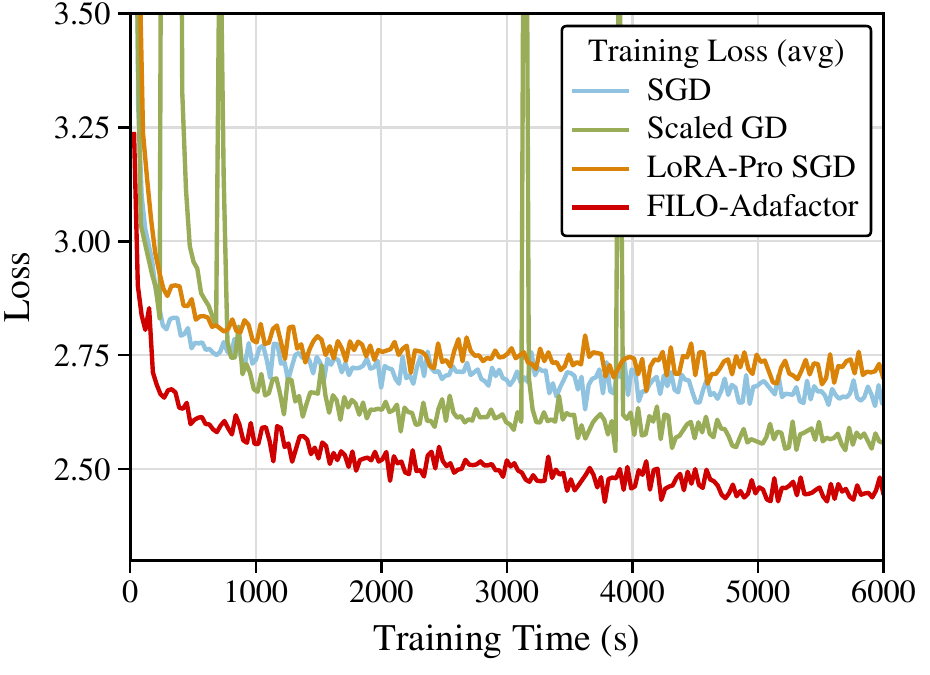}
    \caption{SGD-based methods.}
\end{subfigure}\hfill
\begin{subfigure}[t]{0.49\textwidth}
    \centering
    \includegraphics[width=\linewidth]{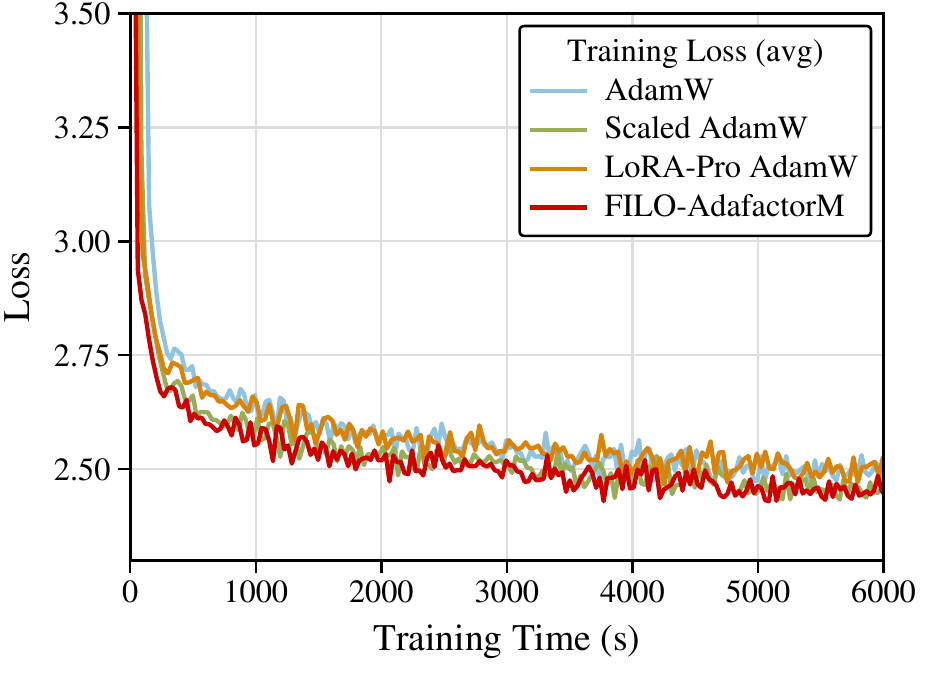}
    \caption{AdamW-based methods.}
\end{subfigure}
\caption{Training loss versus cumulative training time for GPT-2 medium on E2E ($r=4$). We divide consecutive optimizer updates into non-overlapping blocks of 100. Each point shows a block's mean training loss at the cumulative training time of its final update. For each method, we use the learning rate with the lowest validation loss among all the tested learning rates. These experiments were run on H800 GPUs.}
\label{fig:gpt2-loss-time}
\end{figure}

\subsection{Extension to 7B-scale LLMs: Mistral-7B and Qwen2-7B}\label{sec:llm-7b}

We next test whether the controlled-setting gains carry over to the 7B parameter scale, using Mistral-7B~\citep{jiang2023mistral} and Qwen2-7B~\citep{yang2024qwen2}. For Mistral-7B we fine-tune on the GLUE~\citep{wang2018glue} tasks RTE, CoLA, and MRPC, with three random seeds on RTE. For Qwen2-7B we additionally evaluate on the reasoning benchmark ARC~\citep{clark2018arc} and the math benchmark GSM8K~\citep{cobbe2021gsm8k}.

\begin{table*}[t]
\caption{Mistral-7B and Qwen2-7B fine-tuning accuracy (\%) with rank $r=8$. Bold/underline = best/second-best.}\label{tab:llm-7b}
\centering
\small
\begin{tabular}{lccc@{\hskip 16pt}cccc}
\toprule
 & \multicolumn{3}{c}{\textbf{Mistral-7B}} & \multicolumn{4}{c}{\textbf{Qwen2-7B}} \\
\cmidrule(lr){2-4}\cmidrule(lr){5-8}
Method & RTE & CoLA & MRPC & RTE & MRPC & ARC & GSM8K \\
\midrule
AdamW                & \underline{89.4} & 69.4             & 89.5             & \underline{90.6} & \underline{90.0} & 84.3             & 75.1 \\
Scaled AdamW         & 89.1             & \textbf{71.5}    & \underline{89.7} & \underline{90.6} & 87.0             & \underline{85.3} & 74.2 \\
LoRA-Pro AdamW       & 88.8             & 68.5             & 84.8             & 88.1             & 89.2             & 80.9             & \underline{75.7} \\
SOAP                 & 88.9             & 68.3             & 86.3             & 86.6             & 89.2             & 80.9             & 73.2 \\
\midrule
\textbf{FILO-AdafactorM (ours)} & \textbf{89.5} & \underline{71.4} & \textbf{90.0} & \textbf{91.0} & \textbf{90.4} & \textbf{85.6} & \textbf{76.4} \\
\bottomrule
\end{tabular}
\end{table*}

\begin{table*}[htbp]
\centering
\small
\begin{minipage}[c]{0.42\textwidth}
\centering
\caption{Per-step GPU time and peak GPU memory on Mistral-7B. FILO-Adafactor matches the LoRA-level memory footprint of vanilla LoRA (AdamW), while LoRA-Pro AdamW pays roughly $2\times$ peak memory for its full-weight gradient and moments.}\label{tab:mem-time}
\setlength{\tabcolsep}{4pt}
\begin{tabular}{@{}lcc@{}}
\toprule
Method & s/step & Mem (GB) \\
\midrule
AdamW               & 0.24 & 25.8 \\
Scaled AdamW        & 0.36 & 26.0 \\
SOAP                & 0.65 & 26.4 \\
LoRA-Pro AdamW      & 1.35 & 50.4 \\
\textbf{FILO-Adafactor}   & 0.46 & 21.5 \\
\textbf{FILO-AdafactorM} & 0.99 & 26.0 \\
\bottomrule
\end{tabular}
\end{minipage}
\hfill
\begin{minipage}[c]{0.56\textwidth}
\centering
\caption{CLIP and FID scores on Mix-of-Show diffusion fine-tuning. Bold/underline = best/second-best.}\label{tab:fid_clip}
\setlength{\tabcolsep}{4pt}
\begin{tabular}{lcccc}
\toprule
\multirow{2}{*}{Method}  & \multicolumn{2}{c}{scaling=0.7}  & \multicolumn{2}{c}{scaling=1} \\
\cmidrule(lr){2-3} \cmidrule(lr){4-5}
 & CLIP$\uparrow$  & FID$\downarrow$  & CLIP$\uparrow$  & FID$\downarrow$ \\
\midrule
SGD              & 27.79              & 69.90              & \underline{31.40} & 40.95 \\
Scaled GD        & \underline{31.23} & 35.86              & 30.60              & 29.62 \\
LoRA-Pro SGD     & \textbf{31.47}    & \underline{34.30} & 30.48              & \underline{29.19} \\
\textbf{FILO-Adafactor}   & \textbf{31.47} & \textbf{30.17} & \textbf{31.58} & \textbf{28.18} \\
\midrule
AdamW            & \textbf{31.47}    & 34.15              & \underline{30.68} & \underline{27.80} \\
Scaled AdamW     & 24.21              & 48.23              & 24.51              & 34.18 \\
LoRA-Pro AdamW   & \underline{31.04} & \underline{29.18} & 30.60              & 28.18 \\
\textbf{FILO-AdafactorM} & \textbf{31.47} & \textbf{29.01} & \textbf{30.73} & \textbf{27.13} \\
\bottomrule
\end{tabular}
\end{minipage}
\end{table*}

Table~\ref{tab:llm-7b} shows that FILO-AdafactorM achieves the best accuracy on six of the seven settings and ranks second on Mistral-7B CoLA, $0.1$ percentage point behind Scaled AdamW. Notably, LoRA-Pro AdamW incurs $\mathcal{O}(mn)$ memory to store the full-weight gradient and its first and second moments, yet it trails AdamW on multiple Qwen2-7B tasks. To compare with an optimizer based on Kronecker preconditioning, we also include SOAP~\citep{vyas2025soap} in the 7B experiments. FILO-AdafactorM achieves higher accuracy than SOAP in all seven settings. Table~\ref{tab:mem-time} contrasts per-step GPU time and peak memory on Mistral-7B GLUE-RTE: FILO-AdafactorM matches Scaled AdamW's peak memory of $26.0$ GB, while LoRA-Pro AdamW requires $50.4$ GB, nearly twice as much.

\subsection{Diffusion model personalization (Mix-of-Show)}\label{sec:main_part_diffusion}

To test transfer beyond NLP, we evaluate both variants of FILO-Adafactor on diffusion-model personalization with the Mix-of-Show framework~\citep{gu2023mix-of-show}, which uses Embedding Decomposed LoRA (EDLoRA) on the text encoder and U-Net of a Stable Diffusion backbone. Setup follows~\citep{zhang2024RiemannianPreconditioned, gu2023mix-of-show} with embedding tuning disabled. We report CLIP score~\citep{hessel2021clipscore} and FID~\citep{heusel2017fid-distance} at LoRA scaling factors $0.7$ and $1.0$. Higher CLIP scores indicate better alignment with the prompt, while lower FID scores indicate greater distributional similarity to reference images.

Table~\ref{tab:fid_clip} shows that both variants of FILO-Adafactor achieve the lowest FID at every scaling-optimizer combination and the best CLIP at scaling $1.0$ in both SGD- and AdamW-based families, while remaining competitive with the best baseline at scaling $0.7$. This confirms that the gains from gradient-statistics-aware preconditioning carry over from text generation to image generation. Qualitative samples featuring Harry Potter and Hermione Granger are in Appendix~\ref{sec:diffusion}, along with per-prompt grids across LoRA scaling factors.

\section{Conclusion}
We introduced \textbf{FILO}, organizing existing LoRA optimizers along two axes: (i) the choice of $\bm{W}$-space preconditioner $\mathcal{F}_t$ (\S~\ref{subsec:lora-precond-related}, Table~\ref{tab:lora-precond-compare}), and (ii) the factor-space inversion strategy for the singular operator $J_\mathcal{G}^* \mathcal{F}_t J_\mathcal{G}$. This framework exposes a previously underexplored design point: pairing a non-trivial gradient-statistics-aware $\mathcal{F}_t$ with LoRA's $\mathcal{O}((m+n)r)$ optimizer-state budget. We instantiate FILO as \textbf{FILO-Adafactor}, combining an Adafactor diagonal Kronecker preconditioner $\mathcal{H}_t$ on $\bm{W}$ with an $\mathcal{H}_t$-balance criterion that selects a unique factor update from the affine solution set of~\eqref{eq:factor-system}. For the instantaneous-gradient solve, the induced first-order weight update is the closest LoRA approximation to the preconditioned $\bm{W}$-space direction under the $\mathcal{H}_t$-weighted norm, and admits a closed-form expression. Across GPT-2 on E2E, Mistral-7B and Qwen2-7B on GLUE, ARC, and GSM8K, and Mix-of-Show diffusion personalization, FILO-Adafactor and FILO-AdafactorM are competitive with or improve over a representative set of LoRA optimizers while keeping peak GPU memory at the LoRA optimizer level.

\newpage
\section*{Acknowledgment}
The work is partially supported by Hong Kong Research Grant Council GRF grants 16307023, 16306124, and 16307325.

{
\bibliographystyle{plainnat}
\bibliography{main}

@incollection{absil2009optimization,
  title={Optimization algorithms on matrix manifolds},
  author={Absil, P-A and Mahony, Robert and Sepulchre, Rodolphe},
  booktitle={Optimization Algorithms on Matrix Manifolds},
  year={2009},
  publisher={Princeton University Press}
}

@article{bian2024preconditioned,
  title={A preconditioned Riemannian gradient descent algorithm for low-rank matrix recovery},
  author={Bian, Fengmiao and Cai, Jian-Feng and Zhang, Rui},
  journal={SIAM Journal on Matrix Analysis and Applications},
  volume={45},
  number={4},
  pages={2075--2103},
  year={2024},
  publisher={SIAM}
}

@inproceedings{hu2022lora,
  title={Lora: Low-rank adaptation of large language models.},
  author={Hu, Edward J and Shen, Yelong and Wallis, Phillip and Allen-Zhu, Zeyuan and Li, Yuanzhi and Wang, Shean and Wang, Lu and Chen, Weizhu and others},
  booktitle={International Conference on Learning Representations (ICLR)},
  pages={3},
  year={2022}
}

@inproceedings{zhang2024RiemannianPreconditioned,
  title={Riemannian Preconditioned LoRA for Fine-Tuning Foundation Models},
  author={Zhang, Fangzhao and Pilanci, Mert},
  booktitle={International Conference on Machine Learning (ICML)},
  year={2024}
}

@inproceedings{gupta2018shampoo,
  title={Shampoo: Preconditioned stochastic tensor optimization},
  author={Gupta, Vineet and Koren, Tomer and Singer, Yoram},
  booktitle={International Conference on Machine Learning (ICML)},
  pages={1842--1850},
  year={2018},
  organization={PMLR}
}

@article{duchi2011AdaGrad,
  title={Adaptive subgradient methods for online learning and stochastic optimization.},
  author={Duchi, John and Hazan, Elad and Singer, Yoram},
  journal={Journal of Machine Learning Research},
  volume={12},
  number={7},
  year={2011}
}

@article{kingma2014adam,
  title={Adam: A method for stochastic optimization},
  author={Kingma, Diederik P and Ba, Jimmy},
  journal={arXiv preprint arXiv:1412.6980},
  year={2014}
}

@article{achiam2023gpt4,
  title={Gpt-4 technical report},
  author={Achiam, Josh and Adler, Steven and Agarwal, Sandhini and Ahmad, Lama and Akkaya, Ilge and Aleman, Florencia Leoni and Almeida, Diogo and Altenschmidt, Janko and Altman, Sam and Anadkat, Shyamal and others},
  journal={arXiv preprint arXiv:2303.08774},
  year={2023}
}

@article{radford2019gpt2,
  title={Language models are unsupervised multitask learners},
  author={Radford, Alec and Wu, Jeffrey and Child, Rewon and Luan, David and Amodei, Dario and Sutskever, Ilya and others},
  journal={OpenAI blog},
  volume={1},
  number={8},
  pages={9},
  year={2019}
}

@article{yang2024qwen2,
  title={Qwen2.5 technical report},
  author={Yang, An and Yang, Baosong and Zhang, Beichen and Hui, Binyuan and Zheng, Bo and Yu, Bowen and Li, Chengyuan and Liu, Dayiheng and Huang, Fei and Wei, Haoran and others},
  journal={arXiv preprint arXiv:2412.15115},
  year={2024}
}

@article{liu2024deepseek,
  title={Deepseek-v3 technical report},
  author={Liu, Aixin and Feng, Bei and Xue, Bing and Wang, Bingxuan and Wu, Bochao and Lu, Chengda and Zhao, Chenggang and Deng, Chengqi and Zhang, Chenyu and Ruan, Chong and others},
  journal={arXiv preprint arXiv:2412.19437},
  year={2024}
}

@inproceedings{gu2023mix-of-show,
  title={Mix-of-show: Decentralized low-rank adaptation for multi-concept customization of diffusion models},
  author={Gu, Yuchao and Wang, Xintao and Wu, Jay Zhangjie and Shi, Yujun and Chen, Yunpeng and Fan, Zihan and Xiao, Wuyou and Zhao, Rui and Chang, Shuning and Wu, Weijia and others},
  booktitle={Advances in Neural Information Processing Systems (NeurIPS)},
  volume={36},
  pages={15890--15902},
  year={2023}
}

@inproceedings{hayou2024lora+,
  title={LoRA+: Efficient Low Rank Adaptation of Large Models},
  author={Hayou, Soufiane and Ghosh, Nikhil and Yu, Bin},
  booktitle={International Conference on Machine Learning (ICML)},
  pages={17783--17806},
  year={2024},
  organization={PMLR}
}

@inproceedings{zhao2024galore,
  title={GaLore: Memory-Efficient LLM Training by Gradient Low-Rank Projection},
  author={Zhao, Jiawei and Zhang, Zhenyu and Chen, Beidi and Wang, Zhangyang and Anandkumar, Anima and Tian, Yuandong},
  booktitle={International Conference on Machine Learning (ICML)},
  pages={61121--61143},
  year={2024},
  organization={PMLR}
}

@inproceedings{morwani2024new_shampoo,
 author = {Morwani, Depen and Shapira, Itai and Vyas, Nikhil and Malach, Eran and Kakade, Sham and Janson, Lucas},
 booktitle = {International Conference on Learning Representations (ICLR)},
 editor = {Y. Yue and A. Garg and N. Peng and F. Sha and R. Yu},
 pages = {5802--5822},
 title = {A New Perspective on Shampoo\textquotesingle s Preconditioner},
 url = {https://proceedings.iclr.cc/paper_files/paper/2025/file/11332b6b6cf4485b84afadb1352d3a9a-Paper-Conference.pdf},
 volume = {2025},
 year = {2025}
}

@inproceedings{mo2025loro,
  title={Parameter and memory efficient pretraining via low-rank riemannian optimization},
  author={Mo, Zhanfeng and Huang, Long-Kai and Pan, Sinno Jialin},
  booktitle={International Conference on Learning Representations (ICLR)},
  year={2025}
}

@inproceedings{vyas2025soap,
  title={SOAP: Improving and Stabilizing Shampoo using Adam for Language Modeling},
  author={Vyas, Nikhil and Morwani, Depen and Zhao, Rosie and Shapira, Itai and Brandfonbrener, David and Janson, Lucas and Kakade, Sham M},
  booktitle={International Conference on Learning Representations (ICLR)},
  year={2025}
}

@article{novikova2017e2e,
  title={The E2E dataset: New challenges for end-to-end generation},
  author={Novikova, Jekaterina and Du{\v{s}}ek, Ond{\v{r}}ej and Rieser, Verena},
  journal={arXiv preprint arXiv:1706.09254},
  year={2017}
}

@inproceedings{wang2025lora-pro,
  title={LoRA-Pro: Are Low-Rank Adapters Properly Optimized?},
  author={Wang, Zhengbo and Liang, Jian and He, Ran and Wang, Zilei and Tan, Tieniu},
  booktitle={International Conference on Learning Representations(ICLR)},
  year={2025}
}

@inproceedings{paszke2019pytorch,
  title={Pytorch: An imperative style, high-performance deep learning library},
  author={Paszke, Adam and Gross, Sam and Massa, Francisco and Lerer, Adam and Bradbury, James and Chanan, Gregory and Killeen, Trevor and Lin, Zeming and Gimelshein, Natalia and Antiga, Luca and others},
  booktitle={Advances in neural information processing systems (NeurIPS)},
  volume={32},
  year={2019}
}

@inproceedings{martens2015K-FAC,
  title={Optimizing neural networks with kronecker-factored approximate curvature},
  author={Martens, James and Grosse, Roger},
  booktitle={International conference on machine learning (ICML)},
  pages={2408--2417},
  year={2015},
  organization={PMLR}
}

@inproceedings{eschenhagen2023NeurIPS-K-FAC,
  title={Kronecker-factored approximate curvature for modern neural network architectures},
  author={Eschenhagen, Runa and Immer, Alexander and Turner, Richard and Schneider, Frank and Hennig, Philipp},
  booktitle={Advances in Neural Information Processing Systems (NeurIPS)},
  volume={36},
  pages={33624--33655},
  year={2023}
}

@inproceedings{shazeer2018adafactor,
  title={Adafactor: Adaptive learning rates with sublinear memory cost},
  author={Shazeer, Noam and Stern, Mitchell},
  booktitle={International Conference on Machine Learning (ICML)},
  pages={4596--4604},
  year={2018},
  organization={PMLR}
}

@inproceedings{zhu2024imbalanceLoRA,
  title={Imbalance-Regularized LoRA: A Plug-and-Play Method for Improving Fine-Tuning of Foundation Models},
  author={Zhu, Zhenyu and Wu, Yongtao and Gu, Quanquan and Cevher, Volkan},
  booktitle={Adaptive Foundation Models: Evolving AI for Personalized and Efficient Learning},
  year={2024}
}

@article{wei2016RGD_recovery,
  title={Guarantees of Riemannian optimization for low rank matrix recovery},
  author={Wei, Ke and Cai, Jian-Feng and Chan, Tony F and Leung, Shingyu},
  journal={SIAM Journal on Matrix Analysis and Applications},
  volume={37},
  number={3},
  pages={1198--1222},
  year={2016},
  publisher={SIAM}
}

@inproceedings{zhang2025lora_one,
  title={LoRA-One: One-Step Full Gradient Could Suffice for Fine-Tuning Large Language Models, Provably and Efficiently},
  author={Zhang, Yuanhe and Liu, Fanghui and Chen, Yudong},
  booktitle={International Conference on Machine Learning (ICML)},
  year={2025}
}

@inproceedings{wang2024lora-ga,
  title={Lora-ga: Low-rank adaptation with gradient approximation},
  author={Wang, Shaowen and Yu, Linxi and Li, Jian},
  booktitle={Advances in Neural Information Processing Systems (NeurIPS)},
  volume={37},
  pages={54905--54931},
  year={2024}
}

@inproceedings{hessel2021clipscore,
  title={CLIPScore: A Reference-free Evaluation Metric for Image Captioning},
  author={Hessel, Jack and Holtzman, Ari and Forbes, Maxwell and Le Bras, Ronan and Choi, Yejin},
  booktitle={EMNLP (1)},
  year={2021}
}

@inproceedings{radford2021clip-model,
  title={Learning transferable visual models from natural language supervision},
  author={Radford, Alec and Kim, Jong Wook and Hallacy, Chris and Ramesh, Aditya and Goh, Gabriel and Agarwal, Sandhini and Sastry, Girish and Askell, Amanda and Mishkin, Pamela and Clark, Jack and others},
  booktitle={International conference on machine learning (ICML)},
  pages={8748--8763},
  year={2021},
  organization={PMLR}
}

@inproceedings{heusel2017fid-distance,
  title={Gans trained by a two time-scale update rule converge to a local nash equilibrium},
  author={Heusel, Martin and Ramsauer, Hubert and Unterthiner, Thomas and Nessler, Bernhard and Hochreiter, Sepp},
  booktitle={Advances in neural information processing systems (NeurIPS)},
  volume={30},
  year={2017}
}

@inproceedings{nan2021dart,
  title={Dart: Open-domain structured data record to text generation},
  author={Nan, Linyong and Radev, Dragomir and Zhang, Rui and Rau, Amrit and Sivaprasad, Abhinand and Hsieh, Chiachun and Tang, Xiangru and Vyas, Aadit and Verma, Neha and Krishna, Pranav and others},
  booktitle={Proceedings of the 2021 Conference of the North American Chapter of the Association for Computational Linguistics: Human Language Technologies},
  pages={432--447},
  year={2021}
}

@inproceedings{yen2025lorarite,
  title={LoRA Done RITE: Robust Invariant Transformation Equilibration for LoRA Optimization},
  author={Yen, Jui-Nan and Si, Si and Meng, Zhao and Yu, Felix and Duvvuri, Sai Surya and Dhillon, Inderjit S and Hsieh, Cho-Jui and Kumar, Sanjiv},
  booktitle={The Thirteenth International Conference on Learning Representations},
  year={2025}
}

@article{bogachev2025riemannian,
  title={LoRA meets Riemannion: Muon Optimizer for Parametrization-independent Low-Rank Adapters},
  author={Bogachev, Vladimir and Aletov, Vladimir and Molozhavenko, Alexander and Bobkov, Denis and Soboleva, Vera and Alanov, Aibek and Rakhuba, Maxim},
  journal={arXiv preprint arXiv:2507.12142},
  year={2025}
}

@inproceedings{bian2026finding,
 author = {Bian, Fengmiao and ZHENG, Jinyang and Liu, Ziyun and Luo, Jianzhou and CAI, Jian-Feng},
 booktitle = {Advances in Neural Information Processing Systems (NeurIPS)},
 doi = {10.52202/085713-2004},
 editor = {D. Belgrave and C. Zhang and H. Lin and R. Pascanu and P. Koniusz and M. Ghassemi and N. Chen},
 pages = {59916--59950},
 publisher = {Curran Associates, Inc.},
 title = {Finding Low-Rank Matrix Weights in DNNs via Riemannian Optimization: RAdaGrad and RAdamW},
 url = {https://proceedings.neurips.cc/paper_files/paper/2025/file/5679173c400b332796426e443ab5ea0d-Paper-Conference.pdf},
 volume = {38, Main Conference},
 year = {2025}
}

@misc{jiang2023mistral,
      title={Mistral 7B}, 
      author={Albert Q. Jiang and Alexandre Sablayrolles and Arthur Mensch and Chris Bamford and Devendra Singh Chaplot and Diego de las Casas and Florian Bressand and Gianna Lengyel and Guillaume Lample and Lucile Saulnier and Lélio Renard Lavaud and Marie-Anne Lachaux and Pierre Stock and Teven Le Scao and Thibaut Lavril and Thomas Wang and Timothée Lacroix and William El Sayed},
      year={2023},
      eprint={2310.06825},
      archivePrefix={arXiv},
      primaryClass={cs.CL},
      url={https://arxiv.org/abs/2310.06825}, 
}

@misc{wang2018glue,
      title={GLUE: A Multi-Task Benchmark and Analysis Platform for Natural Language Understanding}, 
      author={Alex Wang and Amanpreet Singh and Julian Michael and Felix Hill and Omer Levy and Samuel R. Bowman},
      year={2019},
      eprint={1804.07461},
      archivePrefix={arXiv},
      primaryClass={cs.CL},
      url={https://arxiv.org/abs/1804.07461}, 
}

@misc{clark2018arc,
      title={Think you have Solved Question Answering? Try {ARC}, the {AI2} Reasoning Challenge}, 
      author={Peter Clark and Isaac Cowhey and Oren Etzioni and Tushar Khot and Ashish Sabharwal and Carissa Schoenick and Oyvind Tafjord},
      year={2018},
      eprint={1803.05457},
      archivePrefix={arXiv},
      primaryClass={cs.AI},
      url={https://arxiv.org/abs/1803.05457}, 
}

@misc{cobbe2021gsm8k,
      title={Training Verifiers to Solve Math Word Problems}, 
      author={Karl Cobbe and Vineet Kosaraju and Mohammad Bavarian and Mark Chen and Heewoo Jun and Lukasz Kaiser and Matthias Plappert and Jerry Tworek and Jacob Hilton and Reiichiro Nakano and Christopher Hesse and John Schulman},
      year={2021},
      eprint={2110.14168},
      archivePrefix={arXiv},
      primaryClass={cs.LG},
      url={https://arxiv.org/abs/2110.14168}, 
}

@article{kunstner2019limitations,
  title={Limitations of the empirical fisher approximation for natural gradient descent},
  author={Kunstner, Frederik and Hennig, Philipp and Balles, Lukas},
  journal={Advances in neural information processing systems},
  volume={32},
  year={2019}
}

@inproceedings{tu2016low,
  title={Low-rank solutions of linear matrix equations via procrustes flow},
  author={Tu, Stephen and Boczar, Ross and Simchowitz, Max and Soltanolkotabi, Mahdi and Recht, Ben},
  booktitle={International conference on machine learning},
  pages={964--973},
  year={2016},
  organization={PMLR}
}

@InProceedings{pmlr-v97-kasai19a,
  title = 	 {{R}iemannian adaptive stochastic gradient algorithms on matrix manifolds},
  author =       {Kasai, Hiroyuki and Jawanpuria, Pratik and Mishra, Bamdev},
  booktitle = 	 {Proceedings of the 36th International Conference on Machine Learning},
  pages = 	 {3262--3271},
  year = 	 {2019},
  editor = 	 {Chaudhuri, Kamalika and Salakhutdinov, Ruslan},
  volume = 	 {97},
  series = 	 {Proceedings of Machine Learning Research},
  month = 	 {09--15 Jun},
  publisher =    {PMLR},
  url = 	 {https://proceedings.mlr.press/v97/kasai19a.html}
}
}


\newpage
\appendix

\addtocontents{toc}{\protect\setcounter{tocdepth}{2}}
{
  \hypersetup{linkcolor=black}
  \tableofcontents
}

\section{Limitations and Broader Impact}

\paragraph{Limitations.} The rank-$r$ constraint in low-rank adaptation methods inherently limits their expressiveness compared to full-parameter fine-tuning. 

\paragraph{Broader Impact.} The proposed methods demonstrate broad compatibility across diverse DNN architectures, which is useful to fields like healthcare, farming, and schools, especially in places with limited resources.

\section{Notation}\label{sec:notation}

Table~\ref{tab:notation} summarizes the notation used throughout the paper. Calligraphic letters denote linear operators on $\mathbb{R}^{m\times n}$ (e.g., $\mathcal{F}_t, \mathcal{H}_t$), while bold letters denote matrices (e.g., $\bm{B}_t, \bm{A}_t, \bm{L}_t, \bm{R}_t, \bm{G}_t$).

\begin{table}[h]
\centering
\caption{Notation used throughout the paper.}
\label{tab:notation}
\small
\begin{tabular}{ll}
\toprule
\textbf{Symbol} & \textbf{Meaning} \\
\midrule
\multicolumn{2}{l}{\emph{Dimensions and indices}} \\
$m, n$ & Output and input dimensions of the adapted weight $\bm{W} \in \mathbb{R}^{m\times n}$ \\
$r$ & LoRA rank, $r \ll \min(m,n)$ \\
$t$ & Optimization step index \\
\midrule
\multicolumn{2}{l}{\emph{Matrices (in bold)}} \\
$\bm{W}_0, \bm{W}_t$ & Frozen pretrained weight; LoRA increment at step $t$, $\bm{W}_t = \bm{B}_t \bm{A}_t$ \\
$\bm{B}_t, \bm{A}_t$ & LoRA factors, $\bm{B}_t \in \mathbb{R}^{m\times r}, \bm{A}_t \in \mathbb{R}^{r\times n}$ \\
$\bm{G}_t$ & Stochastic gradient $\nabla_{\bm{W}} \mathcal{L}(\bm{W}_0 + \bm{W}_t)$ on $\bm{W}$ \\
$\bm{G}_{\bm{B}_t}, \bm{G}_{\bm{A}_t}$ & Factor gradients $\bm{G}_t \bm{A}_t^\top$ and $\bm{B}_t^\top \bm{G}_t$ \\
$\bm{\Delta}_{\bm{B}_t}, \bm{\Delta}_{\bm{A}_t}$ & Factor updates: $\bm{B}_{t+1} = \bm{B}_t - \eta_t \bm{\Delta}_{\bm{B}_t}$, similarly for $\bm{A}$ \\
$\bm{\Delta}_t$ & Induced $\bm{W}$-update $\bm{\Delta}_{\bm{B}_t} \bm{A}_t + \bm{B}_t \bm{\Delta}_{\bm{A}_t}$ \\
$\bm{l}_t, \bm{r}_t$ & Adafactor row/column statistics, $\bm{l}_t \in \mathbb{R}^m, \bm{r}_t \in \mathbb{R}^n$ \\
$\bm{L}_t, \bm{R}_t$ & Adafactor diagonal preconditioner factors (Eq.~\eqref{eq:def_LR}) \\
$\bm{L}_{{\rm Sh},t}, \bm{R}_{{\rm Sh},t}$ & Shampoo full-covariance factors $\sum_s \bm{G}_s \bm{G}_s^\top$, $\sum_s \bm{G}_s^\top \bm{G}_s$ \\
$\bm{X}_t$ & Free $r\times r$ matrix parameterizing the $\ker(J_\mathcal{G})$ orbit (Lemma~\ref{lem:find_delta_B_A_with_X}) \\
$\widetilde{\bm{P}}_{B_t}, \widetilde{\bm{Q}}_{A_t}$ & Auxiliary $\bm{L}_t^{1/2}/\bm{R}_t^{1/2}$-weighted projector matrices in Theorem~\ref{thm:find_delta_B_A} \\
\midrule
\multicolumn{2}{l}{\emph{Operators (in calligraphic)}} \\
$\mathcal{F}_t$ & Preconditioner on $\bm{W}$, $\mathcal{F}_t : \mathbb{R}^{m\times n} \to \mathbb{R}^{m\times n}$ \\
$\mathcal{H}_t$ & Preconditioner used in FILO-Adafactor, $\mathcal{H}_t \bm{Y} = \bm{L}_t^{1/2} \bm{Y} \bm{R}_t^{1/2}$ \\
$\mathcal{E}_t$ & Empirical Fisher on the factor space $[\bm{B}_t, \bm{A}_t]$, related to $\mathcal{F}_t$ via $\mathcal{E}_t = J_\mathcal{G}^* \mathcal{F}_t J_\mathcal{G}$ (Eq.~\eqref{eq:factor-precond-pullback}) \\
$\widetilde{\mathcal{P}}_{\mathbb{T}_t}, \widetilde{\mathcal{P}}_{\mathbb{T}_t}^\perp$ & $\mathcal{H}_t$-orthogonal projector onto $\mathbb{T}_t$ and its complement \\
$\mathcal{P}_{\mathbb{T}_t}$ & Frobenius (Euclidean) orthogonal projector onto $\mathbb{T}_t$ \\
$J_\mathcal{G}, J_\mathcal{G}^*, J_\mathcal{G}^\dag$ & Jacobian of $\mathcal{G}: [\bm{B}, \bm{A}] \mapsto \bm{B}\bm{A}$, its adjoint, and Moore--Penrose pseudoinverse \\
\midrule
\multicolumn{2}{l}{\emph{Subspaces and inner products}} \\
$\mathbb{T}_t$ & $\mathrm{Im}(J_\mathcal{G}) = \{\bm{P}\bm{A}_t + \bm{B}_t \bm{Q}\} \subset \mathbb{R}^{m\times n}$, range of one LoRA step \\
$\mathcal{M}_r$ & Manifold of rank-$r$ matrices in $\mathbb{R}^{m\times n}$ \\
$\langle \cdot, \cdot \rangle$ & Frobenius inner product on $\mathbb{R}^{m\times n}$ \\
$\langle \cdot, \cdot \rangle_{\mathcal{H}_t}, \|\cdot\|_{\mathcal{H}_t}$ & $\mathcal{H}_t$-weighted inner product and norm (Eq.~\eqref{eq:weighted-inner-product}) \\
$\widetilde{\bm{G}}_t$ & Adafactor-preconditioned direction $\mathcal{H}_t^{-1} \bm{G}_t = \bm{L}_t^{-1/2} \bm{G}_t \bm{R}_t^{-1/2}$ \\
\midrule
\multicolumn{2}{l}{\emph{Other}} \\
$\eta_t$ & Learning rate at step $t$ \\
$\beta_1, \beta_2$ & EMA decay rates for Adafactor statistics $\bm{l}_t, \bm{r}_t$ \\
$\odot$ & Hadamard (elementwise) product \\
$\diag(\bm{v})$ & Diagonal matrix with $\bm{v}$ on the diagonal \\
$\mathcal{L}$ & Loss function \\
\bottomrule
\end{tabular}
\end{table}

\section{Proof of Theoretical Results}

\subsection{Computation of Jacobian}
\begin{proposition}[Computation of $J_\mathcal{G}$ and $J_\mathcal{G}^*$]\label{prop:Jacobian-comp}
Let $[\bm{B}, \bm{A}]$ be a pair of low-rank factors with $\bm{B} \in \mathbb{R}^{m \times r}, \bm{A} \in \mathbb{R}^{r \times n}$. Define the generator $\mathcal{G}: [\mathbb{R}^{m \times r}, \mathbb{R}^{r \times n}] \rightarrow \mathbb{R}^{m \times n}$ by $\mathcal{G}([\bm{B},\bm{A}]) = \bm{B}\bm{A}$. Denote the Jacobian of $\mathcal{G}$ by $J_{\mathcal{G}}$ and its adjoint by $J_{\mathcal{G}}^*$. Then, for any $[\bm{P}, \bm{Q}] \in [\mathbb{R}^{m \times r}, \mathbb{R}^{r \times n}]$ and any $\bm{C} \in \mathbb{R}^{m \times n}$,
\begin{itemize}
\item $J_\mathcal{G}([\bm{B},\bm{A}])[\bm{P}, \bm{Q}]= \bm{PA}+ \bm{BQ}$,
\item $J_\mathcal{G}^*([\bm{B},\bm{A}])(\bm{C})=[\bm{C} \bm{A}^\top, \bm{B}^\top \bm{C}]$,
\item $J_\mathcal{G}([\bm{B},\bm{A}]) J_\mathcal{G}^*([\bm{B},\bm{A}])(\bm{C}) = \bm{C} \bm{A}^\top \bm{A} + \bm{B} \bm{B}^\top \bm{C}$.
\item $J_\mathcal{G}^*([\bm{B},\bm{A}]) J_\mathcal{G}([\bm{B},\bm{A}])[\bm{P}, \bm{Q}] = [\bm{P}\bm{A}\bm{A}^\top + \bm{B}\bm{Q}\bm{A}^\top,\; \bm{B}^\top\bm{P}\bm{A} + \bm{B}^\top\bm{B}\bm{Q}]$.
\end{itemize}
\end{proposition}

\begin{proof}
The Jacobian operator $J_{\mathcal{G}}([\bm{B}, \bm{A}])[\bm{P}, \bm{Q}]: [\mathbb{R}^{m \times r}, \mathbb{R}^{r \times n}] \rightarrow \mathbb{R}^{m \times n}$ represents the derivative of $\mathcal{G}$ at $[\bm{B}, \bm{A}]$ along the direction $[\bm{P}, \bm{Q}]$. Similarly, $J_{\mathcal{G}}^*([\bm{B}, \bm{A}])(\bm{C}): \mathbb{R}^{m \times n} \rightarrow [\mathbb{R}^{m \times r}, \mathbb{R}^{r \times n}]$ is the adjoint  of $J_{\mathcal{G}}$ at $[\bm{B}, \bm{A}]$ along the direction $\bm{C}$. For more details, see~\citep[\S~6.1]{absil2009optimization}.

\begin{itemize}
\item[(i)] The computation of $J_{\mathcal{G}}$. Let $\bm{B}(t): \mathbb{R} \rightarrow \mathbb{R}^{m \times r}$  and $\bm{A}(t): \mathbb{R} \rightarrow \mathbb{R}^{r \times n}$ be differentiable curves with $\bm{B}(0) = \bm{B}$ and $\bm{A}(0) = \bm{A}$. By the chain rule, the Jacobian of $\mathcal{G}$ at $[\bm{B}, \bm{A}]$ along these curves is
\[
\begin{split}
    \eval{J_{\mathcal{G}}([\bm{B}(t),\bm{A}(t)])[\dot{\bm{B}}(t), \dot{\bm{A}}(t)]}_{t=0} & = \eval{\left[\frac{\mathrm{d} \mathcal{G}([\bm{B},\bm{A}])}{\mathrm{d}\bm{B}} \right] \dot{\bm{B}}(t)}_{t=0} + \eval{\left[\frac{\mathrm{d} \mathcal{G}([\bm{B},\bm{A}])}{\mathrm{d}\bm{A}} \right] \dot{\bm{A}}(t)}_{t=0} \\
    & = \eval{\dot{\bm{B}}(t) \bm{A}(t)}_{t=0} + \eval{\bm{B}(t) \dot{\bm{A}}(t)}_{t=0} \\
    & = \dot{\bm{B}}(0) \bm{A} + \bm{B} \dot{\bm{A}}(0),
\end{split}
\]
where $\dot{\bm{B}}(t)$ and $\dot{\bm{A}}(t)$ denote the derivatives of $\bm{B}(t)$ and $\bm{A}(t)$ with respect to $t$. The second line follows because $\mathcal{G}([\bm{B}, \bm{A}])=\bm{B}\bm{A}$, hence $\frac{\mathrm{d} \mathcal{G}([\bm{B},\bm{A}])}{\mathrm{d}\bm{B}}$ and $\frac{\mathrm{d} \mathcal{G}([\bm{B},\bm{A})]}{\mathrm{d}\bm{A}}$ are both linear operators.

Since $\dot{\bm{B}}(0)$ and $\dot{\bm{A}}(0)$ are arbitrary, for any $[\bm{P}, \bm{Q}] \in [\mathbb{R}^{m \times r}, \mathbb{R}^{r \times n}]$, we obtain
\[
    J_{\mathcal{G}}([\bm{B},\bm{A}])[\bm{P}, \bm{Q}] = \bm{PA}+\bm{BQ}.
\]
\item[(ii)] The computation of $J_{\mathcal{G}}^*$. For brevity, write $J_{\mathcal{G}}[\bm{P}, \bm{Q}]$ for $J_{\mathcal{G}}([\bm{B},\bm{A}])[\bm{P}, \bm{Q}]$ and $J_{\mathcal{G}}^*(\bm{C})$ for $J_{\mathcal{G}}^*([\bm{B}, \bm{A}])(\bm{C})$. By definition of the adjoint (with respect to the Frobenius inner product), for any $[\bm{P}, \bm{Q}] \in (\mathbb{R}^{m \times r}, \mathbb{R}^{r \times n})$ and $\bm{C} \in \mathbb{R}^{m \times n} $,
\[
    \big\langle J_\mathcal{G}[\bm{P}, \bm{Q}], \bm{C} \big\rangle = \big\langle [\bm{P}, \bm{Q}], J_\mathcal{G}^*(\bm{C}) \big\rangle.
\]
For the left-hand side,
\[
\begin{split}
    \big\langle J_\mathcal{G}[\bm{P}, \bm{Q}], \bm{C} \big\rangle 
    & = \big\langle \bm{PA}+ \bm{BQ}, \bm{C} \big\rangle \\
    & = \big\langle \bm{PA}, \bm{C} \big\rangle + \big\langle \bm{BQ}, \bm{C} \big\rangle \\
    & = \big\langle \bm{P}, \bm{C}\bm{A}^\top \big\rangle + \big\langle \bm{Q}, \bm{B}^\top\bm{C} \big\rangle.
\end{split}
\]
For the right-hand side,  writing $J_\mathcal{G}^*(\bm{C}) = [\bm{C}_1, \bm{C}_2]$, then
\[
\begin{split}
    \big\langle [\bm{P}, \bm{Q}], J_\mathcal{G}^*(\bm{C}) \big\rangle 
    & = \big\langle [\bm{P}, \bm{Q}], [\bm{C}_1, \bm{C}_2] \big\rangle \\
    & = \big\langle \bm{P}, \bm{C}_1 \big\rangle + \big\langle \bm{Q}, \bm{C}_2 \big\rangle.
\end{split}
\]
Hence $\bm{C}_1 = \bm{C}\bm{A}^\top$ and $\bm{C}_2 = \bm{B}^\top\bm{C}$, and therefore $J_\mathcal{G}^*([\bm{B},\bm{A}])(\bm{C})=[\bm{C} \bm{A}^\top, \bm{B}^\top \bm{C}]$.

\item[(iii)] By (i)--(ii), $J_\mathcal{G}([\bm{B},\bm{A}]) J_\mathcal{G}^*([\bm{B},\bm{A}])(\bm{C})=J_\mathcal{G}([\bm{B},\bm{A}])[\bm{C} \bm{A}^\top, \bm{B}^\top \bm{C}] = \bm{C} \bm{A}^\top \bm{A} + \bm{B} \bm{B}^\top \bm{C}$ as claimed.

\item[(iv)] By (i)--(ii), $J_\mathcal{G}^*([\bm{B},\bm{A}]) J_\mathcal{G}([\bm{B},\bm{A}])[\bm{P}, \bm{Q}] = J_\mathcal{G}^*([\bm{B},\bm{A}])(\bm{P}\bm{A} + \bm{B}\bm{Q}) = [(\bm{P}\bm{A} + \bm{B}\bm{Q})\bm{A}^\top,\; \bm{B}^\top(\bm{P}\bm{A} + \bm{B}\bm{Q})] = [\bm{P}\bm{A}\bm{A}^\top + \bm{B}\bm{Q}\bm{A}^\top,\; \bm{B}^\top\bm{P}\bm{A} + \bm{B}^\top\bm{B}\bm{Q}]$ as claimed.
\end{itemize}
\end{proof}

\subsection{\texorpdfstring{Proof of Proposition~\ref{prop:gauge-kernel} (Kernel of $J_\mathcal{G}$)}{Proof of Proposition (Kernel of J\_G from factorization redundancy)}}\label{subsec:gauge-proof}

\begin{proposition}[Kernel of $J_\mathcal{G}$ from factorization redundancy]\label{prop:gauge-kernel}
For $\bm{B}_t \in \mathbb{R}^{m \times r}$ of column rank $r$ and $\bm{A}_t \in \mathbb{R}^{r \times n}$ of row rank $r$, the kernel of the Jacobian operator $J_\mathcal{G} = J_\mathcal{G}([\bm{B}_t, \bm{A}_t])$ is
\[
    \ker(J_\mathcal{G}) = \big\{\, [\bm{B}_t \bm{X},\, -\bm{X}\bm{A}_t]\,:\, \bm{X} \in \mathbb{R}^{r \times r}\,\big\},
\]
which is an $r^2$-dimensional linear subspace of $[\mathbb{R}^{m \times r}, \mathbb{R}^{r \times n}]$. Consequently, $\mathrm{rank}(J_\mathcal{G}) = (m+n)r - r^2$.
\end{proposition}

\begin{proof}
We prove the three statements in turn: (i) the indicated set is contained in $\ker(J_\mathcal{G})$; (ii) every kernel element has this form; (iii) the dimension equals $r^2$. The rank statement then follows by the rank--nullity theorem applied to $J_\mathcal{G} : [\mathbb{R}^{m \times r}, \mathbb{R}^{r \times n}] \to \mathbb{R}^{m \times n}$, whose domain has dimension $(m+n)r$.

\textbf{(i) Inclusion.} For any $\bm{X} \in \mathbb{R}^{r \times r}$, by the definition of $J_\mathcal{G}$ in~\eqref{eq:jacobian-def},
\[
    J_\mathcal{G}\big[\bm{B}_t \bm{X},\, -\bm{X}\bm{A}_t\big]
    \;=\; (\bm{B}_t \bm{X}) \bm{A}_t \;+\; \bm{B}_t (-\bm{X}\bm{A}_t)
    \;=\; \bm{B}_t \bm{X} \bm{A}_t \;-\; \bm{B}_t \bm{X} \bm{A}_t
    \;=\; \bm{0}.
\]
Hence $[\bm{B}_t \bm{X},\, -\bm{X}\bm{A}_t] \in \ker(J_\mathcal{G})$ for every $\bm{X}$.

\textbf{(ii) Reverse inclusion.} Suppose $[\bm{P},\bm{Q}] \in \ker(J_\mathcal{G})$, i.e., $\bm{P}\bm{A}_t + \bm{B}_t \bm{Q} = \bm{0}$, equivalently
\begin{equation}\label{eq:kernel-eq}
    \bm{P}\bm{A}_t \;=\; -\bm{B}_t \bm{Q}.
\end{equation}
Since $\bm{B}_t$ has column rank $r$ and $\bm{A}_t$ has row rank $r$, define their one-sided pseudoinverses
\[
    \bm{B}_t^{+} := (\bm{B}_t^\top \bm{B}_t)^{-1} \bm{B}_t^\top \in \mathbb{R}^{r \times m}, \qquad
    \bm{A}_t^{+} := \bm{A}_t^\top (\bm{A}_t \bm{A}_t^\top)^{-1} \in \mathbb{R}^{n \times r},
\]
satisfying $\bm{B}_t^{+} \bm{B}_t = \bm{I}_r$ and $\bm{A}_t \bm{A}_t^{+} = \bm{I}_r$. Set $\bm{X} := \bm{B}_t^{+} \bm{P} \in \mathbb{R}^{r \times r}$.

Left-multiplying~\eqref{eq:kernel-eq} by $\bm{B}_t^{+}$ gives $\bm{X} \bm{A}_t = -\bm{Q}$, i.e.,
\[
    \bm{Q} \;=\; -\bm{X}\bm{A}_t.
\]
Substituting back into~\eqref{eq:kernel-eq} and right-multiplying by $\bm{A}_t^{+}$ gives
\[
    \bm{P} \;=\; \bm{P}\bm{A}_t \bm{A}_t^{+} \;=\; -\bm{B}_t \bm{Q} \bm{A}_t^{+} \;=\; \bm{B}_t \bm{X} \bm{A}_t \bm{A}_t^{+} \;=\; \bm{B}_t \bm{X}.
\]
Therefore $[\bm{P},\bm{Q}] = [\bm{B}_t \bm{X},\, -\bm{X}\bm{A}_t]$, as claimed.

\textbf{(iii) Dimension.} The map $\Phi : \mathbb{R}^{r \times r} \to [\mathbb{R}^{m \times r}, \mathbb{R}^{r \times n}]$ defined by $\Phi(\bm{X}) = [\bm{B}_t \bm{X},\, -\bm{X}\bm{A}_t]$ is linear. It is injective: if $\Phi(\bm{X}) = \bm{0}$ then $\bm{B}_t \bm{X} = \bm{0}$, and the column-rank condition on $\bm{B}_t$ gives $\bm{X} = \bm{0}$. Hence $\dim(\mathrm{im}\,\Phi) = \dim(\mathbb{R}^{r \times r}) = r^2$. By (i)--(ii), $\mathrm{im}\,\Phi = \ker(J_\mathcal{G})$, so $\dim \ker(J_\mathcal{G}) = r^2$.

\textbf{Rank.} By the rank--nullity theorem,
\[
    \mathrm{rank}(J_\mathcal{G}) \;=\; \dim\big([\mathbb{R}^{m \times r}, \mathbb{R}^{r \times n}]\big) - \dim \ker(J_\mathcal{G}) \;=\; (m+n)r - r^2.
\]
Equivalently, $\mathrm{rank}(J_\mathcal{G}) = \dim \mathcal{M}_r$, the dimension of the rank-$r$ manifold at $\bm{W}_t = \bm{B}_t \bm{A}_t$, consistent with $\mathrm{im}(J_\mathcal{G}) = \mathbb{T}_t$ (Proposition~\ref{prop:delta_W_on_tangent}).
\end{proof}

\subsection{\texorpdfstring{Optimality of $\mathcal{H}_t$-Projection (Solution~\ref{sol:transport})}{Optimality of H-Projection (Solution 1)}}\label{sec:proof_transport_optimality}

\begin{proposition}[Common $\bm{W}$-update across the affine solution set]\label{prop:transport-optimality}
Let $\mathcal{H}_t : \mathbb{R}^{m\times n} \to \mathbb{R}^{m\times n}$ be a symmetric positive-definite operator, inducing the inner product $\langle \bm{X}, \bm{Y}\rangle_{\mathcal{H}_t} := \langle \mathcal{H}_t \bm{X}, \bm{Y}\rangle$ on $\mathbb{R}^{m\times n}$. Let $\mathbb{T}_t := \mathrm{range}(J_\mathcal{G}) \subset \mathbb{R}^{m\times n}$ and $\widetilde{\bm{G}}_t := \mathcal{H}_t^{-1} \bm{G}_t$. Then every $[\bm{\Delta}_{\bm{B}_t}, \bm{\Delta}_{\bm{A}_t}]$ in the affine solution set of~\eqref{eq:factor-system-instantiated} satisfies
\[
    J_\mathcal{G}[\bm{\Delta}_{\bm{B}_t}, \bm{\Delta}_{\bm{A}_t}] = \widetilde{\mathcal{P}}_{\mathbb{T}_t}(\widetilde{\bm{G}}_t),
\]
where $\widetilde{\mathcal{P}}_{\mathbb{T}_t}$ is the $\mathcal{H}_t$-orthogonal projector onto $\mathbb{T}_t$, characterized by $\langle \widetilde{\mathcal{P}}_{\mathbb{T}_t}(\bm{X}) - \bm{X}, \bm{Y}\rangle_{\mathcal{H}_t} = 0$ for all $\bm{X} \in \mathbb{R}^{m\times n}, \bm{Y} \in \mathbb{T}_t$.
\end{proposition}

\begin{proof}
Equip $V := \mathbb{R}^{m\times n}$ with the inner product $\langle\cdot,\cdot\rangle_{\mathcal{H}_t}$. Since $\mathcal{H}_t$ is SPD, this is a genuine inner product, and $V$ is a finite-dimensional Hilbert space. The subspace $\mathbb{T}_t$ is finite-dimensional, hence closed in $V$, and admits a unique $\mathcal{H}_t$-orthogonal decomposition $V = \mathbb{T}_t \oplus \mathbb{T}_t^{\perp_{\mathcal{H}_t}}$, defining $\widetilde{\mathcal{P}}_{\mathbb{T}_t}$.

Let $[\bm{\Delta}_{\bm{B}_t}, \bm{\Delta}_{\bm{A}_t}]$ be any solution of~\eqref{eq:factor-system-instantiated}. Using $J_\mathcal{G}^*(\bm{G}_t) = J_\mathcal{G}^* \mathcal{H}_t \widetilde{\bm{G}}_t$, the equation rewrites as $J_\mathcal{G}^* \mathcal{H}_t (J_\mathcal{G}[\bm{\Delta}_{\bm{B}_t}, \bm{\Delta}_{\bm{A}_t}] - \widetilde{\bm{G}}_t) = 0$, i.e.\ $J_\mathcal{G}[\bm{\Delta}_{\bm{B}_t}, \bm{\Delta}_{\bm{A}_t}] - \widetilde{\bm{G}}_t \in \mathbb{T}_t^{\perp_{\mathcal{H}_t}}$ (since $\ker(J_\mathcal{G}^* \mathcal{H}_t) = \mathbb{T}_t^{\perp_{\mathcal{H}_t}}$ by adjointness). Decomposing $\widetilde{\bm{G}}_t = \widetilde{\mathcal{P}}_{\mathbb{T}_t}(\widetilde{\bm{G}}_t) + \widetilde{\mathcal{P}}_{\mathbb{T}_t}^\perp(\widetilde{\bm{G}}_t)$, this gives $J_\mathcal{G}[\bm{\Delta}_{\bm{B}_t}, \bm{\Delta}_{\bm{A}_t}] - \widetilde{\mathcal{P}}_{\mathbb{T}_t}(\widetilde{\bm{G}}_t) \in \mathbb{T}_t^{\perp_{\mathcal{H}_t}}$. But the left-hand side also lies in $\mathbb{T}_t$ (both summands do), and $\mathbb{T}_t \cap \mathbb{T}_t^{\perp_{\mathcal{H}_t}} = \{0\}$, so $J_\mathcal{G}[\bm{\Delta}_{\bm{B}_t}, \bm{\Delta}_{\bm{A}_t}] = \widetilde{\mathcal{P}}_{\mathbb{T}_t}(\widetilde{\bm{G}}_t)$.
\end{proof}

\subsection{Orthogonal Projection to Tangent Space}

In this subsection, we derive the orthogonal projection onto the tangent space under both the standard metric and the weighted metric. The specific forms of $\bm{L}_t$ and $\bm{R}_t$ are presented here and will not be repeated in subsequent propositions and proofs. For the sake of simplicity, the subscript $t$ will be omitted in this subsection.
\begin{equation}
    \begin{split}
    \bm{L}_t &= \diag(\bm{l}_t / \sqrt{\| \bm{l}_t\|_1} ) ~~\textmd{with}~~ \bm{l}_t  = \beta_1\bm{l}_{t-1} + (1-\beta_1)\sum_{j=1}^{n} (\bm{G}_t \odot \bm{G}_t)_{i,j},\\
    \bm{R}_{t} &= \diag(\bm{r}_{t}/ \sqrt{\| \bm{r}_t \|_1} ) ~~\textmd{with}~~ \bm{r}_t  = \beta_2 \bm{r}_{t-1} + (1-\beta_2)\sum_{i=1}^m (\bm{G}_t \odot \bm{G}_t)_{i,j},
\end{split}
\end{equation}
where $\odot$ denotes the Hadamard (elementwise) product and $\bm{G}_t = \nabla \mathcal{L}(\bm{W}_0 + \bm{W}_t)$.

\begin{proposition}[Orthogonal Projection to Tangent Space Under the Standard Metric]\label{prop:proj_tangent_stand}
Let $\bm{W} \in \mathcal{M}_r$ be a rank-$r$ matrix with a low-rank decomposition $\bm{W} = \bm{B} \bm{A}$, where $\bm{B} \in \mathbb{R}^{m \times r}, \bm{A} \in \mathbb{R}^{r \times n}$. Denote by $\mathbb{T}_{\bm{W}}$ the tangent space of the smooth manifold $\mathcal{M}_r$ at the point $\bm{W}$. Then, the orthogonal projection of any matrix $\bm{Z}\in \mathbb{R}^{m \times n}$ onto $\mathbb{T}_{\bm{W}}$ is given by
\[
\mathcal{P}_{\mathbb{T}_{\bm{W}}}(\bm{Z})= \bm{B} (\bm{B}^\top \bm{B})^{-1} \bm{B}^\top \bm{Z} + \bm{Z}\bm{A}^{\top} ( \bm{A}  \bm{A}^\top )^{-1} \bm{A} - \bm{B} (\bm{B}^\top \bm{B})^{-1} \bm{B}^\top\bm{Z}\bm{A}^{\top} ( \bm{A}  \bm{A}^\top )^{-1} \bm{A}.
\]
\end{proposition}

\begin{proof}
    Suppose $\bm{W}$ has a compact singular value decomposition, given by $\bm{W} = \bm{U} \bm{\Sigma} \bm{V}^\top$, where $\bm{U} \in \mathbb{R}^{m \times r}, \bm{\Sigma} \in \mathbb{R}^{r \times r}, \bm{V} \in \mathbb{R}^{n \times r}$. Then the tangent space $\mathbb{T}_{\bm{W}}$ at $\bm{W}$ is characterized as
    \[
        \mathbb{T}_{\bm{W}} = \{ \bm{U} \bm{M}^\top + \bm{N} \bm{V}^\top, \text{for } \bm{M} \in \mathbb{R}^{n \times r}, \bm{N} \in \mathbb{R}^{m \times r} \}.
    \]
    Therefore, the orthogonal projection of $\bm{Z}$ onto $\mathbb{T}_{\bm{W}}$ is known to be~\citep{wei2016RGD_recovery} 
    \begin{equation}\label{eq:standard_orth_proj}
        \mathcal{P}_{\mathbb{T}_{\bm{W}}}(\bm{Z})= \bm{U} \bm{U}^\top \bm{Z} + \bm{Z}\bm{V} \bm{V}^{\top} - \bm{U} \bm{U}^\top\bm{Z}\bm{V} \bm{V}^{\top}.
    \end{equation}
    Since the columns of $\bm{B}$ and $\bm{U}$  span the same column space (i.e., the column space of $\bm{W}$), then there exists an invertible matrix $\bm{S} \in \mathbb{R}^{r \times r}$ such that $\bm{B} = \bm{US}$ and $\bm{U} = \bm{B} \bm{S}^{-1}$. Using this relation, we have
    \[
        \bm{U}^\top \bm{U} = (\bm{B} \bm{S}^{-1})^\top \bm{B} \bm{S}^{-1} = \bm{S}^{-\top} (\bm{B}^\top \bm{B}) \bm{S}^{-1}.
    \]
   Since $\bm{U}^\top \bm{U} = \bm{I}_{r}$, it follows that
    \[
    \bm{S}^{-\top} (\bm{B}^\top \bm{B}) \bm{S}^{-1} = \bm{I}_r \implies \bm{B}^\top \bm{B} = \bm{S}^\top \bm{S}.
    \]
    Using this, we compute $\bm{U} \bm{U}^\top$
    \begin{equation}\label{eq:standard_UU_T}
        \bm{U} \bm{U}^\top = \bm{B} \bm{S}^{-1} \bm{S}^{-\top} \bm{B}^{\top} 
        = \bm{B}  (  \bm{S}^\top  \bm{S}  )^{-1} \bm{B}^\top  = \bm{B} (\bm{B}^\top \bm{B})^{-1} \bm{B}^\top
    \end{equation}        
    Similarly, since the rows of $\bm{A}$ and the columns of $\bm{V}$ span the same row space (i.e., the row space of $\bm{W}$), there exists an invertible matrix $\bm{Q} \in \mathbb{R}^{r \times r}$ such that $\bm{A} = \bm{Q} \bm{V}^\top$ and $\bm{V}^\top = \bm{Q}^{-1} \bm{A}$. Further, using $\bm{V}^\top \bm{V} = \bm{I}_r$, we obtain
        \[
        \bm{V}^\top \bm{V} = \bm{Q}^{-1} (\bm{A} \bm{A}^\top) \bm{Q}^{-\top} = \bm{I}_r,
    \]
    hence $\bm{A} \bm{A}^\top = \bm{Q} \bm{Q}^{\top}$ and 
    \begin{equation}\label{eq:standard_VV_T}
        \bm{V} \bm{V}^\top =  \bm{A}^{\top}  \bm{Q}^{-\top} \bm{Q}^{-1} \bm{A} = \bm{A}^{\top} ( \bm{Q} \bm{Q}^\top )^{-1} \bm{A}= \bm{A}^{\top} ( \bm{A}  \bm{A}^\top )^{-1} \bm{A}
    \end{equation}
Substituting (\ref{eq:standard_UU_T}) and (\ref{eq:standard_VV_T}) into (\ref{eq:standard_orth_proj}) yields
    \[
        \mathcal{P}_{\mathbb{T}_{\bm{W}}}(\bm{Z}) = \bm{B} (\bm{B}^\top \bm{B})^{-1} \bm{B}^\top \bm{Z} + \bm{Z}\bm{A}^{\top} ( \bm{A}  \bm{A}^\top )^{-1} \bm{A} - \bm{B} (\bm{B}^\top \bm{B})^{-1} \bm{B}^\top\bm{Z}\bm{A}^{\top} ( \bm{A}  \bm{A}^\top )^{-1} \bm{A}.
    \]
\end{proof}

\begin{proposition}[Orthogonal Projection onto the Tangent Space Under the Weighted Metric]\label{prop:proj_tangent_weight}
Let $\bm{W} \in \mathcal{M}_r$ have a low-rank decomposition  $\bm{W} = \bm{B} \bm{A}$, where $\bm{B} \in \mathbb{R}^{m \times r}, \bm{A} \in \mathbb{R}^{r \times n}$. Denote the tangent space of the Riemannian manifold $\mathcal{M}_r$ at the point $\bm{W}$ as $\mathbb{T}_{W}$. The weighted metric is  defined as $\langle \bm{Y}, \bm{Z} \rangle_{\bm{H}} = \langle \bm{L}^{\frac{1}{2}} \bm{Y}\bm{R}^{\frac{1}{2}}, \bm{Z}\rangle$ for any $\bm{Y}, \bm{Z} \in \mathbb{R}^{m \times n}$. Then, the orthogonal projection of any matrix $\bm{Z}\in \mathbb{R}^{m \times n}$ onto $\mathbb{T}_{W}$ under the weighted metric is given by
\[ 
\begin{split}
\mathcal{P}_{\mathbb{T}_{\bm{W}}}(\bm{Z})&= \bm{B} (\bm{B}^\top \bm{L}^{\frac{1}{2}}\bm{B})^{-1} \bm{B}^\top \bm{L}^{\frac{1}{2}}\bm{Z} + \bm{Z}\bm{R}^{\frac{1}{2}}\bm{A}^{\top} ( \bm{A} \bm{R}^{\frac{1}{2}} \bm{A}^\top )^{-1} \bm{A}\\
&- \bm{B} (\bm{B}^\top \bm{L}^{\frac{1}{2}}\bm{B})^{-1} \bm{B}^\top \bm{L}^{\frac{1}{2}} \bm{Z}\bm{R}^{\frac{1}{2}} \bm{A}^{\top} ( \bm{A}  \bm{R}^{\frac{1}{2}} \bm{A}^\top )^{-1} \bm{A}.
\end{split}
\]
\end{proposition}

\begin{proof} 
This proof is inspired by~\citep{bian2024preconditioned}. Here, we briefly provide a sketch of the proof.
\begin{itemize}
\item [(i)] {\it The new orthonormal basis under the weighted metric.} Let $\bm{W} = \bm{U} \bm{\Sigma} \bm{V}^\top$ be a compact SVD with $\bm{U} =[\bm{u}_1, \bm{u}_2, \cdots, \bm{u}_r]\in \mathbb{R}^{m \times r}, \bm{V}=[\bm{v}_1, \bm{v}_2, \cdots, \bm{v}_r] \in \mathbb{R}^{n \times r}$. Normalize the singular vectors under the weighted vector 
    \[
        \langle \bm{x},\bm{y} \rangle_{\bm{L}^{\frac{1}{2}}} = \langle \bm{L}^{\frac{1}{2}}\bm{x},\bm{y} \rangle \ \ \text{in} \ \ \mathbb{R}^{m} \quad \text{and} \quad \langle \bm{x},\bm{y} \rangle_{\bm{R}^{\frac{1}{2}}} = \langle \bm{R}^{\frac{1}{2}}\bm{x},  \bm{y} \rangle \ \ \text{in} \ \ \mathbb{R}^{n}
    \]
to obtain 
    \[
    \begin{split}
        \widetilde{\bm{U}} & =\bm{U}(\bm{U}^\top \bm{L}^{\frac{1}{2}} \bm{U})^{-\frac{1}{2}} \coloneq [\tilde{\bm{u}}_1, \tilde{\bm{u}}_2, \cdots, \tilde{\bm{u}}_r] \in \mathbb{R}^{m \times r}, \\
        \widetilde{\bm{V}} & =\bm{V}(\bm{V}^\top \bm{R}^{\frac{1}{2}} \bm{V})^{-\frac{1}{2}} \coloneq [\tilde{\bm{v}}_1, \tilde{\bm{v}}_2, \cdots, \tilde{\bm{v}}_r] \in \mathbb{R}^{n \times r}.
    \end{split}
    \]
Next, we extend $\widetilde{\bm{U}}$ and $\widetilde{\bm{V}}$ to full orthonormal bases of $(\mathbb{R}^m, \langle \cdot,\cdot \rangle_{\bm{L}^{\frac{1}{2}}})$ and $(\mathbb{R}^n, \langle \cdot,\cdot \rangle_{\bm{R}^{\frac{1}{2}}})$ respectively. Then, an orthonormal basis of $\mathbb{T}_{\bm{W}}$ with respect to $\langle \cdot, \cdot \rangle_{\bm{H}_t}$ is $\{ \tilde{\bm{u}}_i \tilde{\bm{v}}_j^\top\}_{\min\{i,j\} \leq r}$.

\item[(ii)]{\it Orthogonal projection represented by the new orthonormal basis.} Using the orthonormal bases $\widetilde{\bm{U}}$ and $\widetilde{\bm{V}}$, the projection of $\bm{Z}$ onto $\mathbb{T}_{\bm{W}}$ is expressed as:
    \[
    \begin{split}
        \widetilde{\mathcal{P}}_{\mathbb{T}_{\bm{W}}}(\bm{Z}) & = \sum_{(i,j):\min\{i,j\} \leq r} \langle  \bm{Z}, \tilde{\bm{u}}_i \tilde{\bm{v}}_j^\top \rangle_{\bm{H}_t} \cdot \tilde{\bm{u}}_i \tilde{\bm{v}}_j^\top = \sum_{(i,j):\min\{i,j\} \leq r} \langle \bm{L}^{\frac{1}{2}} \bm{Z}  \bm{R}^{\frac{1}{2}}, \tilde{\bm{u}}_i \tilde{\bm{v}}_j^\top \rangle \cdot \tilde{\bm{u}}_i \tilde{\bm{v}}_j^\top \\
        & = \sum_{(i,j):\min\{i,j\} \leq r} \tilde{\bm{u}}_i^\top \bm{L}^{\frac{1}{2}} \bm{Z} \bm{R}^{\frac{1}{2}} \tilde{\bm{v}}_j \cdot \tilde{\bm{u}}_i \tilde{\bm{v}}_j^\top \\
        & = \widetilde{\bm{U}}\widetilde{\bm{U}}^\top \bm{L}^{\frac{1}{2}} \bm{Z} + \bm{Z} \bm{R}^{\frac{1}{2}} \widetilde{\bm{V}}\widetilde{\bm{V}}^\top - \widetilde{\bm{U}}\widetilde{\bm{U}}^\top \bm{L}^{\frac{1}{2}} \bm{Z} \bm{R}^{\frac{1}{2}} \widetilde{\bm{V}}\widetilde{\bm{V}}^\top.
    \end{split}
    \]
    \item[(iii)] {\it Express the basis projectors via factors $\bm{B}$ and $\bm{A}$.} Since $\bm{B}$ and $\bm{A}$ span the same spaces as $\bm{U}$ and $\bm{V}$, we derive
\[
 \widetilde{\bm{U}} \widetilde{\bm{U}}^\top = \bm{B} \left(  \bm{B}^\top \bm{L}^{\frac{1}{2}} \bm{B}  \right)^{-1} \bm{B}^\top, ~~~ \widetilde{\bm{V}} \widetilde{\bm{V}}^\top = \bm{A}^{\top} \left( \bm{A} \bm{R}^{\frac{1}{2}} \bm{A}^\top  \right)^{-1} \bm{A}.
\]
Substituting these expressions into the formula for   $\widetilde{\mathcal{P}}_{\mathbb{T}_{\bm{W}}}$, we obtain
    \[
    \begin{split}
        \widetilde{\mathcal{P}}_{\mathbb{T}_{\bm{W}}}(\bm{Z}) & = \bm{B} ( \bm{B}^\top \bm{L}^{\frac{1}{2}}\bm{B} )^{-1} \bm{B}^\top \bm{L}^{\frac{1}{2}} \bm{Z} + \bm{Z} \bm{R}^{\frac{1}{2}} \bm{A}^\top (\bm{A} \bm{R}^{\frac{1}{2}} \bm{A}^\top)^{-1} \bm{A} \\
        & \quad-\bm{B} ( \bm{B}^\top \bm{L}^{\frac{1}{2}}\bm{B} )^{-1} \bm{B}^\top \bm{L}^{\frac{1}{2}} \bm{Z} \bm{R}^{\frac{1}{2}} \bm{A}^\top (\bm{A} \bm{R}^{\frac{1}{2}} \bm{A}^\top)^{-1} \bm{A}.
    \end{split}
    \]
    \end{itemize}

\end{proof}

\begin{proposition}\label{prop:delta_W_on_tangent}
    Suppose $\bm{W} \in \mathcal{M}_r$ has a low-rank decomposition $\bm{W} = \bm{B} \bm{A}$, where $\bm{B} \in \mathbb{R}^{m \times r}$ and $\bm{A} \in \mathbb{R}^{r \times n}$. For any matrix $\bm{M} \in \mathbb{R}^{m \times r}, \bm{N} \in \mathbb{R}^{r \times n}$, the matrix $\bm{M} \bm{A} + \bm{B} \bm{N}$ lies in the tangent space $\mathbb{T}_{\bm{W}}$ of $\mathcal{M}_r$ at the point $\bm{W}$.
    \begin{proof}
    Let $\bm{W} \in \mathcal{M}_r$ have a compact singular value decomposition $\bm{W} = \bm{U} \bm{\Sigma} \bm{V}^\top$, where $\bm{U} \in \mathbb{R}^{m \times r}, \bm{\Sigma} \in \mathbb{R}^{r \times r}$, and $ \bm{V} \in \mathbb{R}^{n \times r}$.  By definition, the tangent space $\mathbb{T}_{\bm{W}}$ at $\bm{W}$ is given by
    \[
    \mathbb{T}_{\bm{W}} = \{\bm{U}\bm{K}_1^{\top} + \bm{K}_2 \bm{V}^{\top} |  \bm{K}_1 \in \mathbb{R}^{n \times r}, \bm{K}_2 \in \mathbb{R}^{m \times r}\}.
    \]
    Since $\bm{B}$ and  $\bm{A}$ are low-rank factors of $\bm{W}$, there exist invertible matrices $\bm{S} \in \mathbb{R}^{r \times r}$ and $\bm{Q} \in \mathbb{R}^{r \times r}$ such that
    \[
    \bm{B} = \bm{U}\bm{S}, ~~ \bm{A} = \bm{Q} \bm{V}^\top.
    \]
Substituting these expressions, the matrix $\bm{M} \bm{A} + \bm{B} \bm{N}$   can be rewritten as
        \[
       \bm{M} \bm{A} + \bm{B} \bm{N} =     \bm{M} \bm{Q} \bm{V}^\top + \bm{U}\bm{S} \bm{N}.
        \]
The first term, $\bm{M} \bm{Q} \bm{V}^\top $, lies in $\textmd{span}(\bm{V}^\top )$, and the second term, $\bm{U} \bm{S} \bm{N}$, lies in $\textmd{span}(\bm{U} )$. Thus, the sum $\bm{M} \bm{Q} \bm{V}^\top + \bm{U}\bm{S} \bm{N}$ lies in the tangent space  $\mathbb{T}_{\bm{W}}$ by the definition of the tangent space. Then, it follows that  $\bm{M} \bm{A} + \bm{B} \bm{N}$ is on the tangent space $\mathbb{T}_{\bm{W}}$. This completes the proof.
    \end{proof}
\end{proposition}

\subsection{Proof of Theorem \ref{thm:find_delta_B_A} (closed-form factor update)}\label{sec:proof_find_delta_B_A}

The proof has two stages: (i) solve~\eqref{prob:fit-proj} for $(\bm{\Delta}_{\bm{B}_t}, \bm{\Delta}_{\bm{A}_t})$ in terms of any matrix $\bm{X}_t \in \mathbb{R}^{r \times r}$ (Lemma~\ref{lem:find_delta_B_A_with_X}); (ii) determine $\bm{X}_t$ by minimizing the $\mathcal{H}_t$-imbalance from Solution~\ref{sol:balance} (Lemma~\ref{lem:find_X}). Substituting (ii) into (i) yields the closed form in Theorem~\ref{thm:find_delta_B_A}.

\begin{lemma}[$\bm{X}_t$-parameterized factor solution]\label{lem:find_delta_B_A_with_X}
For problem~\eqref{prob:fit-proj}, the optimal $(\bm{\Delta}_{\bm{B}_t}, \bm{\Delta}_{\bm{A}_t})$ are
\[
    \bm{\Delta}_{\bm{B}_t}^{\rm opt} = (\bm{I} - \widetilde{\bm{P}}_{B_t}) \bm{L}_t^{-1/2} \bm{G}_{\bm{B}_t} (\bm{A}_t \bm{R}_t^{1/2} \bm{A}_t^\top)^{-1} - \bm{B}_t \bm{X}_t, \quad
    \bm{\Delta}_{\bm{A}_t}^{\rm opt} = (\bm{B}_t^\top \bm{L}_t^{1/2} \bm{B}_t)^{-1} \bm{G}_{\bm{A}_t} \bm{R}_t^{-1/2} + \bm{X}_t \bm{A}_t,
\]
where $\bm{X}_t \in \mathbb{R}^{r \times r}$ is arbitrary.
\end{lemma}

\textbf{Proof of Lemma~\ref{lem:find_delta_B_A_with_X}.} Define
\[
    \Gamma(\bm{\Delta}_{\bm{B}_t}, \bm{\Delta}_{\bm{A}_t}) :=  \frac{1}{2} \|  \bm{\Delta}_{\bm{B}_t}\bm{A}_t + \bm{B}_t\bm{\Delta}_{\bm{A}_t} -\widetilde{\mathcal{P}}_{\mathbb{T}_t}(\bm{L}_t^{-\frac{1}{2}}  \bm{G}_t\bm{R}_t^{-\frac{1}{2}}) \|_{\mathcal{H}_t}^2 .
\]
Differentiating $\Gamma(\bm{\Delta}_{\bm{B}_t}, \bm{\Delta}_{\bm{A}_t})$ with respect to $\bm{\Delta}_{\bm{B}_t}$ and $\bm{\Delta}_{\bm{A}_t}$ yields
\begin{equation}
\begin{split}
    \nabla_{\bm{\Delta}_{\bm{B}_t}} \Gamma (\bm{\Delta}_{\bm{B}_t}, \bm{\Delta}_{\bm{A}_t})  
    & = \bm{L}_t^{\frac{1}{2}} \bm{\Delta}_{\bm{B}_t}(\bm{A}_t \bm{R}_t^{\frac{1}{2}} \bm{A}_t^\top) + \bm{L}_t^{\frac{1}{2}} \bm{B}_t \bm{\Delta}_{\bm{A}_t} \bm{R}_t^{\frac{1}{2}} \bm{A}_t^\top - \bm{G}_t \bm{A}_t^\top,
\end{split}
\end{equation}
and 
\begin{equation}\label{eq:deriva_delta_A}
    \begin{split}
    \nabla_{\bm{\Delta}_{\bm{A}_t}} \Gamma (\bm{\Delta}_{\bm{B}_t}, \bm{\Delta}_{\bm{A}_t}) 
    & = \bm{B}_t^\top \bm{L}_t^{\frac{1}{2}} \bm{\Delta}_{\bm{B}_t}\bm{A}_t \bm{R}_t^{\frac{1}{2}}  + \bm{B}_t^\top \bm{L}_t^{\frac{1}{2}} \bm{B}_t \bm{\Delta}_{\bm{A}_t} \bm{R}_t^{\frac{1}{2}}  - \bm{B}_t^\top \bm{G}_t .
\end{split}
\end{equation}

Setting  $\nabla_{\bm{\Delta}_{\bm{B}_t}} \Gamma(\bm{\Delta}_{\bm{B}_t}, \bm{\Delta}_{\bm{A}_t})= \bm{0}$ and using the invertibility of  $(\bm{A}_t \bm{R}_t^{\frac{1}{2}} \bm{A}_t^\top)$ and $\bm{L}_t^{\frac{1}{2}}$ gives
\begin{equation}\label{eq:delta_B}
   \bm{\Delta}_{\bm{B}_t} = \bm{L}_t^{-\frac{1}{2}} \bm{G}_t \bm{A}_t^\top (\bm{A}_t \bm{R}_t^{\frac{1}{2}} \bm{A}_t^\top)^{-1} - \bm{B}_t \bm{\Delta}_{\bm{A}_t} \bm{R}_t^{\frac{1}{2}} \bm{A}_t^\top (\bm{A}_t \bm{R}_t^{\frac{1}{2}} \bm{A}_t^\top)^{-1}.
\end{equation}
Substituting \eqref{eq:delta_B} into $ \nabla_{\bm{\Delta}_{\bm{A}_t}} \Gamma (\bm{\Delta}_{\bm{B}_t}, \bm{\Delta}_{\bm{A}_t}) = \bm{0}$ and using the invertibility of $\bm{B}_t^\top \bm{L}_t^{\frac{1}{2}} \bm{B}_t$ and $\bm{R}_t^{\frac{1}{2}}$ yields
\[
   \bm{\Delta}_{\bm{A}_t}[\bm{I} - \widetilde{\bm{Q}}_{A_t}] =  (\bm{B}_t^\top \bm{L}_t^{\frac{1}{2}} \bm{B}_t)^{-1} \bm{B}_t^\top \bm{G}_t \bm{R}_t^{-\frac{1}{2}} [\bm{I} - \widetilde{\bm{Q}}_{A_t}],
\]
where $\widetilde{\bm{Q}}_{A_t}=\bm{R}_t^{\frac{1}{2}} \bm{A}_t^\top (\bm{A}_t \bm{R}_t^{\frac{1}{2}} \bm{A}_t^\top)^{-1} \bm{A}_t$, which is the projection matrix onto the row space of $\bm{A}_t$. Since $\bm{I} -\widetilde{\bm{Q}}_{A_t}$ is the residual maker matrix, then a general solution is 
\[
    \bm{\Delta}_{\bm{A}_t}^{\textmd{opt}} =  (\bm{B}_t^\top \bm{L}_t^{\frac{1}{2}} \bm{B}_t)^{-1} \bm{B}_t^\top \bm{G}_t \bm{R}_t^{-\frac{1}{2}} + \bm{X}_t \bm{A}_t,
\]
with arbitrary  matrix $\bm{X}_t \in \mathbb{R}^{r\times r}$. Plugging this $\bm{\Delta}_{\bm{A}_t}$ back into \eqref{eq:delta_B} gives
\[
    \bm{\Delta}_{\bm{B}_t}^{\textmd{opt}} = [\bm{I} - \widetilde{\bm{P}}_{B_t}] \bm{L}_t^{-\frac{1}{2}} \bm{G}_t \bm{A}_t^\top (\bm{A}_t \bm{R}_t^{\frac{1}{2}} \bm{A}_t^\top )^{-1} - \bm{B}_t \bm{X}_t ,
\]
where $\widetilde{\bm{P}}_{B_t} = \bm{B}_t (\bm{B}_t^\top \bm{L}_t^{\frac{1}{2}} \bm{B}_t )^{-1} \bm{B}_t^\top \bm{L}_t^{\frac{1}{2}}$, which is the projection matrix onto the column space of $\bm{B}_t$.
\hfill $\square$

\begin{lemma}[$\mathcal{H}_t$-balance closed form]\label{lem:find_X}
The minimizer of $\frac{1}{2}\| \bm{\Delta}_{\bm{B}_t}\bm{A}_t - \bm{B}_t\bm{\Delta}_{\bm{A}_t} \|_{\mathcal{H}_t}^2$ over $\bm{X}_t \in \mathbb{R}^{r\times r}$ (with $\bm{\Delta}_{\bm{B}_t}, \bm{\Delta}_{\bm{A}_t}$ given by Lemma~\ref{lem:find_delta_B_A_with_X}) is
\[
    \bm{X}_t^{\rm opt} = -\frac{1}{2} (\bm{B}_t^\top \bm{L}_t^{1/2} \bm{B}_t)^{-1} \bm{B}_t^\top \bm{G}_t \bm{A}_t^\top (\bm{A}_t \bm{R}_t^{1/2} \bm{A}_t^\top)^{-1}.
\]
Substituting $\bm{X}_t^{\rm opt}$ into Lemma~\ref{lem:find_delta_B_A_with_X} yields the closed form in Theorem~\ref{thm:find_delta_B_A}.
\end{lemma}

\textbf{Proof of Lemma~\ref{lem:find_X}.} 
Let the objective function be $\Psi(\bm{X}_t) =  \frac{1}{2}\| \bm{\Delta}_{\bm{B}_t} \bm{A}_t - \bm{B}_t \bm{\Delta}_{\bm{A}_t} \|_{\mathcal{H}_t}^2 $. To minimize $\Psi(\bm{X}_t)$, we compute its gradient with respect to $\bm{X}_t$,
         \[
            \nabla_{\bm{X}_t} \Psi (\bm{X}_t)  = -2\bm{B}_t^\top \bm{L}_t^{\frac{1}{2}} (\bm{\Delta}_{\bm{B}_t} \bm{A}_t - \bm{B}_t \bm{\Delta}_{\bm{A}_t}) \bm{R}_t^{\frac{1}{2}} \bm{A}^\top.
         \]
Substituting the expressions for $\bm{\Delta}_{\bm{B}_t}$ and $\bm{\Delta}_{\bm{A}_t}$ from Lemma \ref{lem:find_delta_B_A_with_X}, we have
         \[
         \begin{split}
              \nabla_{\bm{X}_t} \Psi (\bm{X}_t)  &=  -2\bm{B}_t^\top \bm{L}_t^{\frac{1}{2}} \bigg( [\bm{I} - \bm{B}_t (\bm{B}_t^\top \bm{L}_t^{\frac{1}{2}} \bm{B}_t )^{-1} \bm{B}_t^\top \bm{L}_t^{\frac{1}{2}} ]\bm{L}_t^{-\frac{1}{2}} \bm{G}_t \bm{A}_t^\top (\bm{A}_t \bm{R}_t^{\frac{1}{2}} \bm{A}_t^\top )^{-1} \bm{A}_t  \\
             & \quad - \bm{B}_t (\bm{B}_t^\top \bm{L}_t^{\frac{1}{2}} \bm{B}_t )^{-1} \bm{B}_t^\top \bm{G}_t \bm{R}^{-\frac{1}{2}} - 2\bm{B}_t \bm{X}_t \bm{A}_t \bigg) \bm{R}_t^{\frac{1}{2}} \bm{A}^\top \\
             & = 2\bm{B}^\top_t \bm{G}_t \bm{A}_t^\top +4 (\bm{B}_t^\top \bm{L}_t^{\frac{1}{2}} \bm{B}_t) \bm{X}_t (\bm{A}_t \bm{R}_t^{\frac{1}{2}} \bm{A}_t^\top).
         \end{split}
         \]
Setting $\nabla_{\bm{X}_t} \Psi (\bm{X}_t) = \bm{0}$, we obtain
\[
- \bm{B}^\top_t \bm{G}_t \bm{A}_t^\top  = 2 (\bm{B}_t^\top \bm{L}_t^{\frac{1}{2}} \bm{B}_t) \bm{X}_t (\bm{A}_t \bm{R}_t^{\frac{1}{2}} \bm{A}_t^\top).
\]
Since $\bm{B}_t^\top \bm{L}_t^{\frac{1}{2}} \bm{B}_t$ and $\bm{A}_t \bm{R}_t^{\frac{1}{2}} \bm{A}_t^\top$ are invertible, we solve for $\bm{X}_t$ as
         \[
            \bm{X}_t^{\textmd{opt}} = -\frac{1}{2} (\bm{B}_t^\top \bm{L}_t^{\frac{1}{2}}  \bm{B}_t )^{-1} \bm{B}_t^\top \bm{G}_t  \bm{A}_t^\top (\bm{A}_t \bm{R}_t^{\frac{1}{2}} \bm{A}_t^\top)^{-1}.
         \]
This proves Lemma~\ref{lem:find_X}. Substituting $\bm{X}_t^{\rm opt}$ into Lemma~\ref{lem:find_delta_B_A_with_X} and simplifying $-\bm{B}_t \bm{X}_t = \frac{1}{2} \widetilde{\bm{P}}_{B_t} \bm{L}_t^{-1/2} \bm{G}_t \bm{A}_t^\top (\bm{A}_t \bm{R}_t^{1/2} \bm{A}_t^\top)^{-1}$ (and the symmetric expression for $\bm{X}_t \bm{A}_t$) gives the closed form in Theorem~\ref{thm:find_delta_B_A}, where $\bm{G}_{\bm{B}_t} = \bm{G}_t \bm{A}_t^\top$ and $\bm{G}_{\bm{A}_t} = \bm{B}_t^\top \bm{G}_t$ are the chain-rule low-rank gradients. \hfill $\square$

\section{Algorithm}

In this section, we provide the complete algorithm of FILO-Adafactor.

\begin{algorithm*}[htbp]
\caption{FILO-Adafactor for fine-tuning.}\label{alg:RAdaGrad}
\begin{algorithmic}[1]
\State Initialize $\bm{B}_1 = \bm{0}_{m \times r}$, $\bm{A}_1 = \text{Kaiming uniform}_{r \times n}$, $\bm{l}_0 = \bm{0}_{m}, \bm{r}_0 = \bm{0}_{n}$, $\epsilon = 1e-6$. 
\For{$t=1, \cdots, T$}
    \State $\bm{l}_{t} = \beta_1 \bm{l}_{t-1} + (1-\beta_1) \sum_{j=1}^{n} (\bm{G}_t \odot \bm{G}_t)_{i,j}$, $\bm{L}_{t} = \diag(\bm{l}_{t} / \sqrt{\| \bm{l}_t\|_1} )$.
    \State $\bm{r}_{t} = \beta_2 \bm{r}_{t-1} + (1-\beta_2) \sum_{i=1}^m (\bm{G}_t \odot \bm{G}_t)_{i,j}$, $\bm{R}_{t} = \diag(\bm{r}_{t}/ \sqrt{\| \bm{r}_t \|_1} )$.
    \State $\bm{\Delta}_{\bm{B}_t} = \left[\bm{I} - \frac{1}{2} \bm{B}_t (\bm{B}_t^\top \bm{L}_t^{\frac{1}{2}} \bm{B}_t )^{-1} \bm{B}_t^\top \bm{L}_t^{\frac{1}{2}}  \right] \bm{L}_t^{-\frac{1}{2}} \bm{G}_{\bm{B}_t} (\bm{A}_t \bm{R}_t^{\frac{1}{2}} \bm{A}_t^\top )^{-1}  $.
    \State $\bm{\Delta}_{\bm{A}_t} = (\bm{B}_t^\top \bm{L}_t^{\frac{1}{2}} \bm{B}_t )^{-1} \bm{G}_{\bm{A}_t} \bm{R}_t^{-\frac{1}{2}} \left[ \bm{I}  - \frac{1}{2} \bm{R}_t^{\frac{1}{2}} \bm{A}_t^\top (\bm{A}_t \bm{R}_t^{\frac{1}{2}} \bm{A}_t^\top )^{-1} \bm{A}_t  \right]$.
    \State $\bm{B}_{t+1} = \bm{B}_t - \eta_t \bm{\Delta}_{\bm{B}_t}$, $\bm{A}_{t+1} = \bm{A}_t - \eta_t \bm{\Delta}_{\bm{A}_t}$.
\EndFor
\\
Note: factor gradients $\bm{G}_{\bm{B}_t} = \partial \mathcal{L} / \partial \bm{B}_t$ and $\bm{G}_{\bm{A}_t} = \partial \mathcal{L} / \partial \bm{A}_t$ are obtained from autograd. The weight-space gradient $\bm{G}_t$ can be obtained via a backward hook on $\bm{W}$. Add $\epsilon \bm{I}$ to matrix $\bm{B}_t^\top \bm{L}_t^{\frac{1}{2}} \bm{B}_t$ if it is not invertible.
\end{algorithmic}
\end{algorithm*}

\begin{algorithm*}[htbp]
\caption{FILO-Adafactor with Momentum (FILO-AdafactorM) for fine-tuning.}\label{alg:RAdaGradM}
\begin{algorithmic}[1]
\State Initialize moment $\bm{M}_{\bm{B}_0} = \bm{0}_{m \times r}$, $\bm{M}_{\bm{A}_0} = \bm{0}_{r \times n}$, $\bm{B}_1 = \bm{0}_{m \times r}$, $ \bm{A}_1 = \text{Kaiming uniform}_{r \times n}$; $\bm{l}_0 = \bm{0}_{m}, \bm{r}_0 = \bm{0}_{n}$, weight decay $\lambda$, coefficients $\beta_1 = \beta_2$, and $\beta_3$, $\epsilon = 1e-6$.
\For{$t=1, \cdots, T$}
    \State $\bm{l}_{t} = \beta_1 \bm{l}_{t-1} + (1-\beta_1) \sum_{j=1}^{n} (\bm{G}_t \odot \bm{G}_t)_{i,j}$, $\bm{L}_{t} = \diag(\bm{l}_{t} / \sqrt{\| \bm{l}_t\|_1} )$.
    \State $\bm{r}_{t} = \beta_2 \bm{r}_{t-1} + (1-\beta_2) \sum_{i=1}^m (\bm{G}_t \odot \bm{G}_t)_{i,j}$, $\bm{R}_{t} = \diag(\bm{r}_{t}/ \sqrt{\| \bm{r}_t \|_1} )$.
    \State $\bm{M}_{\bm{B}_t} = \beta_3 \bm{M}_{\bm{B}_{t-1}} + (1-\beta_3) \bm{G}_{\bm{B}_t}$.
    \State $\bm{M}_{\bm{A}_t} = \beta_3 \bm{M}_{\bm{A}_{t-1}} + (1-\beta_3) \bm{G}_{\bm{A}_t}$.
    \State $\bm{\Delta}_{\bm{B}_t} = \left[\bm{I} - \frac{1}{2} \bm{B}_t \left(\bm{B}_t^\top \bm{L}_t^{\frac{1}{2}} \bm{B}_t \right)^{-1} \bm{B}_t^\top  \bm{L}_t^{\frac{1}{2}}  \right] \bm{L}_t^{-\frac{1}{2}} \bm{M}_{\bm{B}_t} \left(\bm{A}_t \bm{R}_t^{\frac{1}{2}} \bm{A}_t^\top \right)^{-1}  $.
    \State $\bm{\Delta}_{\bm{A}_t} = \left(\bm{B}_t^\top \bm{L}_t^{\frac{1}{2}} \bm{B}_t \right)^{-1} \bm{M}_{\bm{A}_t} \bm{R}_t^{-\frac{1}{2}} \left[ \bm{I}  - \frac{1}{2} \bm{R}_t^{\frac{1}{2}} \bm{A}_t^\top \left(\bm{A}_t \bm{R}_t^{\frac{1}{2}} \bm{A}_t^\top \right)^{-1} \bm{A}_t  \right]$.
    \State $\bm{B}_{t+1} = (1-\lambda \eta_t)\bm{B}_t - \eta_t \frac{\sqrt{1-\beta_1^t}}{1-\beta_3^t} \bm{\Delta}_{\bm{B}_t}$, $\bm{A}_{t+1} = (1-\lambda \eta_t ) \bm{A}_t - \eta_t \frac{\sqrt{1-\beta_1^t}}{1-\beta_3^t} \bm{\Delta}_{\bm{A}_t}$.
\EndFor
\\
Note: factor gradients $\bm{G}_{\bm{B}_t} = \partial \mathcal{L} / \partial \bm{B}_t$ and $\bm{G}_{\bm{A}_t} = \partial \mathcal{L} / \partial \bm{A}_t$ are obtained from autograd. The weight-space gradient $\bm{G}_t$ can be obtained via a backward hook on $\bm{W}$. Add $\epsilon \bm{I}$ to matrix $\bm{B}_t^\top \bm{L}_t^{\frac{1}{2}} \bm{B}_t$ if it is not invertible.
\end{algorithmic}
\end{algorithm*}

\section{Computational and Memory Complexity Analysis of FILO-Adafactor}\label{sec:alg-anal}
The update rule of FILO-Adafactor is given by
\begin{equation*}
\begin{split}
    \bm{\Delta}_{\bm{A}_t} & = ( \bm{B}_t^\top \bm{L}_t^{\frac{1}{2}} \bm{B}_t )^{-1} \underbrace{\bm{B}_t^\top \bm{G}_t}_{\bm{G}_{\bm{A}_t}}\bm{R}_t^{-\frac{1}{2}}  \left[ \bm{I} - \frac{1}{2} \bm{R}_t^{\frac{1}{2}} \bm{A}_t^{\top} ( \bm{A}_t \bm{R}_t^{\frac{1}{2}} \bm{A}_t^\top  )^{-1} \bm{A}_t \right], \\
    \bm{\Delta}_{\bm{B}_t} & = \left[ \bm{I} - \frac{1}{2}\bm{B}_t ( \bm{B}_t^\top \bm{L}_t^{\frac{1}{2}} \bm{B}_t  )^{-1} \bm{B}_t^\top \bm{L}_t^{\frac{1}{2}} \right] \bm{L}_t^{-\frac{1}{2}} \underbrace{\bm{G}_t \bm{A}_t^{\top}}_{\bm{G}_{\bm{B}_t}} ( \bm{A}_t \bm{R}_t^{\frac{1}{2}} \bm{A}_t^\top  )^{-1} .
\end{split}
\end{equation*}
We now analyze the computational complexity of computing the updates $\bm{\Delta}_{\bm{A}_t}$ and $\bm{\Delta}_{\bm{B}_t}$. For simplicity, we focus on $\bm{\Delta}_{\bm{A}_t}$, as the complexity for $\bm{\Delta}_{\bm{B}_t}$ is symmetric.

\begin{itemize}
    \item Compute gradient $\bm{G}_t$. The stochastic gradient $\bm{G}_t$ of $\bm{W}_t$ is obtained during the backpropagation process.
   \item Row and column sums for $\bm{l}_t$ and $\bm{r}_t$. Compute $\bm{l}_t$ and $\bm{r}_t$ by summing the square of the element of $\bm{G}_t$ along rows or columns, which is in the computation $\mathcal{O}(mn)$. $\bm{L}_t^{\frac{1}{2}}$ and $\bm{L}_t^{-\frac{1}{2}}$ can be computed in $\mathcal{O}(m)$, $\bm{R}_t^{\frac{1}{2}}$ and $\bm{R}_t^{-\frac{1}{2}}$ can be computed in $\mathcal{O}(n)$.
\item Compute $(\bm{B}_t^\top \bm{L}_t^{\frac{1}{2}} \bm{B}_t)^{-1} \bm{G}_{\bm{A}_t} \bm{R}_t^{-\frac{1}{2}}$.  First to compute the inverse matrices $(\bm{B}_t^\top \bm{L}^{\frac{1}{2}}_t \bm{B}_t)^{-1}$ in $\mathcal{O}((m+r)r^2) $. Then multiply the inverse $(\bm{B}_t^\top \bm{L}^{\frac{1}{2}}_t \bm{B}_t)^{-1}$ by $\bm{G}_{\bm{A}_t}$ in $\mathcal{O}(nr^2)$, and multiply the diagonal matrix $\bm{R}_t^{-\frac{1}{2}}$ in $\mathcal{O}(nr)$.
\item Compute $(\bm{B}_t^\top \bm{L}_t^{\frac{1}{2}} \bm{B}_t )^{-1} \bm{G}_{\bm{A}_t}\bm{A}_t^\top (\bm{A}_t \bm{R}_t^{\frac{1}{2}} \bm{A}_t^\top )^{-1} \bm{A}_t$. First to compute the inverse matrices $(\bm{A}_t \bm{R}^{\frac{1}{2}}_t \bm{A}_t^\top)^{-1}$ in $\mathcal{O}((n+r)r^2) $. Use the result from the last step, multiply $(\bm{B}_t^\top \bm{L}_t^{\frac{1}{2}} \bm{B}_t )^{-1} \bm{G}_{\bm{A}_t}$ by $\bm{A}_t^\top$ in computation $\mathcal{O}(nr^2)$, then multiply $(\bm{A}_t \bm{R}_t^{\frac{1}{2}} \bm{A}_t^\top)^{-1}$ in computation $\mathcal{O}(r^3)$, and multiply $\bm{A}_t$ in $\mathcal{O}(nr^2)$.
\end{itemize}
The computation complexity of $\bm{\Delta}_{\bm{A}_t}$ is $\mathcal{O}(mn+(m+n)r^2+r^3)$. The computation of $\bm{\Delta}_{\bm{B}_t}$ follows a similar structure, with symmetric terms. Its complexity is also $\mathcal{O}(mn+(m+n)r^2+r^3)$. Then we have
\begin{itemize}
\item  \textbf{Per Iteration Computational Complexity.} Combining the computations of $\bm{\Delta}_{\bm{A}_t}$ and $\bm{\Delta}_{\bm{B}_t}$, the total computation complexity per iteration is $\mathcal{O}(mn+(m+n)r^2+r^3)$.\\
\item \textbf{Memory Complexity.}  The algorithm requires storing the vectors $\bm{l}_t$ and $\bm{r}_t$ in each iteration, hence  the memory complexity is $\mathcal{O}(m+n)$.
\end{itemize}

\section{Supplementary Experiments of GPT-2 Fine-tuning}\label{app:GPT2}

\subsection{Parameter Settings}\label{sec:gpt2-setting}
To ensure the reproducibility of the experiments described in Section \ref{sec:experiment} and to facilitate verification and comparison by others, we provide the complete details of the experimental parameter settings. Tables \ref{tab:GPT2-parameter} and \ref{tab:GPT2-learning_rate} list the parameters used during the fine-tuning of GPT-2 models and the learning rates corresponding to different optimizers, respectively.  Specifically, we conduct experiments with GPT-2 models of various sizes. ``Rank 4 (M)'' represents a medium-sized model using LoRA with rank 4, while ``Rank 4'', ``Rank 16'', and ``Rank 64'' represent small models using LoRA with ranks 4, 16, and 64, respectively.  
To ensure the fairness of the experimental setup, we follow the parameter settings in LoRA~\citep{hu2022lora} and Riemannian Preconditioned LoRA~\citep{zhang2024RiemannianPreconditioned}. However, considering the sensitivity of different optimizers to learning rates, we use a grid search strategy to independently tune the optimal learning rate for each optimizer. This ensures that each optimizer operates under its best-performing configuration, providing more objective and reliable experimental results.

\begin{table}[htbp]
\centering
\caption{Training and Inference Configuration for GPT-2 Fine-tuning.}\label{tab:GPT2-parameter}
\begin{tabular}{lc lc lc}
\toprule
\multicolumn{2}{c}{\textbf{Training}} & \multicolumn{2}{c}{\textbf{LoRA $\alpha$}} & \multicolumn{2}{c}{\textbf{Inference}} \\
\midrule
Parameter & Value & Parameter & Value & Parameter & Value \\
\midrule
Dropout Probability & 0.1 &  &  &  &  \\
Batch Size & 8 &  &  &  &  \\
Number of Epochs & 5 & $\alpha$ (for Rank 4) & 32  & Beam Size & 10  \\
Warm-up Steps & 500 & $\alpha$ (for Rank 16) & 32 & Length Penalty & 0.8 \\
Learning Rate Scheduler & Linear & $\alpha$ (for Rank 64) & 128  & No Repeat Ngram Size & 4 \\
Label Smoothing & 0.1 & & & & \\
Weight Decay & 0.01 & & & & \\
\bottomrule
\end{tabular}
\end{table}

\begin{table}[htbp]
\centering
\caption{Core Optimizer Parameters for GPT-2 fine-tuning.}\label{tab:GPT2-learning_rate}
\begin{tabular}{ccccccc}
\toprule
\multirow{2}{*}{Methods} & \multicolumn{4}{c}{Learning Rate ($\times 10^{-3}$)} & \multirow{2}{*}{$\beta_3$} & \multirow{2}{*}{$\beta_1=\beta_2$} \\ 
 
&  Rank 4 & Rank 4 (M) & Rank 16 & Rank 64 \\
\midrule
SGD & 90 & 90 & 200 & 90 & / &  / \\
Scaled GD & 20 & 20 & 40 & 10 & / & / \\
LoRA-Pro SGD & 40 & 40 & 40 & 40 & / & / \\
FILO-Adafactor & 0.05 & 0.05 & 0.5 & 0.8 & / & 0.98 \\
AdamW & 0.2 & 0.2 & 0.2 & 0.2 & 0.9 & 0.999 \\
Scaled AdamW & 0.8 & 0.8 & 2 & 4 & 0.7 & 0.8 \\
LoRA-Pro AdamW & 0.1 & 0.1 & 0.2 & 0.4 &  0.9 & 0.999 \\
FILO-AdafactorM & 0.5 & 0.1 & 0.8 & 0.3 & 0.9 & 0.98 \\
\bottomrule
\end{tabular}
\end{table}

\subsection{Cross-rank ablation on GPT-2 small}\label{sec:gpt2-rank-ablation}
To complement the $r=4$ comparisons in the main text (Table~\ref{tab:score_GPT2}), we report cross-rank scores at $r \in \{16, 64\}$ on GPT-2 small in Table~\ref{tab:score_GPT2_rank_ablation}. Hyperparameters follow Table~\ref{tab:GPT2-parameter} and Table~\ref{tab:GPT2-learning_rate}. Across both ranks and both optimizer families, FILO-Adafactor achieves the best or tied-best score on essentially every metric; gains over Scaled GD/AdamW and LoRA-Pro narrow at higher rank but remain positive, consistent with the main-text observation that the gradient-statistics-aware projection retains an advantage even when the LoRA factorization is less rank-constrained.

\begin{table}[htbp]
\caption{GPT-2 small fine-tuned on E2E at $r \in \{16, 64\}$ (cross-rank ablation). Bold/underline = best/second-best per metric per (rank, optimizer family).}\label{tab:score_GPT2_rank_ablation}
\centering
\begin{tabular}{cccccccc}
\toprule
\multirow{2}{*}{$r$} & \multirow{2}{*}{Method} & \multicolumn{5}{c}{E2E} \\
\cmidrule(lr){3-7}
& & BLEU & NIST & MET & ROUGE-L & CIDEr \\
\midrule
\multirow{8}{*}{16} & SGD & 65.4 & 8.07 & 40.7 & 67.0 & 2.07 \\
& Scaled GD & \underline{68.8} & \underline{8.75} & 45.0 & 69.2 & \underline{2.39} \\
& LoRA-Pro SGD & 68.3 & 8.67 & \underline{45.1} & \underline{69.3} & 2.37 \\
& \textbf{FILO-Adafactor (ours)} & \textbf{70.0} & \textbf{8.82} & \textbf{46.6} & \textbf{71.6} & \textbf{2.53} \\
\cmidrule{2-7}
& AdamW & 69.5 & 8.77 & 46.4 & 71.2 & 2.48 \\
& Scaled AdamW & \underline{69.8} & \underline{8.79} & 46.5 & \underline{71.7} & \underline{2.51} \\
& LoRA-Pro AdamW & 69.7 & 8.73 & \textbf{46.8} & \underline{71.7} & \underline{2.51} \\
& \textbf{FILO-AdafactorM (ours)} & \textbf{70.2} & \textbf{8.85} & \underline{46.6} & \textbf{71.9} & \textbf{2.52} \\
\midrule
\multirow{8}{*}{64} & SGD & 64.7 & 8.08 & 40.8 & 66.7 & 2.04 \\
& Scaled GD & 68.5 & 8.68 & 45.0 & 69.4 & \underline{2.38} \\
& LoRA-Pro SGD & \underline{68.6} & \underline{8.71} & \underline{45.4} & \underline{69.7} & \underline{2.38} \\
& \textbf{FILO-Adafactor (ours)} & \textbf{70.1} & \textbf{8.85} & \textbf{46.7} & \textbf{71.8} & \textbf{2.53} \\
\cmidrule{2-7}
& AdamW & 69.6 & 8.76 & \underline{46.7} & \underline{71.5} & 2.50 \\
& Scaled AdamW & \underline{70.0} & \underline{8.83} & 46.4 & \underline{71.5} & 2.50 \\
& LoRA-Pro AdamW & \underline{70.0} & 8.82 & 46.6 & \underline{71.5} & \underline{2.51} \\
& \textbf{FILO-AdafactorM (ours)} & \textbf{70.2} & \textbf{8.84} & \textbf{46.8} & \textbf{72.1} & \textbf{2.52} \\
\bottomrule
\end{tabular}
\end{table}

\subsection{Training Efficiency Comparison}
To validate the training and inference efficiency of FILO-Adafactor, we report in Table~\ref{tab:gpt2-time} the total training and inference time required for all algorithms on the GPT-2 small model (rank 64) fine-tuned on E2E.

\begin{table}[htbp]
\centering
\caption{Training and inference time (hours) of GPT-2 small (rank=64) fine-tuned with different optimizers on E2E.}
\label{tab:gpt2-time}
\setlength{\tabcolsep}{4pt}
\begin{tabular}{lcccc}
\toprule
Method & SGD & Scaled GD & LoRA-Pro SGD & FILO-Adafactor \\
\midrule
Training  & 1.79 & 1.92 & 2.78 & 2.04 \\
Inference & 1.86 & 1.87 & 1.58 & 1.89 \\
\midrule
Method & AdamW & Scaled AdamW & LoRA-Pro AdamW & FILO-AdafactorM \\
\midrule
Training  & 1.79 & 1.94 & 2.93 & 2.04 \\
Inference & 1.87 & 1.88 & 1.89 & 1.86 \\
\bottomrule
\end{tabular}
\end{table}

\section{Supplementary Experiments of Diffusion Model Fine-tuning}\label{sec:diffusion}

We evaluate FILO-Adafactor on diffusion-model personalization with the Mix-of-Show framework~\citep{gu2023mix-of-show}, which integrates Embedding Decomposed LoRA (EDLoRA) into the text encoder and U-Net of a Stable Diffusion backbone. Following~\citep{zhang2024RiemannianPreconditioned, gu2023mix-of-show}, we disable embedding-vector tuning and fine-tune only the LoRA components of the text encoder and U-Net submodules. The CLIP score~\citep{hessel2021clipscore} (ViT-B/32 variant of the CLIP model~\citep{radford2021clip-model}) measures alignment between generated images and text prompts (range $[0, 100]$, higher better); FID~\citep{heusel2017fid-distance} measures the distributional similarity between generated and reference images (lower better). Aggregate CLIP/FID scores at LoRA scaling factors $\{0.7, 1.0\}$ are in the main text (Table~\ref{tab:fid_clip}); this section reports per-prompt qualitative samples at the same two scaling factors, using the per-optimizer learning rates listed in Table~\ref{tab:diffusion-learning_rate}. At $s=1.0$, Figures~\ref{fig:potter_adamw_fuji_1.0} and~\ref{fig:hermione_adamw_sea_1.0} compare AdamW-based optimizers on Harry Potter and Hermione Granger prompts respectively, and Figures~\ref{fig:potter_sgd_lake_1.0} and~\ref{fig:hermione_sgd_shirt_1.0} compare SGD-based optimizers on the same prompts. At $s=0.7$, Figures~\ref{fig:potter_adamw_hat_0.7} and~\ref{fig:hermione_adamw_sea_0.7} report the corresponding AdamW comparisons. Across both scaling factors, FILO-Adafactor preserves character identity and follows the prompt scene layout while remaining visually consistent across prompts.

\begin{table}[t]
\centering
\caption{Optimizer Parameters for fine-tuning the Mix-of-Show Model.}\label{tab:diffusion-learning_rate}
\begin{tabular}{ccccccc}
\toprule
\multirow{2}{*}{Methods} & \multicolumn{2}{c}{Learning Rate } & \multirow{2}{*}{$\beta_3$} & \multirow{2}{*}{$\beta_1=\beta_2$} \\ 
&  Text-Encoder & U-Net  \\
\midrule
SGD & 1e-1 & 1e-1   & / &  / \\
Scaled GD & 1e-1 & 1e-1   & / & / \\
LoRA-Pro SGD & 1e-1 & 1e-1   & / & / \\
FILO-Adafactor & 1e-5 & 1e-5   & / & 0.98 \\
AdamW & 1e-5 & 1e-4 & 0.9   & 0.999 \\
Scaled AdamW & 1e-5 & 1e-4   & 0.7 & 0.8 \\
LoRA-Pro AdamW & 1e-5 & 1e-5   &  0.9 & 0.999 \\
FILO-AdafactorM & 1e-5 & 1e-5   & 0.9 & 0.98 \\
\bottomrule
\end{tabular}
\end{table}

\begin{figure}[h]
\centering
AdamW with s=1.0 \par
\includegraphics[width=0.12\textwidth]{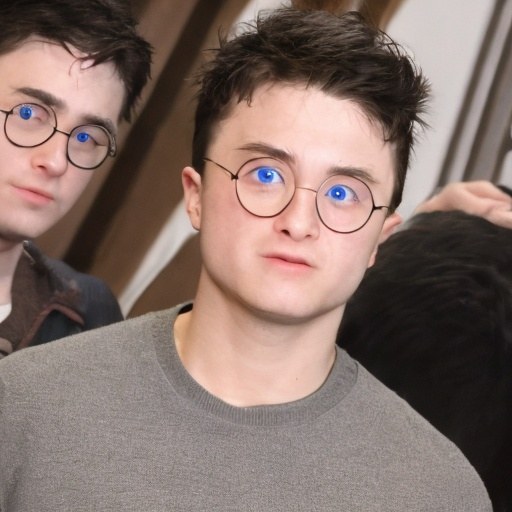}
\includegraphics[width=0.12\textwidth]{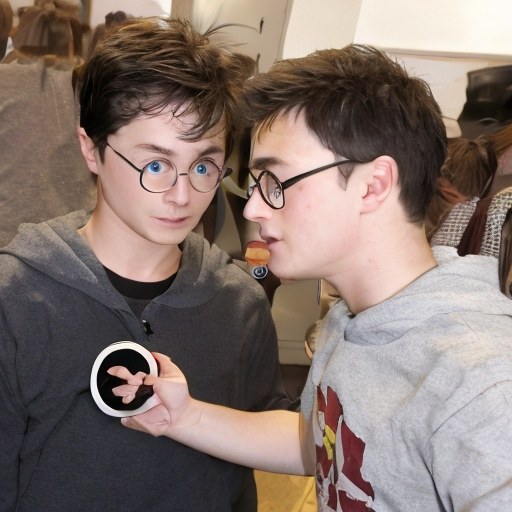}
\includegraphics[width=0.12\textwidth]{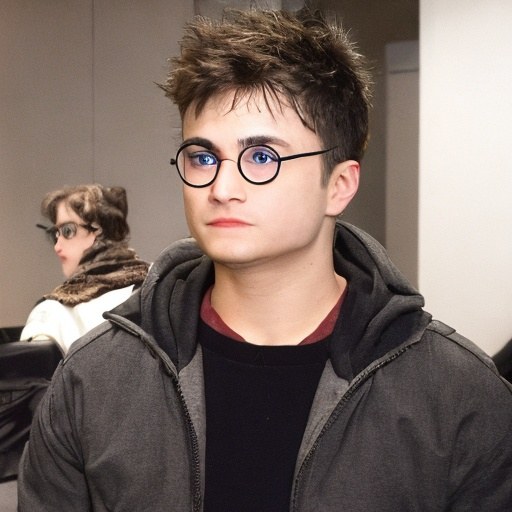}
\includegraphics[width=0.12\textwidth]{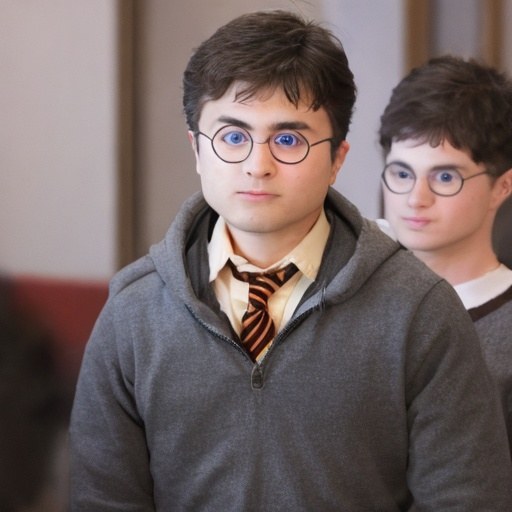}
\includegraphics[width=0.12\textwidth]{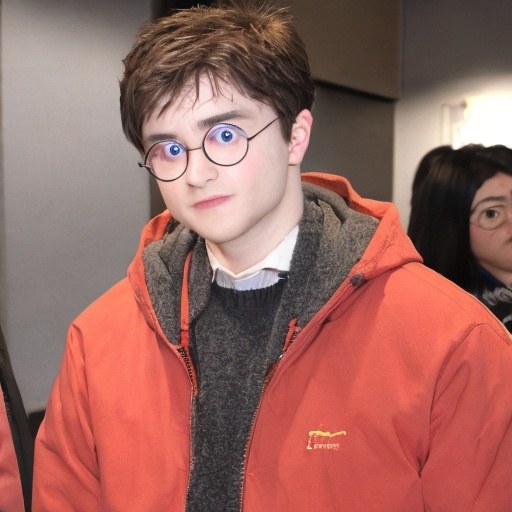} \\
Scaled AdamW with s=1.0 \par
\includegraphics[width=0.12\textwidth]{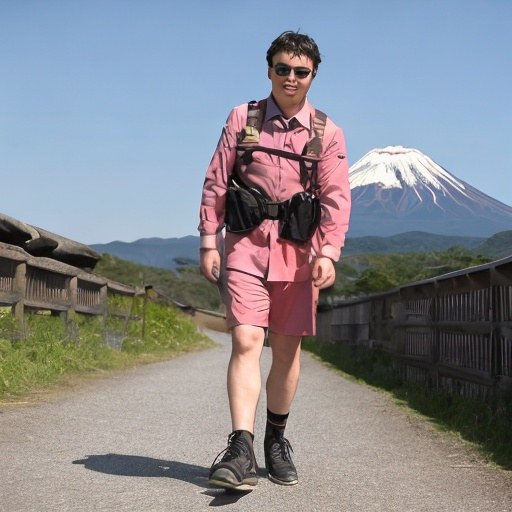}
\includegraphics[width=0.12\textwidth]{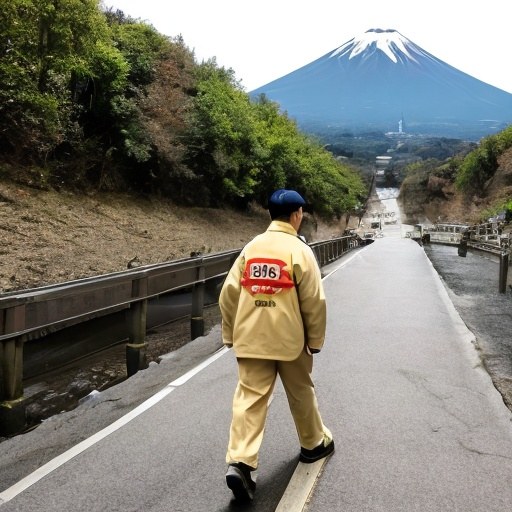}
\includegraphics[width=0.12\textwidth]{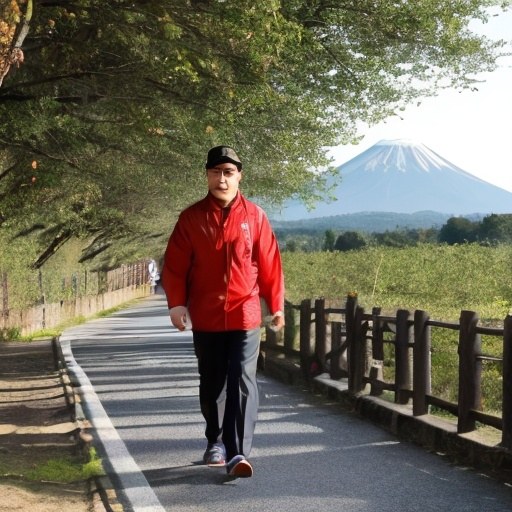}
\includegraphics[width=0.12\textwidth]{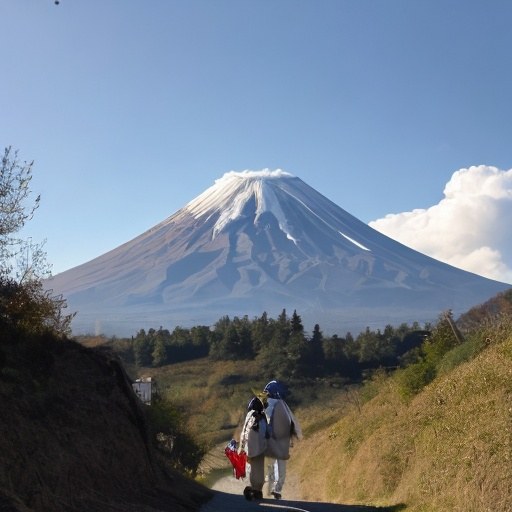}
\includegraphics[width=0.12\textwidth]{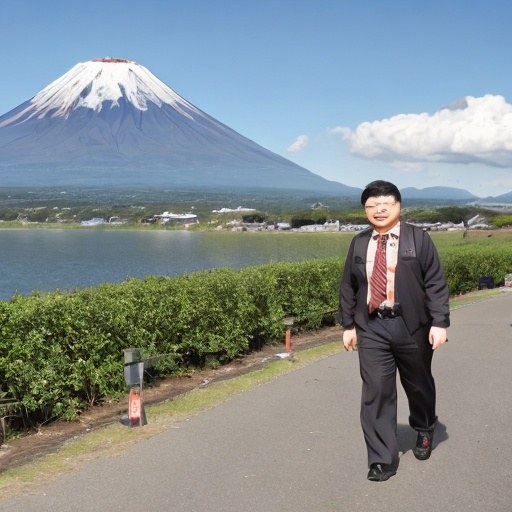} \\
LoRA-Pro AdamW with s=1.0 \par
\includegraphics[width=0.12\textwidth]{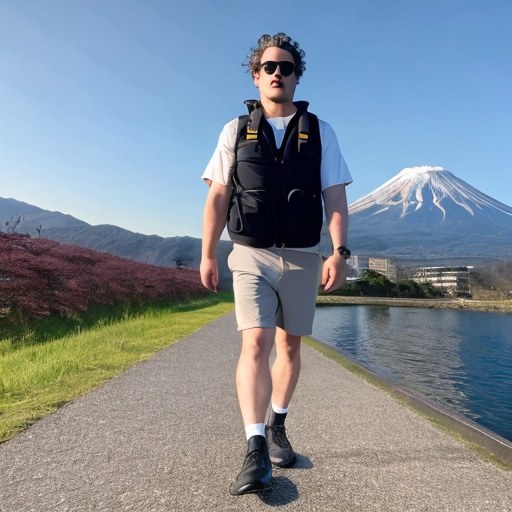}
\includegraphics[width=0.12\textwidth]{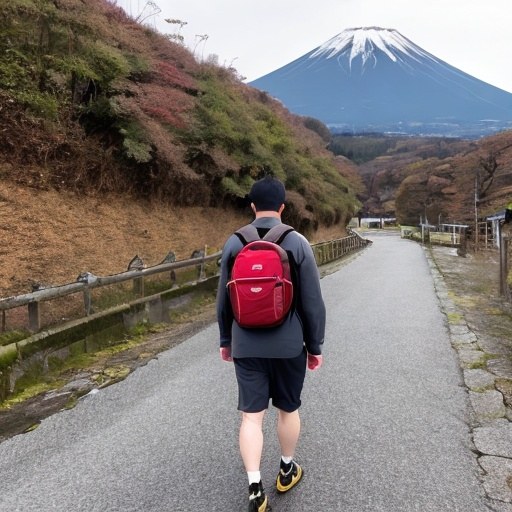}
\includegraphics[width=0.12\textwidth]{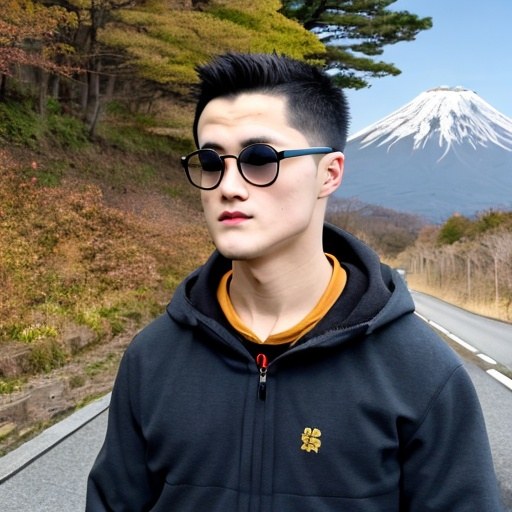}
\includegraphics[width=0.12\textwidth]{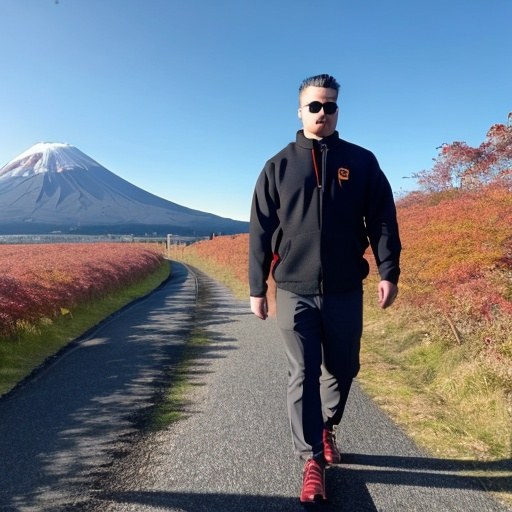}
\includegraphics[width=0.12\textwidth]{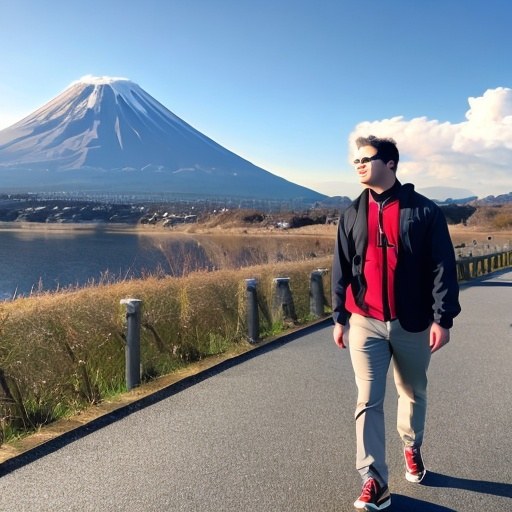} \\
\textbf{FILO-AdafactorM with s=1.0 (ours)} \par
\includegraphics[width=0.12\textwidth]{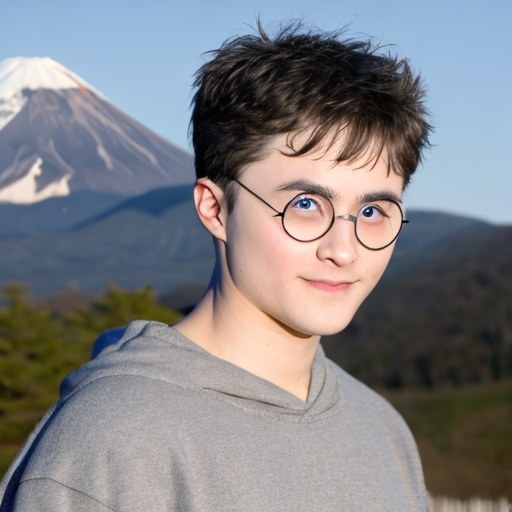}
\includegraphics[width=0.12\textwidth]{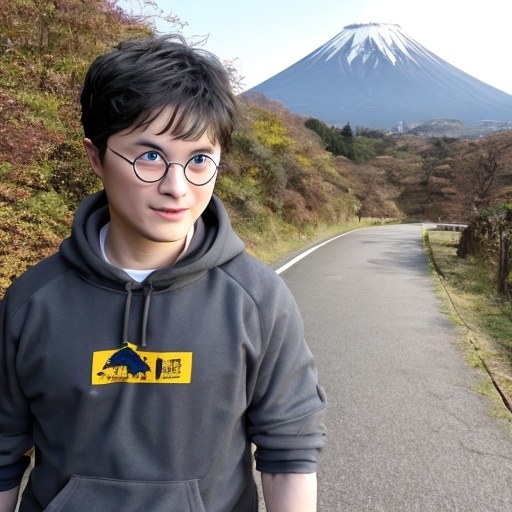}
\includegraphics[width=0.12\textwidth]{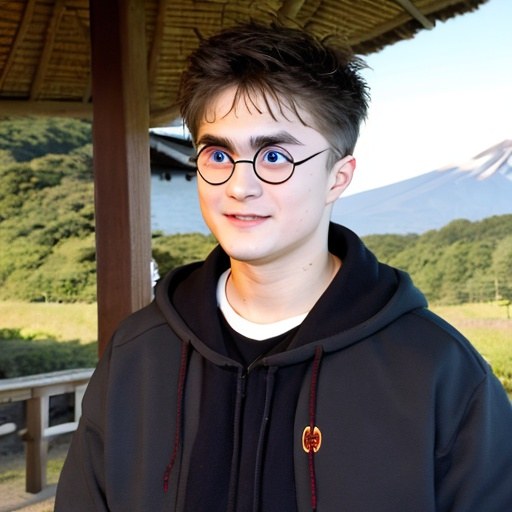}
\includegraphics[width=0.12\textwidth]{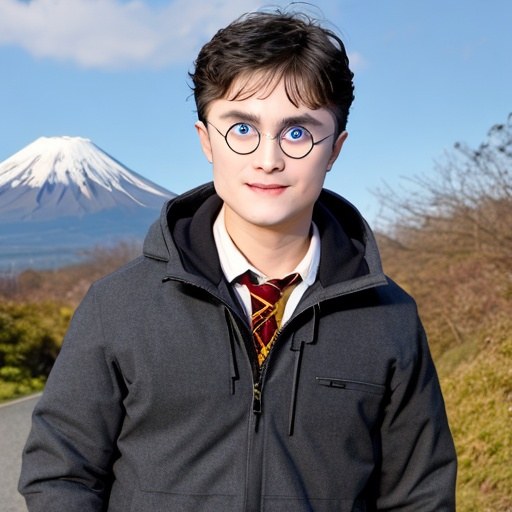}
\includegraphics[width=0.12\textwidth]{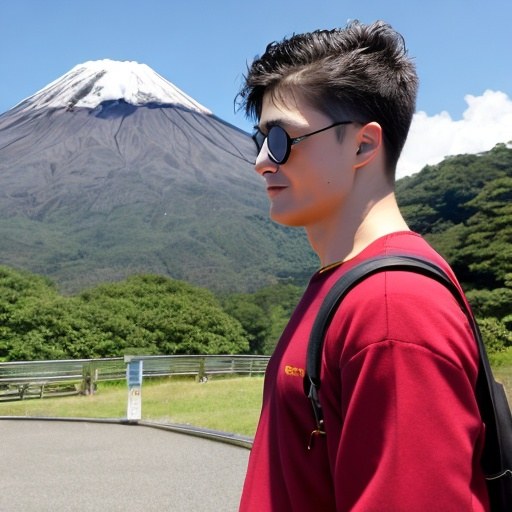}
\caption{Generated results based on the prompt ``Harry Potter is walking near Mount Fuji'' when fine-tuned using AdamW-based optimizers. All optimizers employed a LoRA scaling factor of 1.0, with the best learning rate. The results indicate that the output of the model trained with our optimizer incorporates the character ``Harry Potter'', the action``walking'', and the scene ``Mount Fuji'', yielding superior image quality compared to alternative approaches.}
\label{fig:potter_adamw_fuji_1.0}
\end{figure}

\begin{figure}[htbp] 
\centering
AdamW with s=1.0 \par
\includegraphics[width=0.12\textwidth]{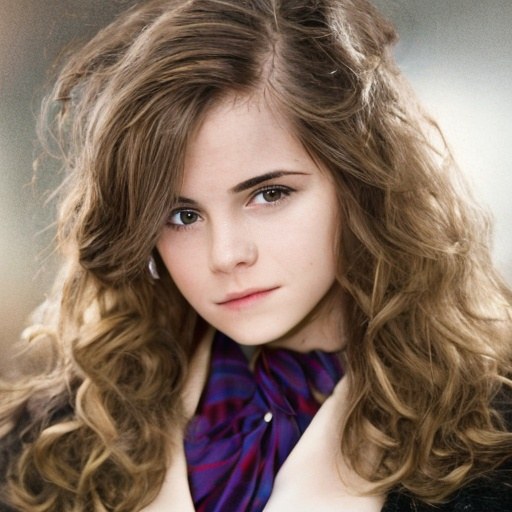}
\includegraphics[width=0.12\textwidth]{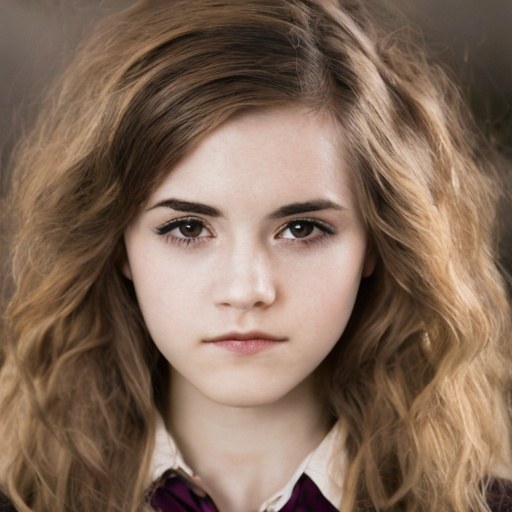}
\includegraphics[width=0.12\textwidth]{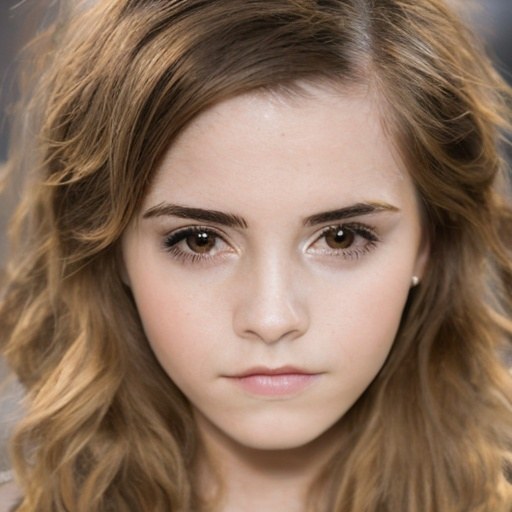}
\includegraphics[width=0.12\textwidth]{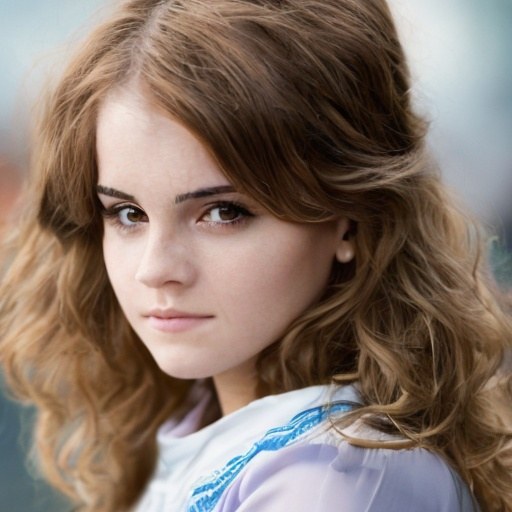}
\includegraphics[width=0.12\textwidth]{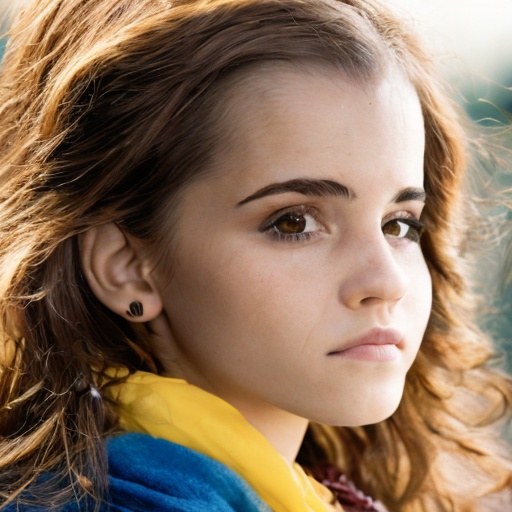} \\
Scaled AdamW with s=1.0 \par
\includegraphics[width=0.12\textwidth]{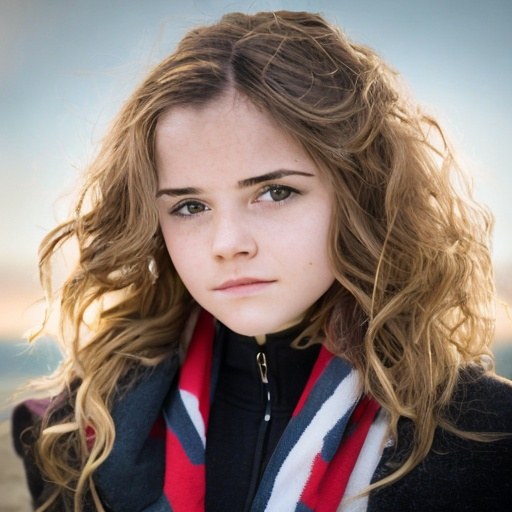}
\includegraphics[width=0.12\textwidth]{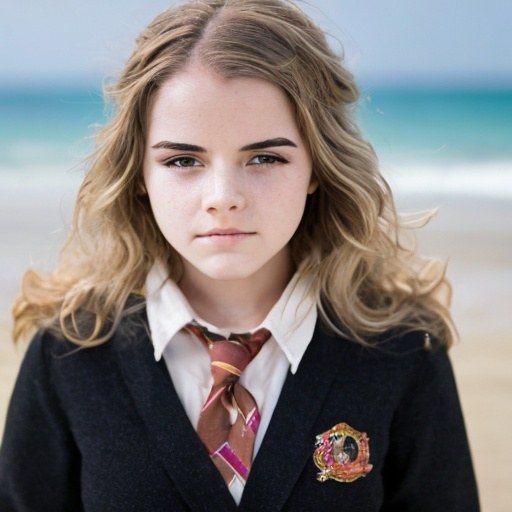}
\includegraphics[width=0.12\textwidth]{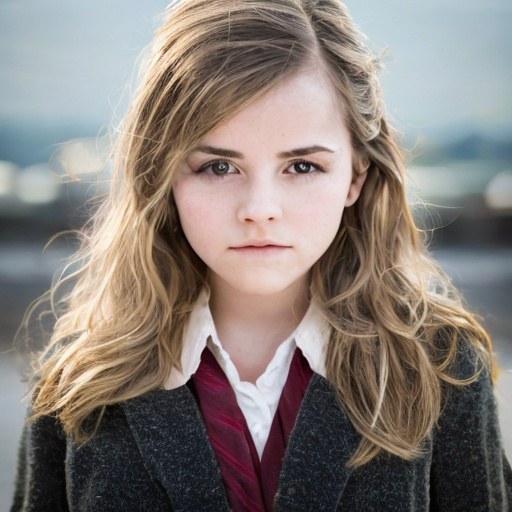}
\includegraphics[width=0.12\textwidth]{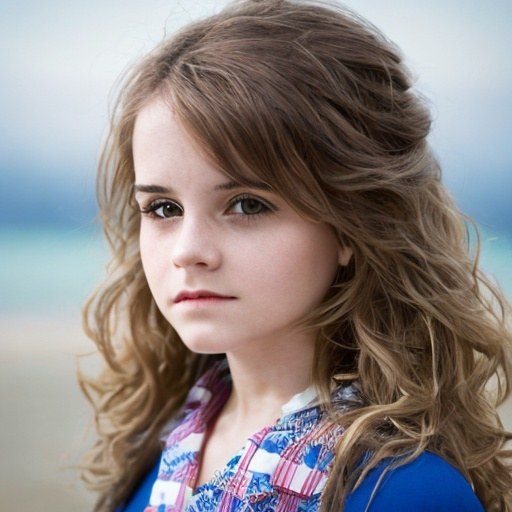}
\includegraphics[width=0.12\textwidth]{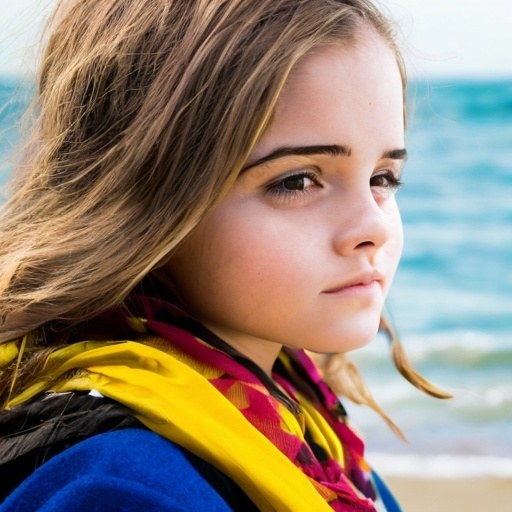} \\
LoRA-Pro AdamW with s=1.0 \par
\includegraphics[width=0.12\textwidth]{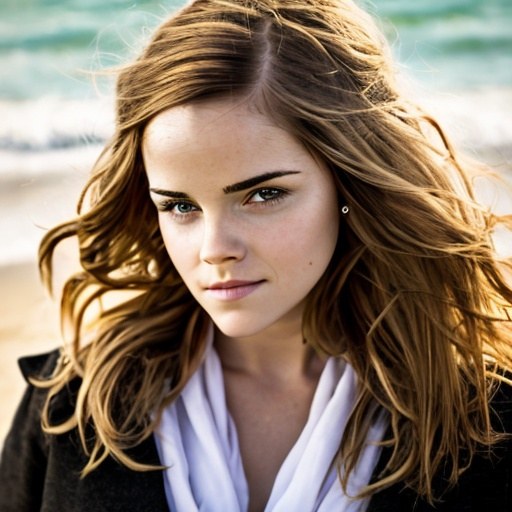}
\includegraphics[width=0.12\textwidth]{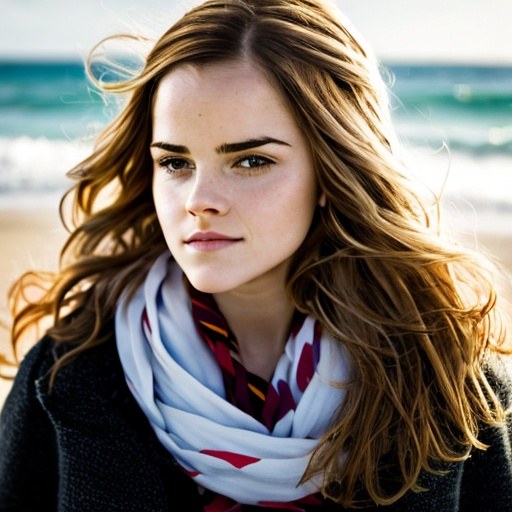}
\includegraphics[width=0.12\textwidth]{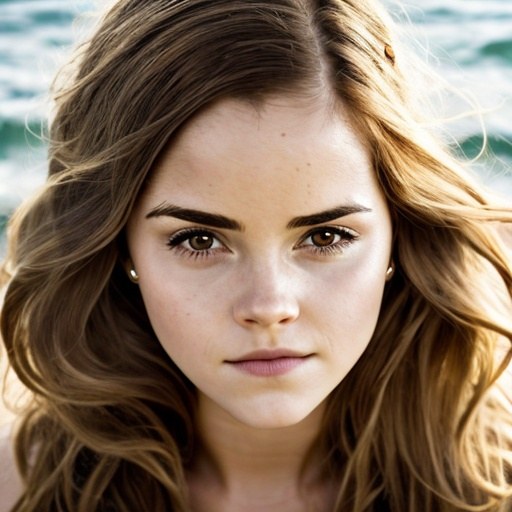}
\includegraphics[width=0.12\textwidth]{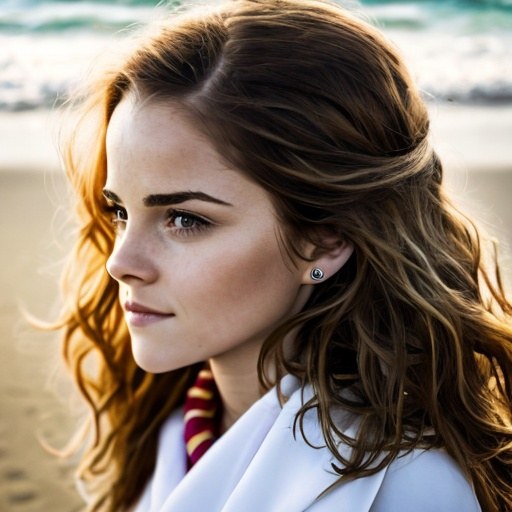}
\includegraphics[width=0.12\textwidth]{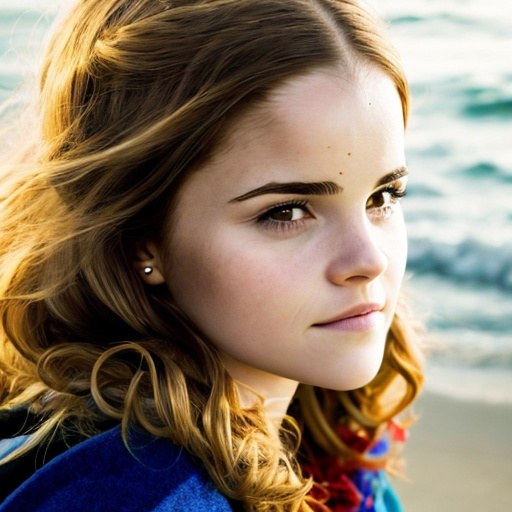} \\
\textbf{FILO-AdafactorM with s=1.0 (ours)} \par
\includegraphics[width=0.12\textwidth]{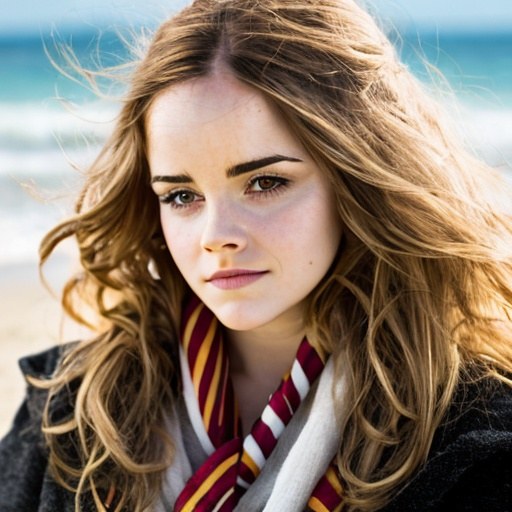}
\includegraphics[width=0.12\textwidth]{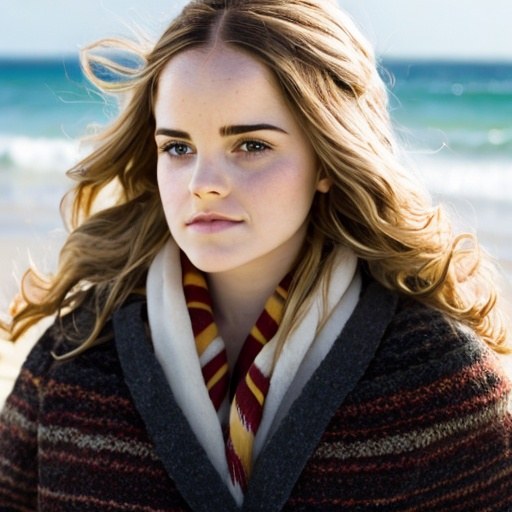}
\includegraphics[width=0.12\textwidth]{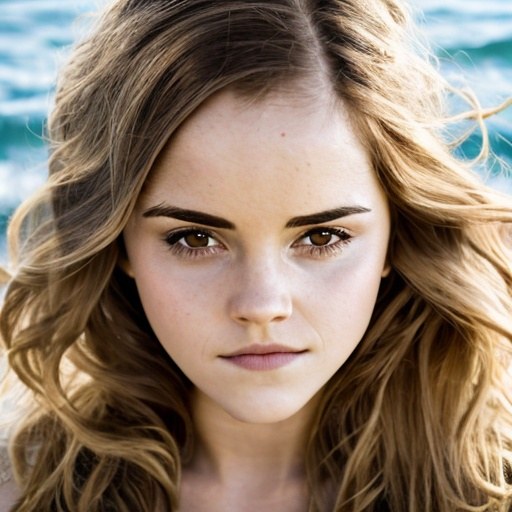}
\includegraphics[width=0.12\textwidth]{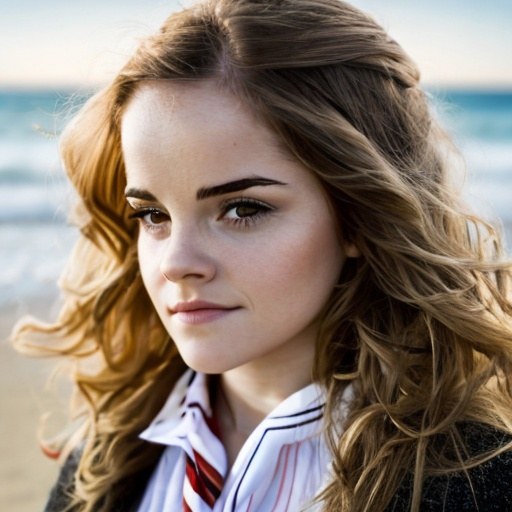}
\includegraphics[width=0.12\textwidth]{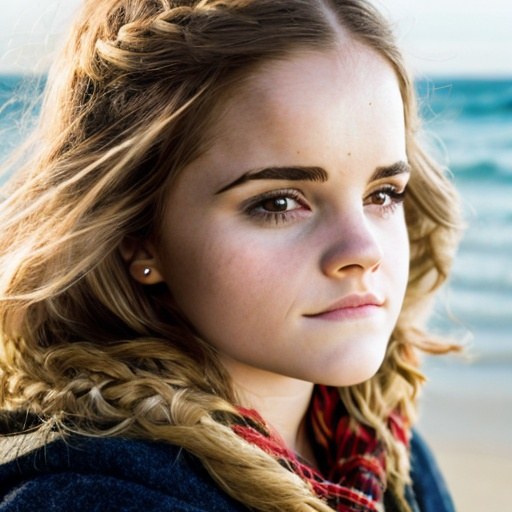}
\caption{Generation results from the prompt ``A photo of Hermione Granger on the beach, small waves, detailed symmetric face, beautiful composition'' using AdamW-based optimizers. All the optimizers apply LoRA scaling factor as 1.0, with the best learning rate. Results demonstrate that the model trained with our optimizer generates higher-quality images than others, especially the face of Hermione Granger and the scene.}\label{fig:hermione_adamw_sea_1.0}
\end{figure}

\begin{figure}[t]
\centering
SGD with s=1.0\par
\includegraphics[width=0.12\textwidth]{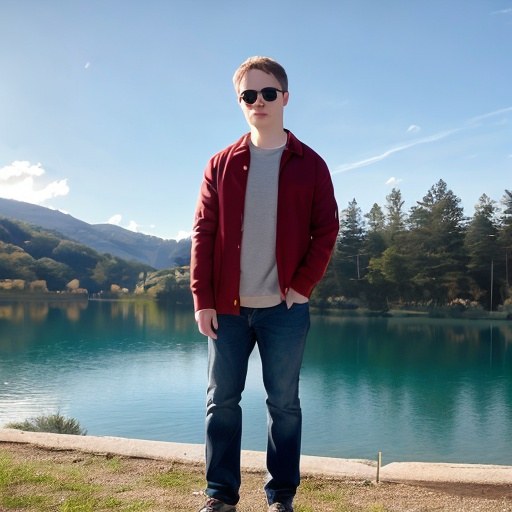}
\includegraphics[width=0.12\textwidth]{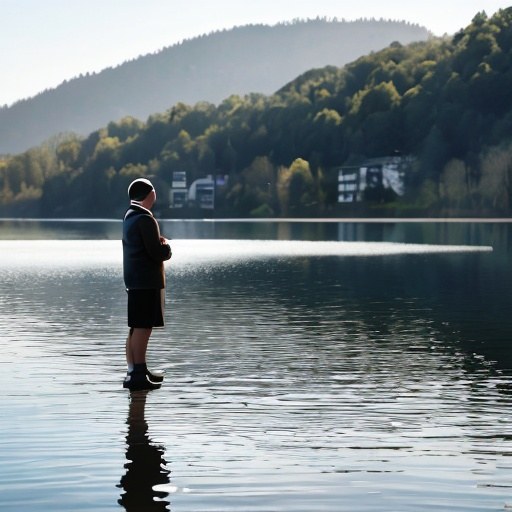}
\includegraphics[width=0.12\textwidth]{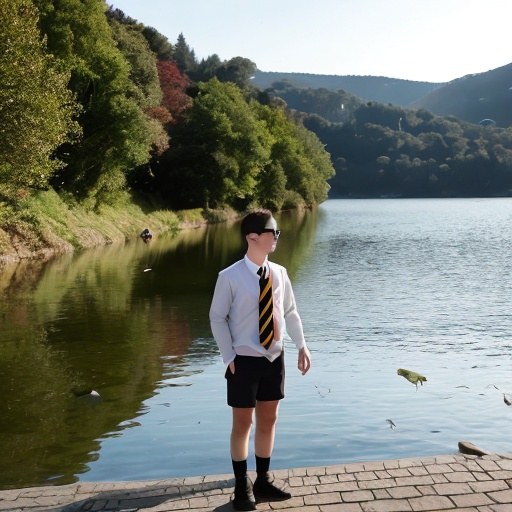}
\includegraphics[width=0.12\textwidth]{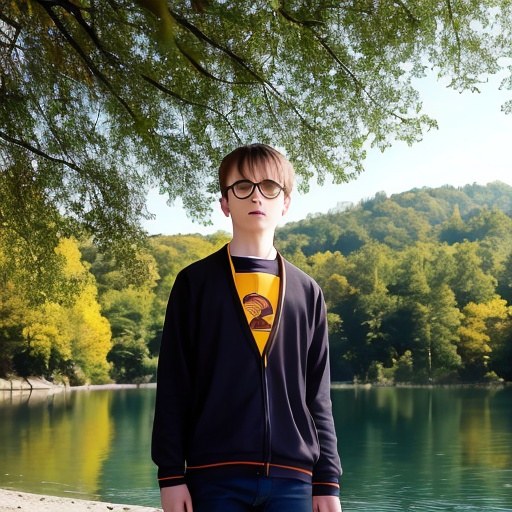}
\includegraphics[width=0.12\textwidth]{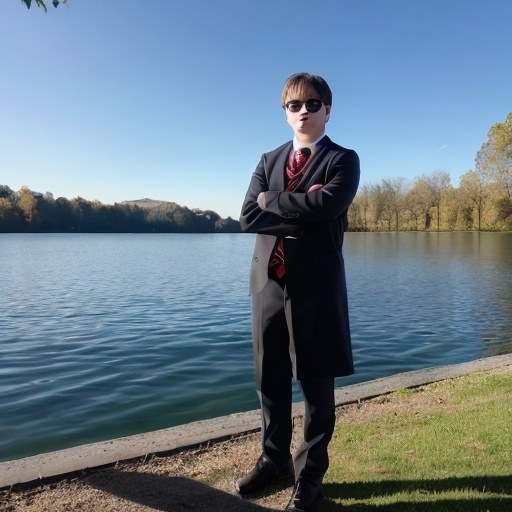} \\
Scaled GD with s=1.0 \par
\includegraphics[width=0.12\textwidth]{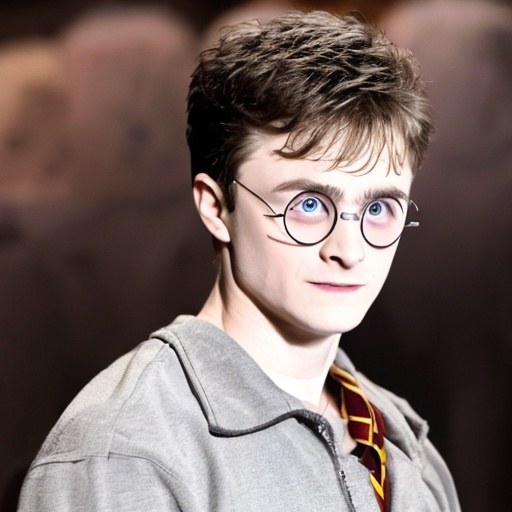}
\includegraphics[width=0.12\textwidth]{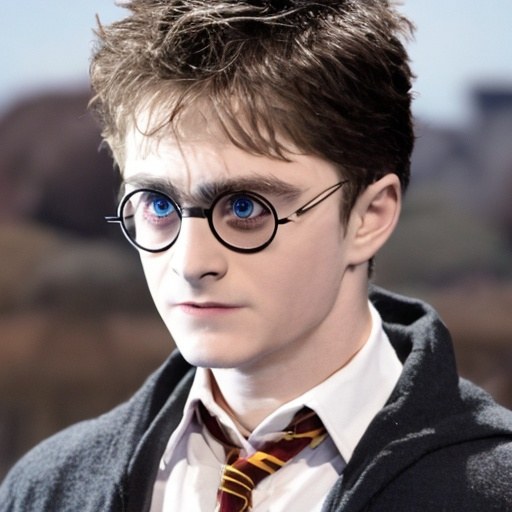}
\includegraphics[width=0.12\textwidth]{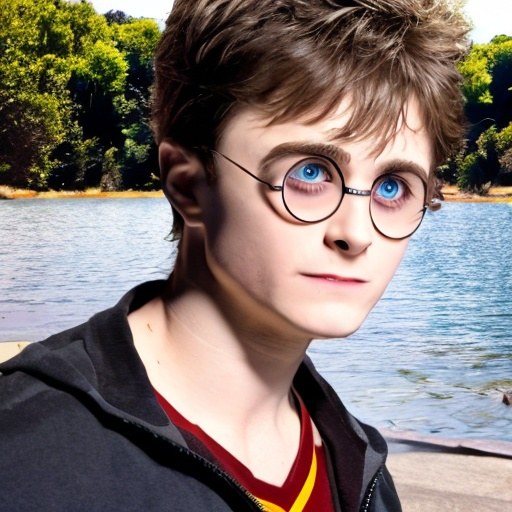}
\includegraphics[width=0.12\textwidth]{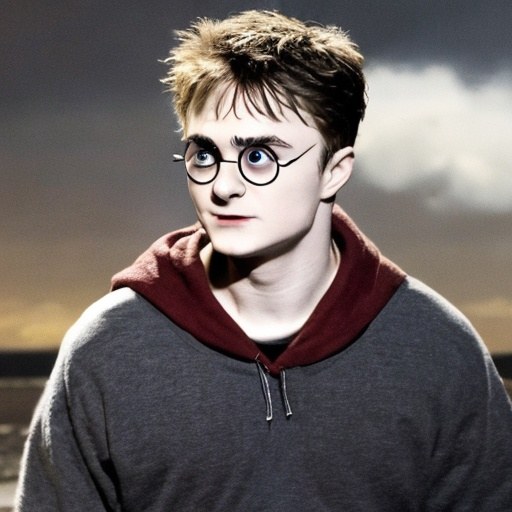}
\includegraphics[width=0.12\textwidth]{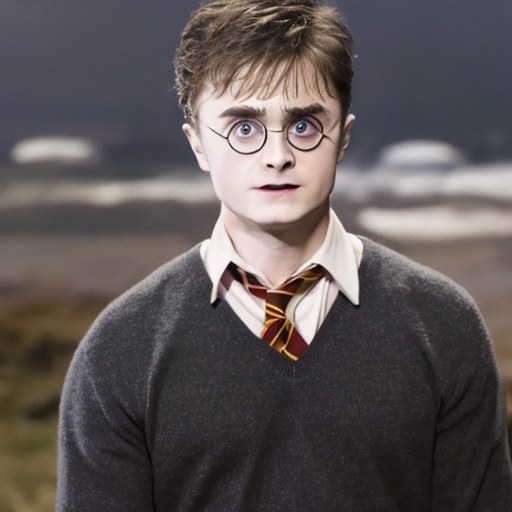} \\
LoRA-Pro SGD with s=1.0\par
\includegraphics[width=0.12\textwidth]{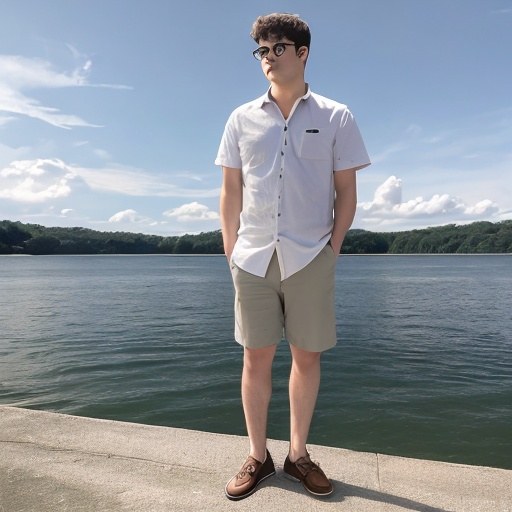}
\includegraphics[width=0.12\textwidth]{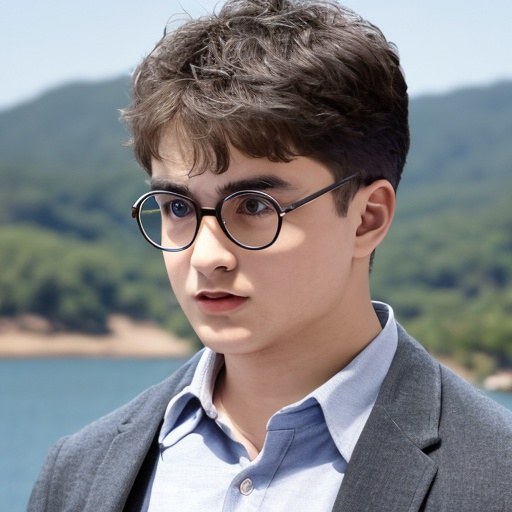}
\includegraphics[width=0.12\textwidth]{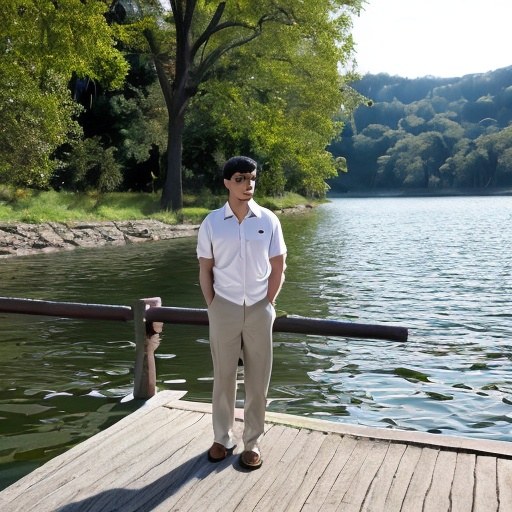}
\includegraphics[width=0.12\textwidth]{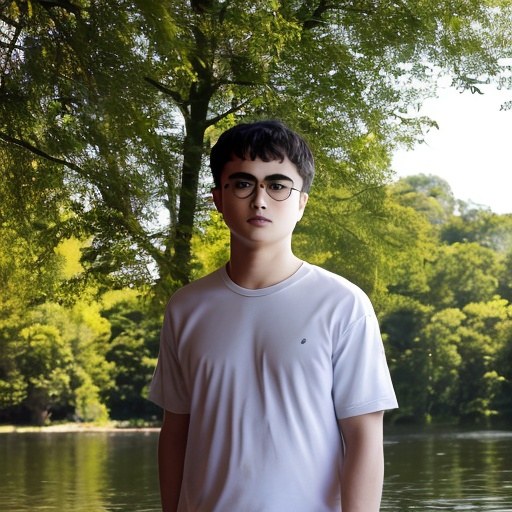}
\includegraphics[width=0.12\textwidth]{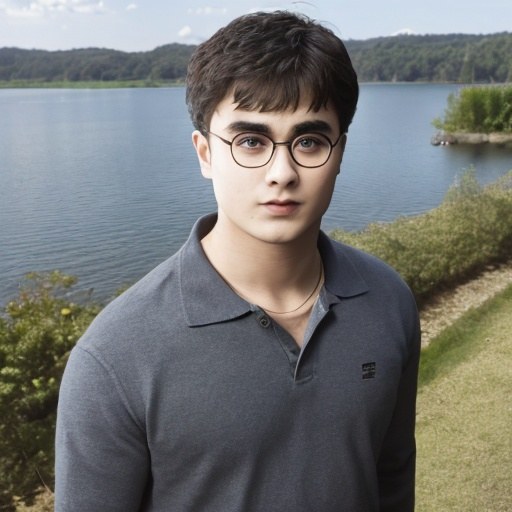} \\
\textbf{FILO-Adafactor with s=1.0 (ours)} \par
\includegraphics[width=0.12\textwidth]{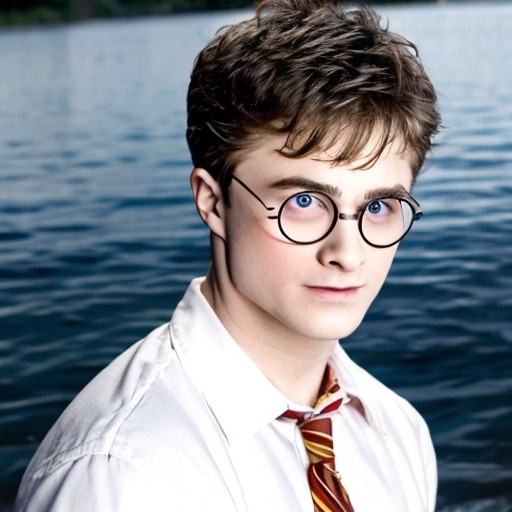}
\includegraphics[width=0.12\textwidth]{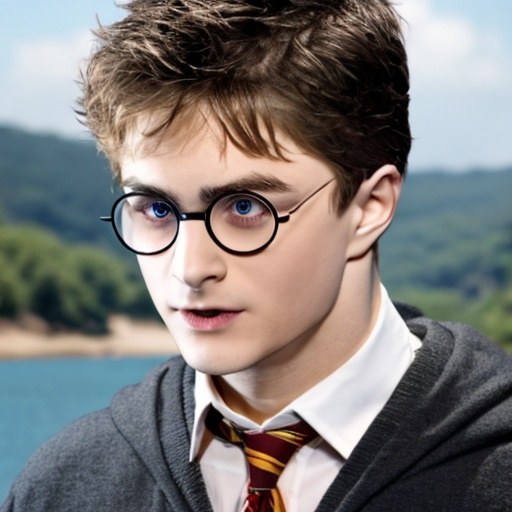}
\includegraphics[width=0.12\textwidth]{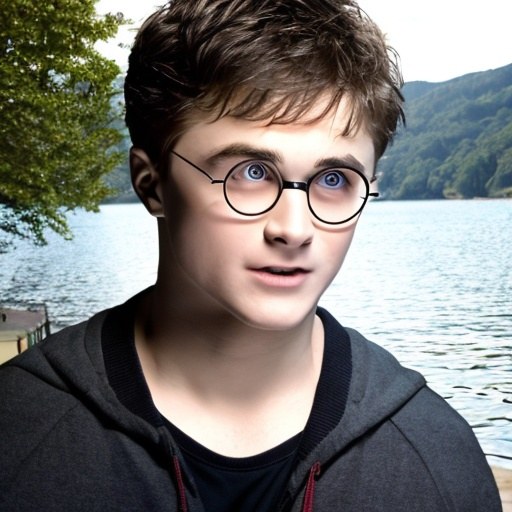}
\includegraphics[width=0.12\textwidth]{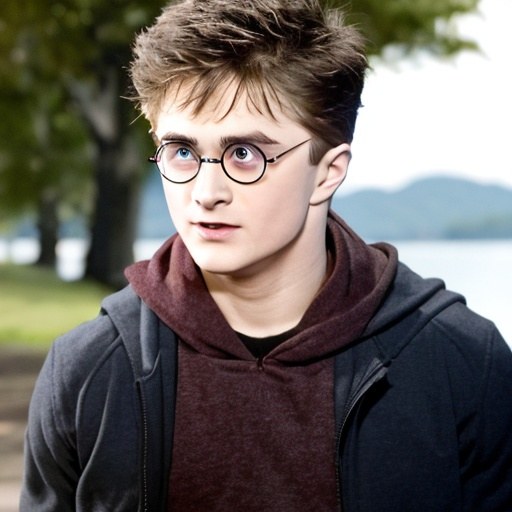}
\includegraphics[width=0.12\textwidth]{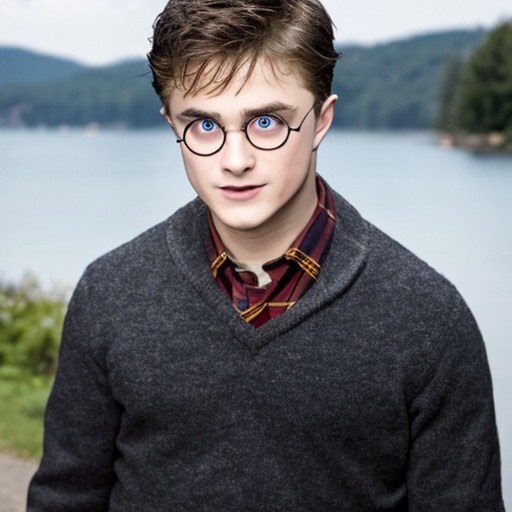}
\caption{Generated results based on the prompt ``Harry Potter standing near the lake'' when fine-tuned using SGD-based optimizers. All optimizers employed a LoRA scaling factor of 1.0, with the best learning rate. Results demonstrate that the output images of the model trained with our optimizer have higher-quality than others, especially the face of Harry Potter.}\label{fig:potter_sgd_lake_1.0}
\end{figure}

\begin{figure}[t]
\centering
SGD with s=1.0\par
\includegraphics[width=0.12\textwidth]{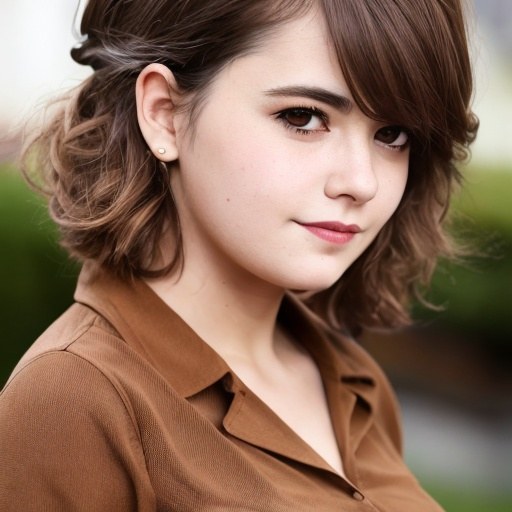}
\includegraphics[width=0.12\textwidth]{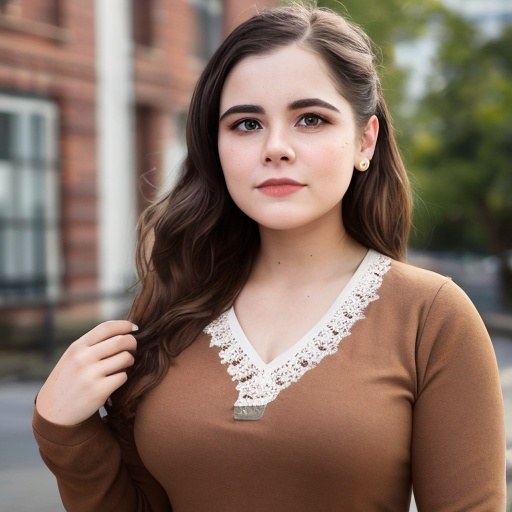}
\includegraphics[width=0.12\textwidth]{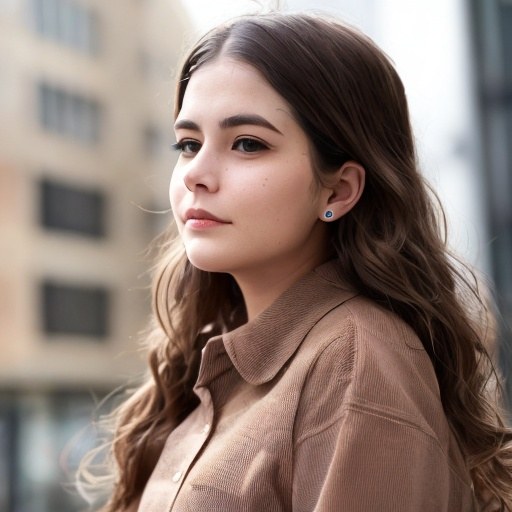}
\includegraphics[width=0.12\textwidth]{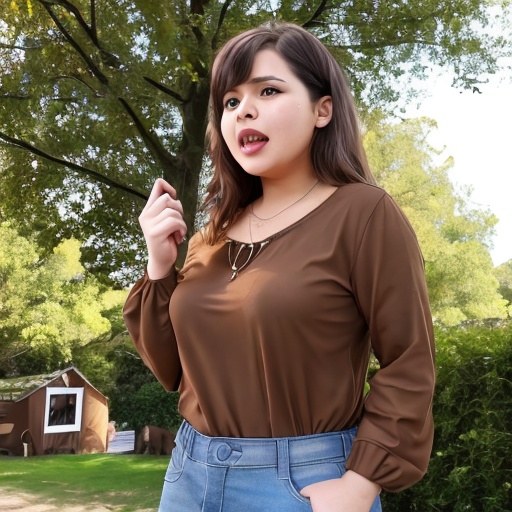}
\includegraphics[width=0.12\textwidth]{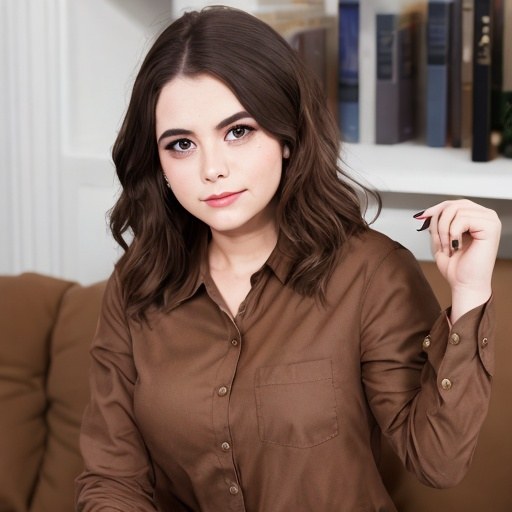} \\
Scaled GD with s=1.0 \par
\includegraphics[width=0.12\textwidth]{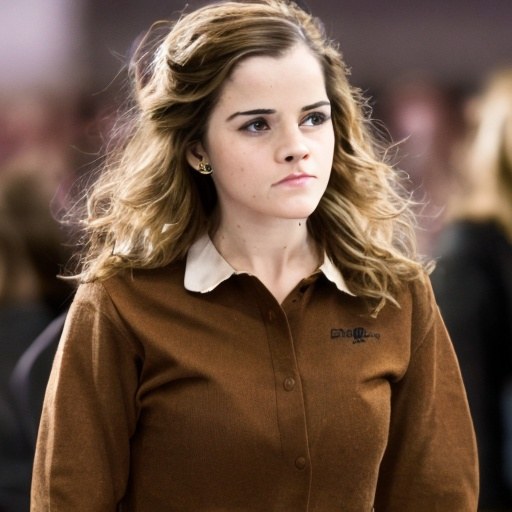}
\includegraphics[width=0.12\textwidth]{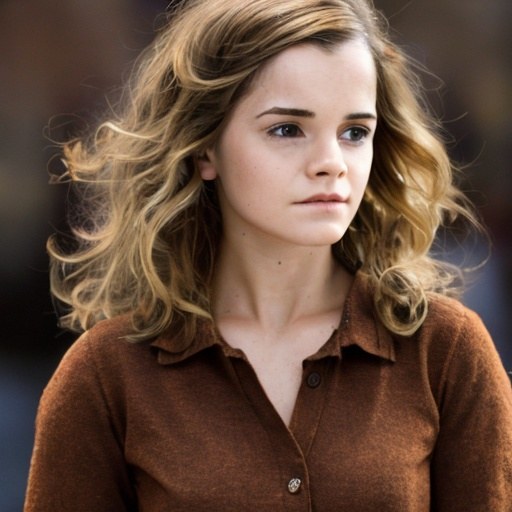}
\includegraphics[width=0.12\textwidth]{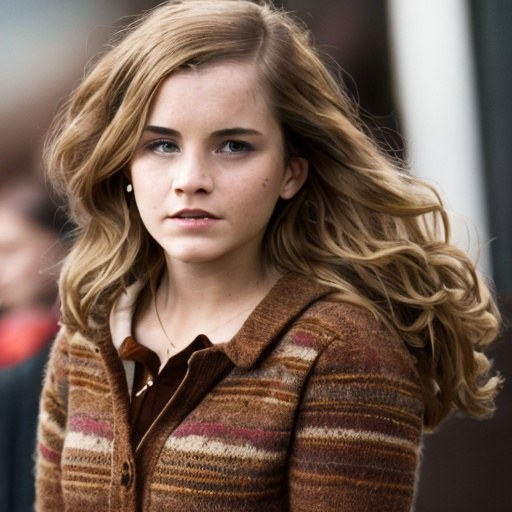}
\includegraphics[width=0.12\textwidth]{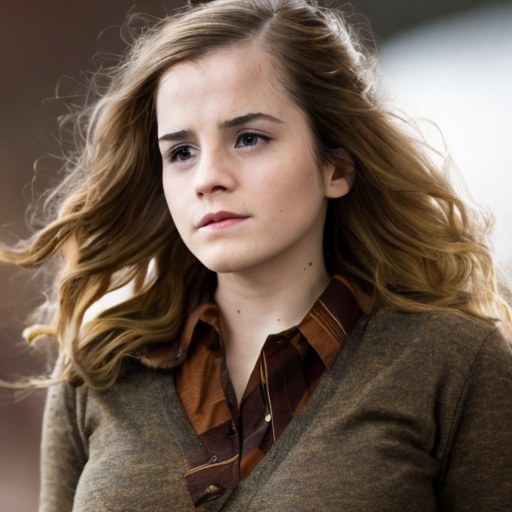}
\includegraphics[width=0.12\textwidth]{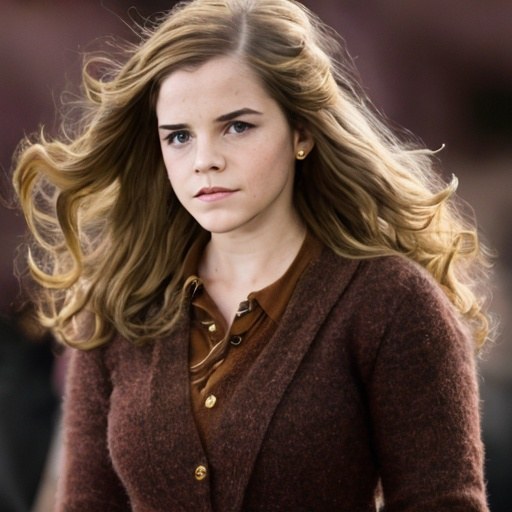} \\
LoRA-Pro SGD with s=1.0\par
\includegraphics[width=0.12\textwidth]{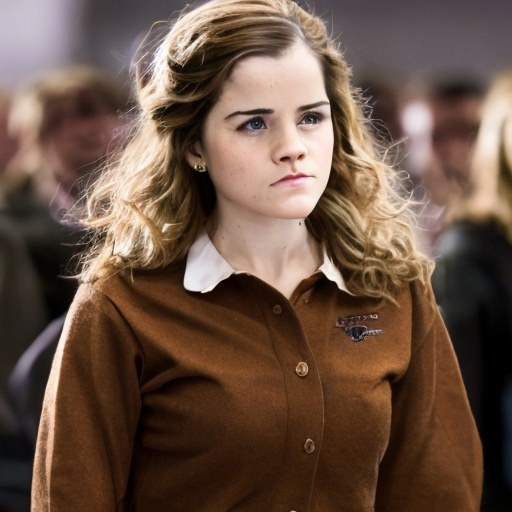}
\includegraphics[width=0.12\textwidth]{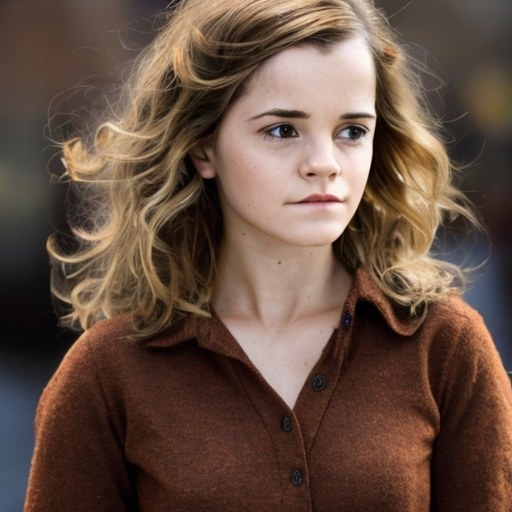}
\includegraphics[width=0.12\textwidth]{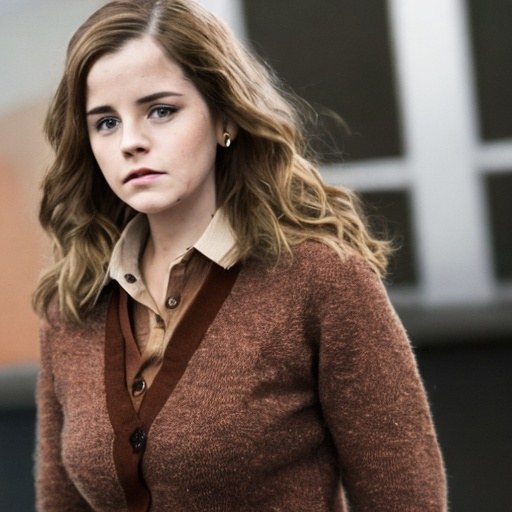}
\includegraphics[width=0.12\textwidth]{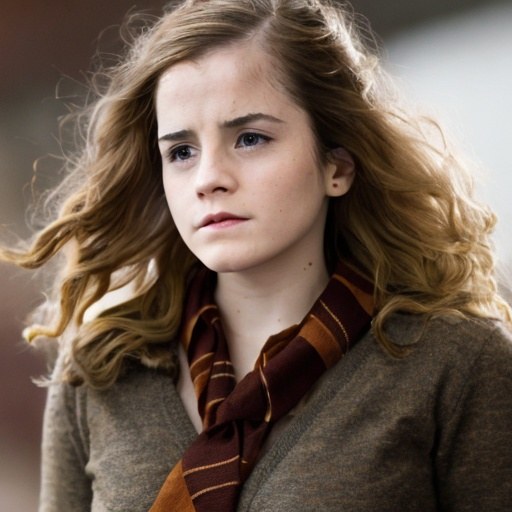}
\includegraphics[width=0.12\textwidth]{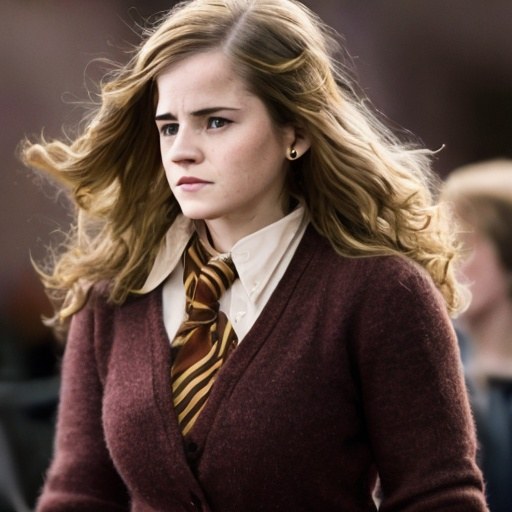} \\
\textbf{FILO-Adafactor with s=1.0 (ours)} \par
\includegraphics[width=0.12\textwidth]{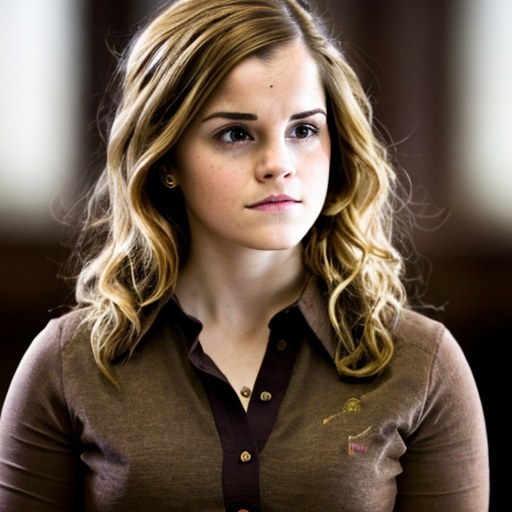}
\includegraphics[width=0.12\textwidth]{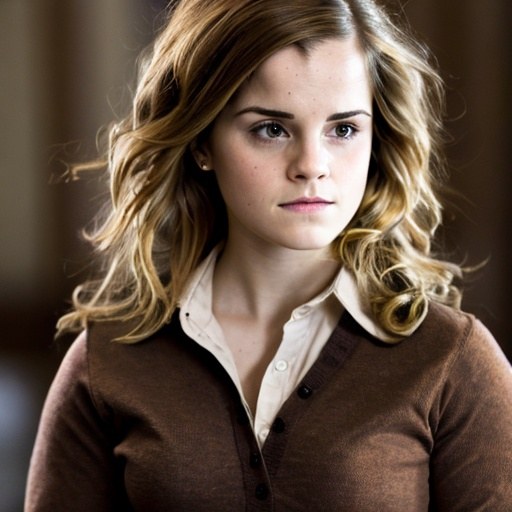}
\includegraphics[width=0.12\textwidth]{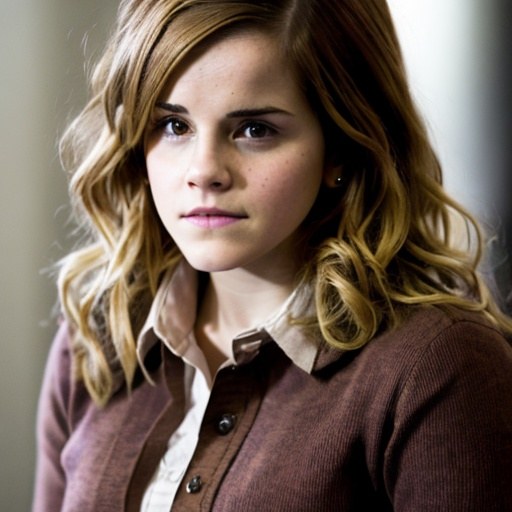}
\includegraphics[width=0.12\textwidth]{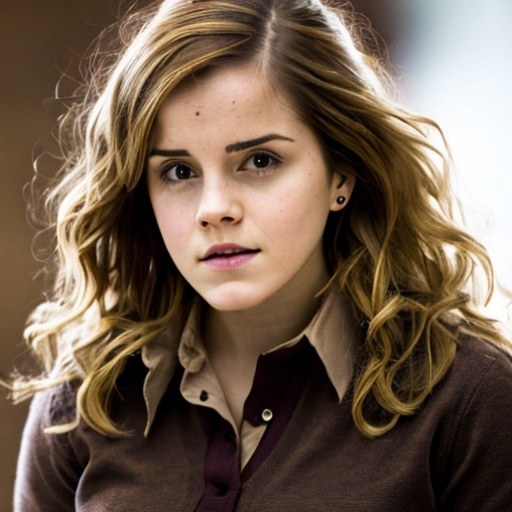}
\includegraphics[width=0.12\textwidth]{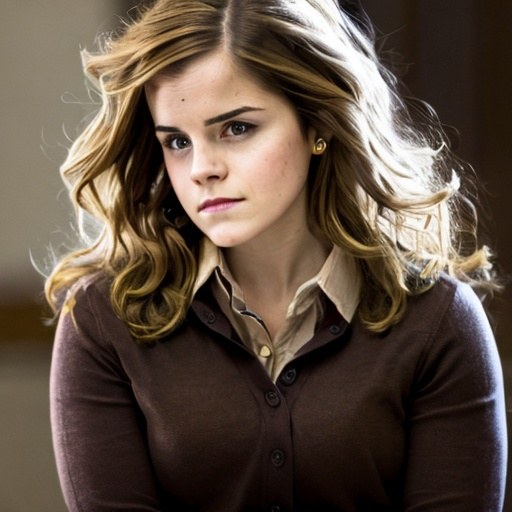}
\caption{Generated results based on the prompt ``Hermione Granger wearing a brown shirt'' when fine-tuned using SGD-based optimizers. All optimizers employed a LoRA scaling factor of 1.0, with the best learning rate. Results demonstrate that the model trained with FILO-Adafactor generates higher-quality images than others, especially the face of Hermione Granger.}\label{fig:hermione_sgd_shirt_1.0}
\end{figure}

\begin{figure}[htbp]
\centering
AdamW with s=0.7 \par
\includegraphics[width=0.12\textwidth]{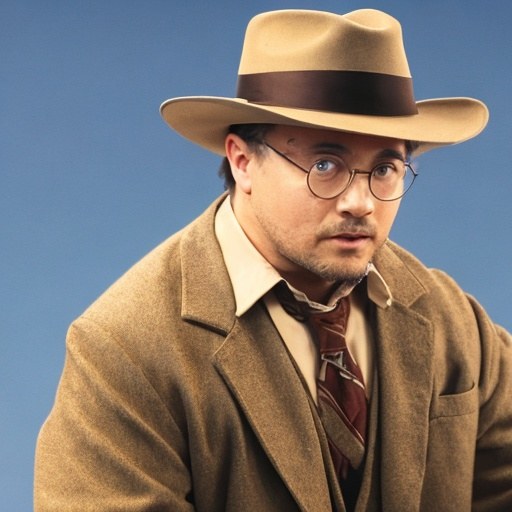}
\includegraphics[width=0.12\textwidth]{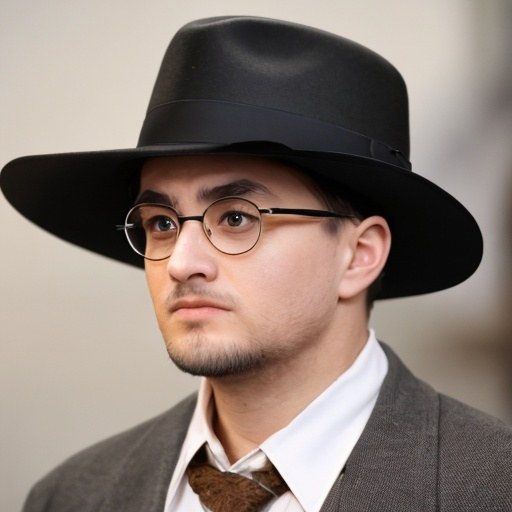}
\includegraphics[width=0.12\textwidth]{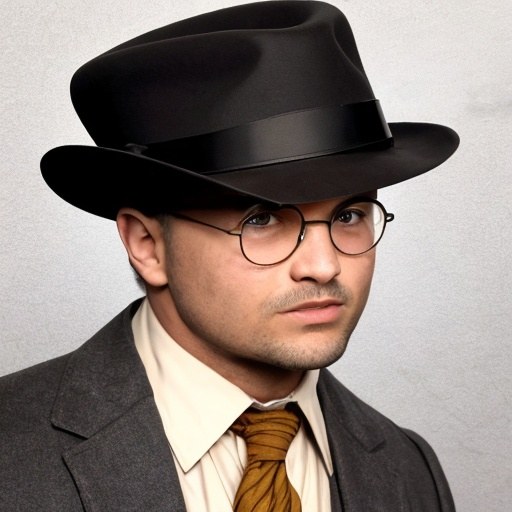}
\includegraphics[width=0.12\textwidth]{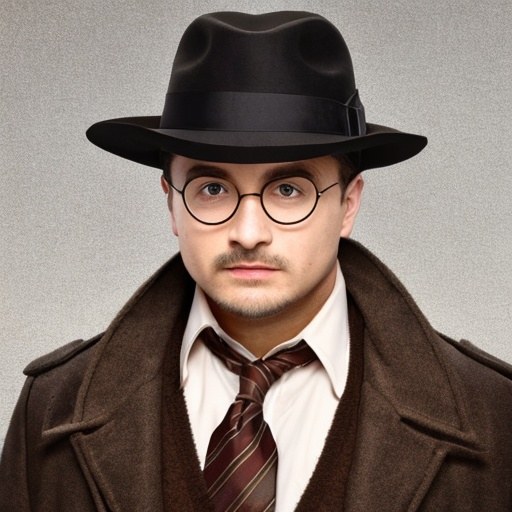}
\includegraphics[width=0.12\textwidth]{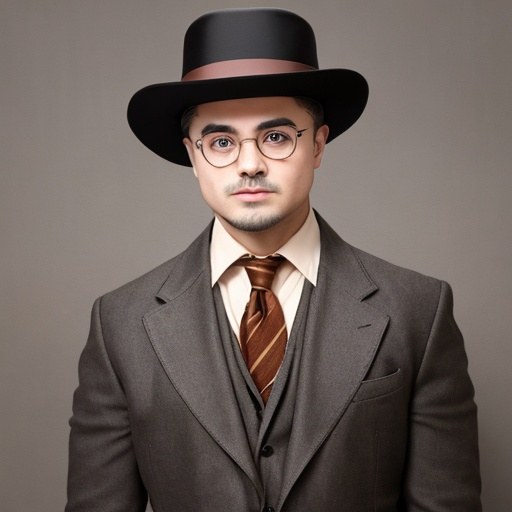} \\
Scaled AdamW with s=0.7 \par
\includegraphics[width=0.12\textwidth]{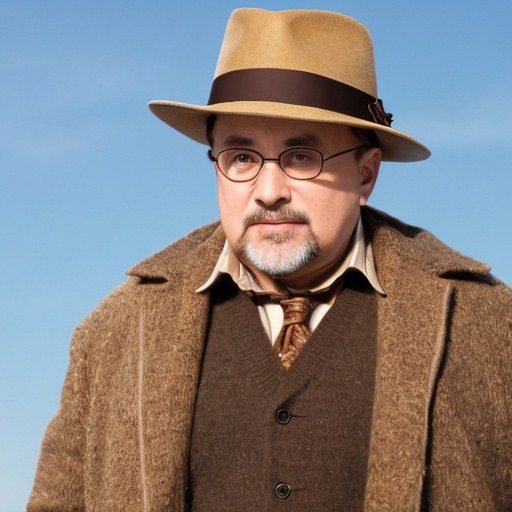}
\includegraphics[width=0.12\textwidth]{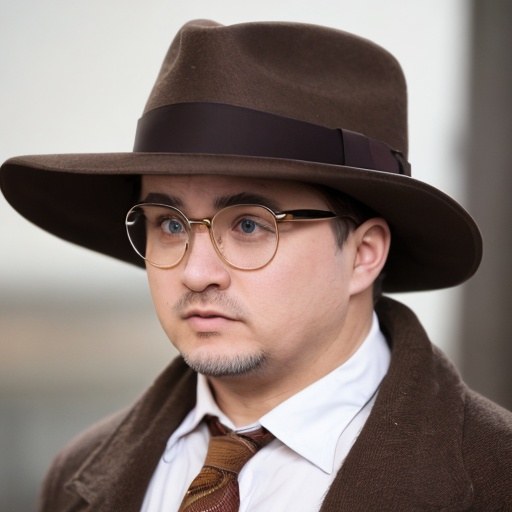}
\includegraphics[width=0.12\textwidth]{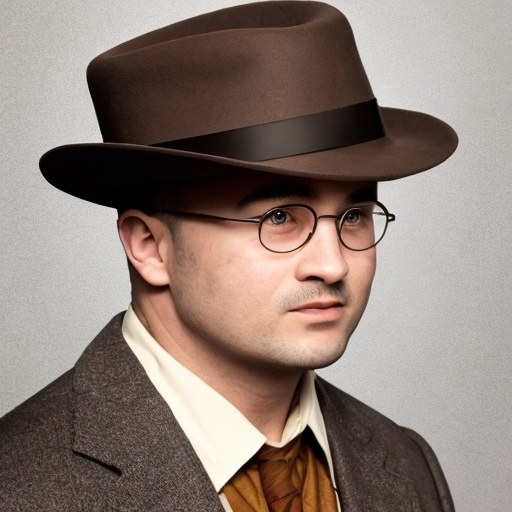}
\includegraphics[width=0.12\textwidth]{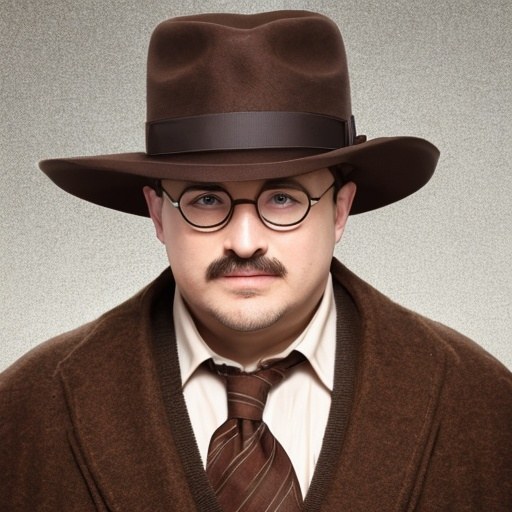}
\includegraphics[width=0.12\textwidth]{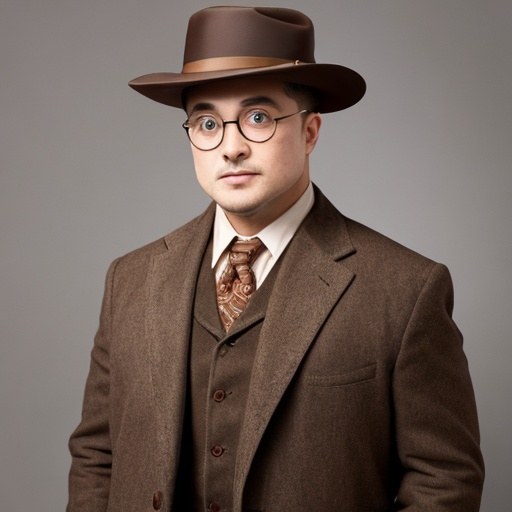} \\
LoRA-Pro AdamW with s=0.7 \par
\includegraphics[width=0.12\textwidth]{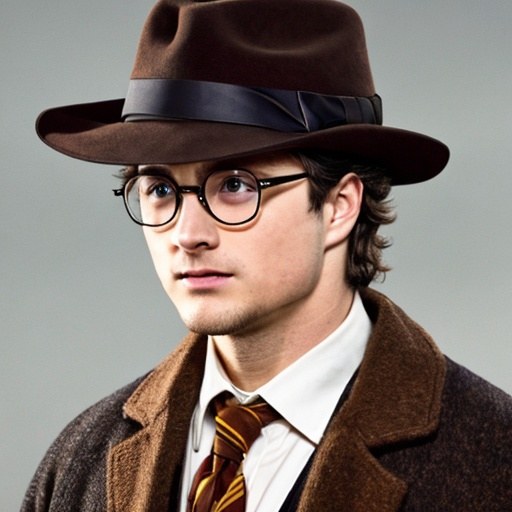}
\includegraphics[width=0.12\textwidth]{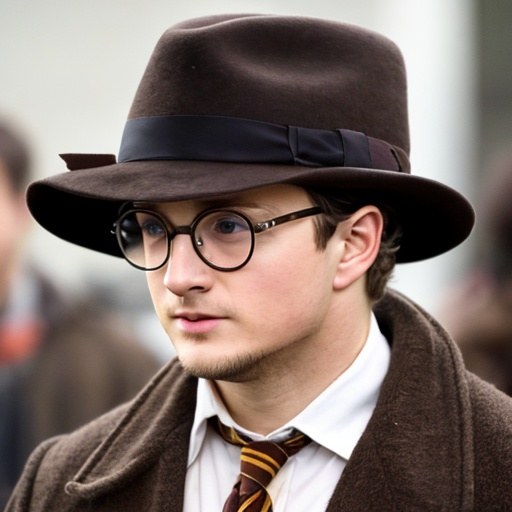}
\includegraphics[width=0.12\textwidth]{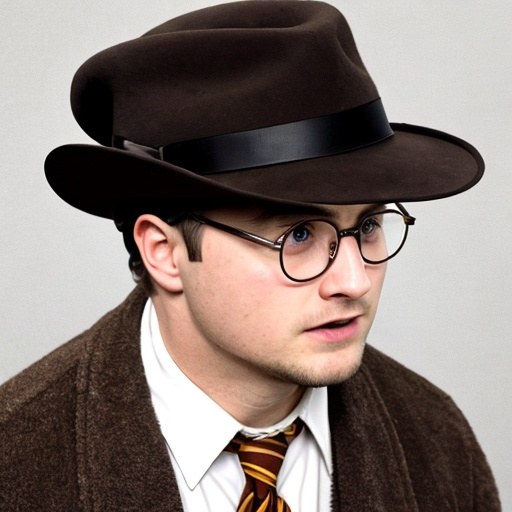}
\includegraphics[width=0.12\textwidth]{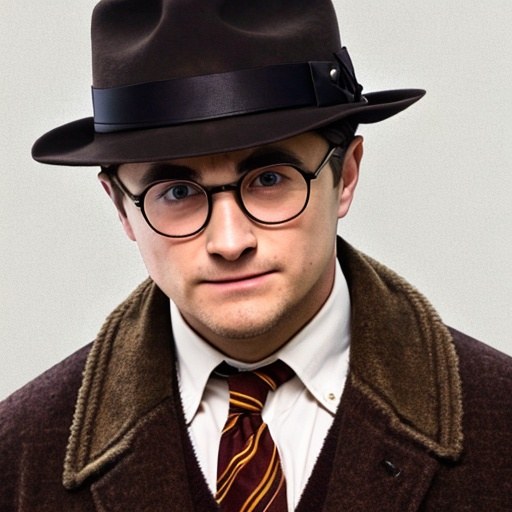}
\includegraphics[width=0.12\textwidth]{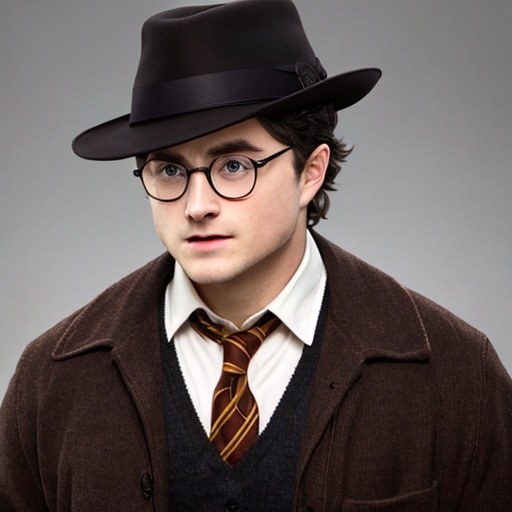} \\
\textbf{FILO-AdafactorM with s=0.7 (ours)} \par
\includegraphics[width=0.12\textwidth]{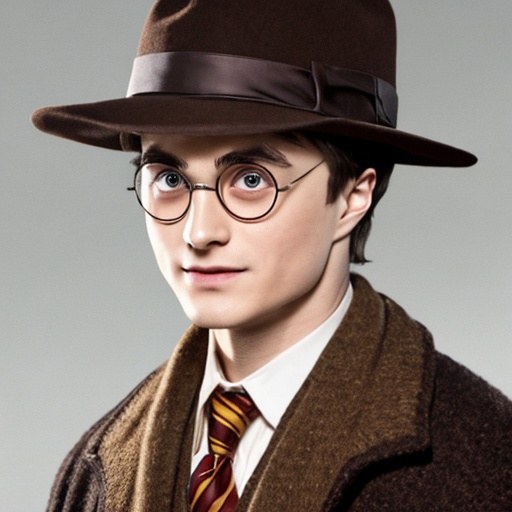}
\includegraphics[width=0.12\textwidth]{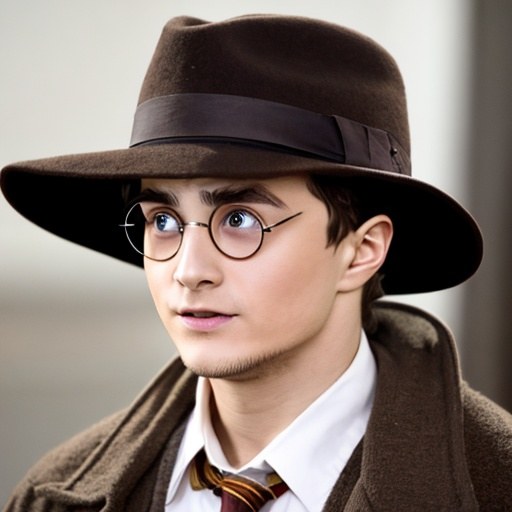}
\includegraphics[width=0.12\textwidth]{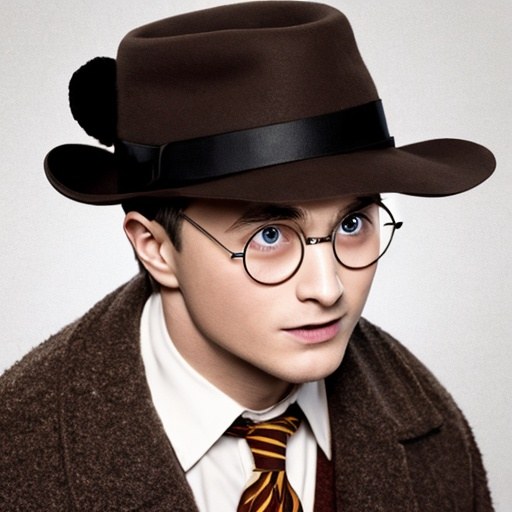}
\includegraphics[width=0.12\textwidth]{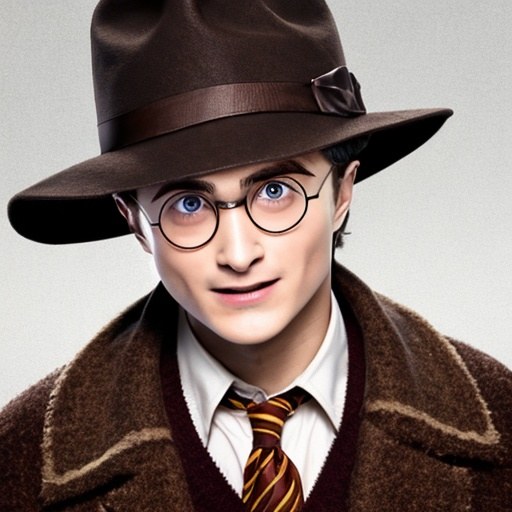}
\includegraphics[width=0.12\textwidth]{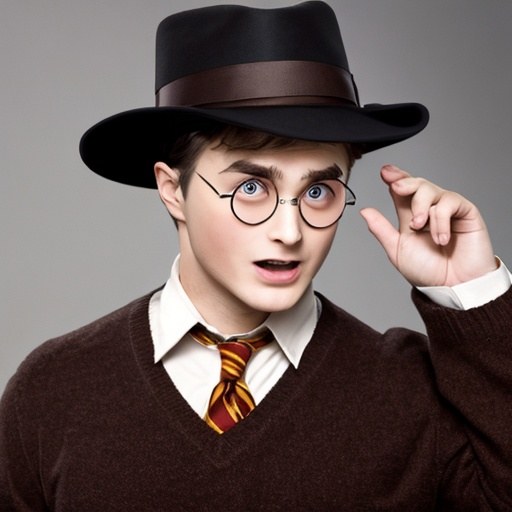}
\caption{Generated results based on the prompt ``Harry Potter wearing a brown hat'' when fine-tuned using AdamW-based optimizers. All optimizers employed a LoRA scaling factor of 0.7, with the best learning rate. The results indicate that the output of the model trained with FILO-Adafactor incorporates the character ``Harry Potter'', and the ``hat'', yielding superior quality compared to alternative approaches.}
\label{fig:potter_adamw_hat_0.7}
\end{figure}

\begin{figure}[htbp]
\centering
AdamW with s=0.7 \par
\includegraphics[width=0.12\textwidth]{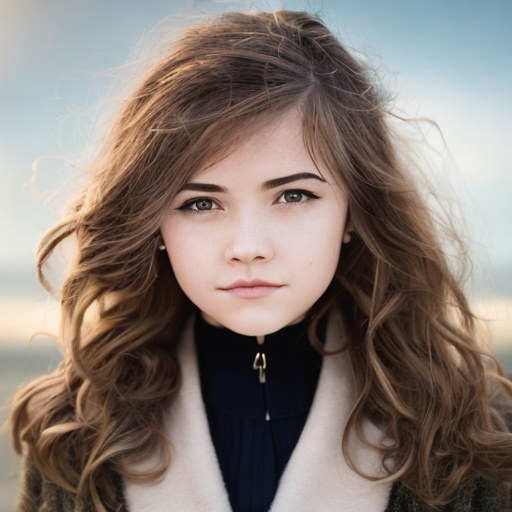}
\includegraphics[width=0.12\textwidth]{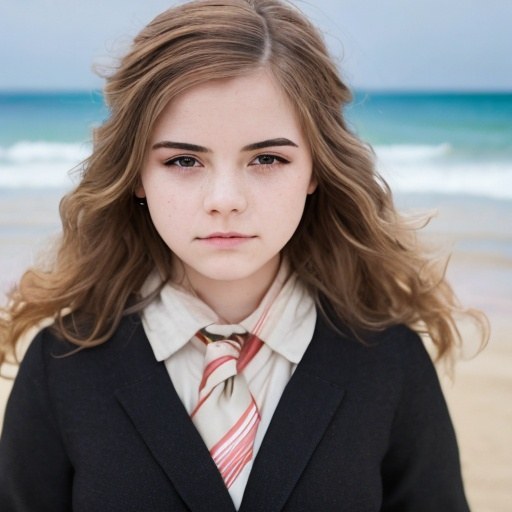}
\includegraphics[width=0.12\textwidth]{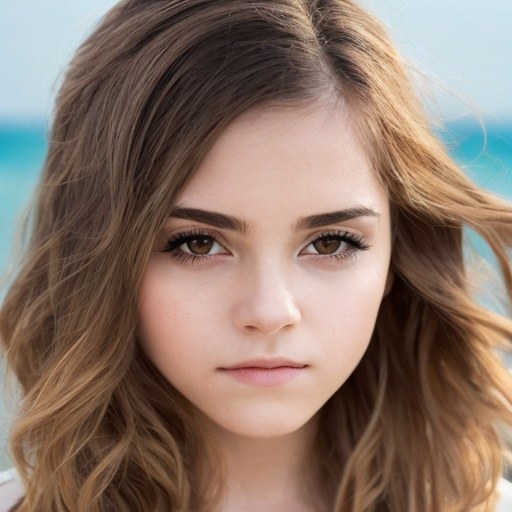}
\includegraphics[width=0.12\textwidth]{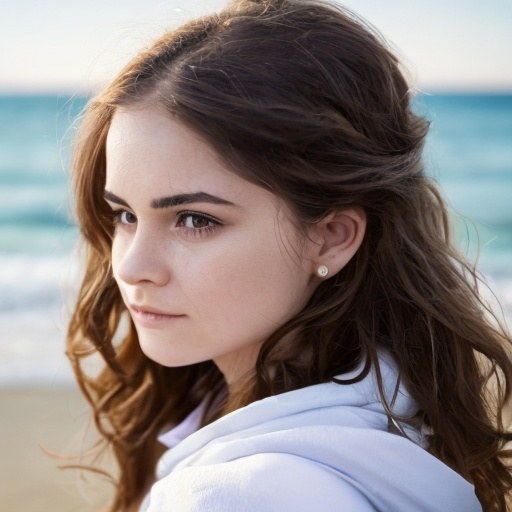}
\includegraphics[width=0.12\textwidth]{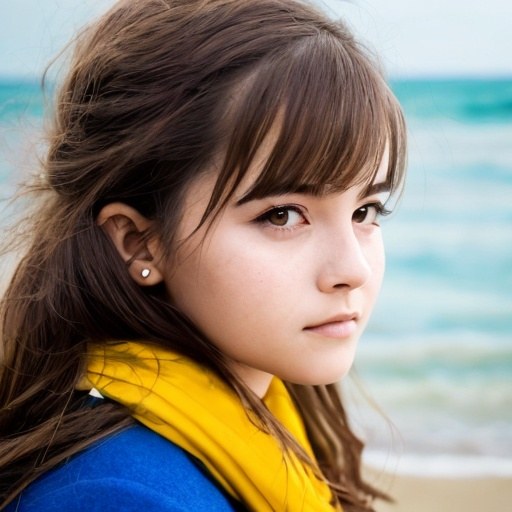} \\
Scaled AdamW with s=0.7 \par
\includegraphics[width=0.12\textwidth]{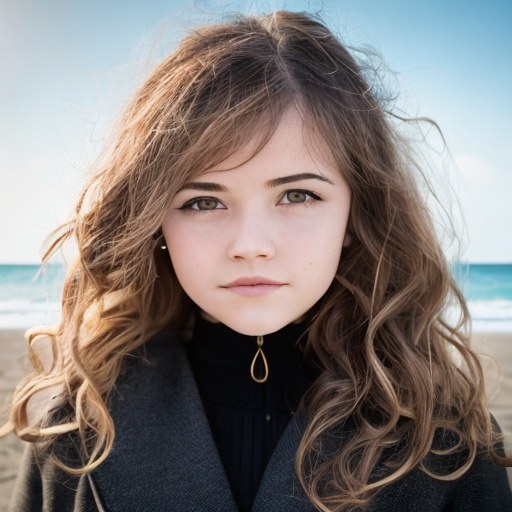}
\includegraphics[width=0.12\textwidth]{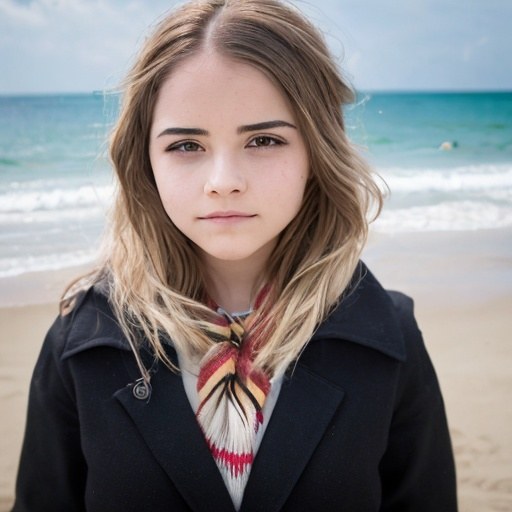}
\includegraphics[width=0.12\textwidth]{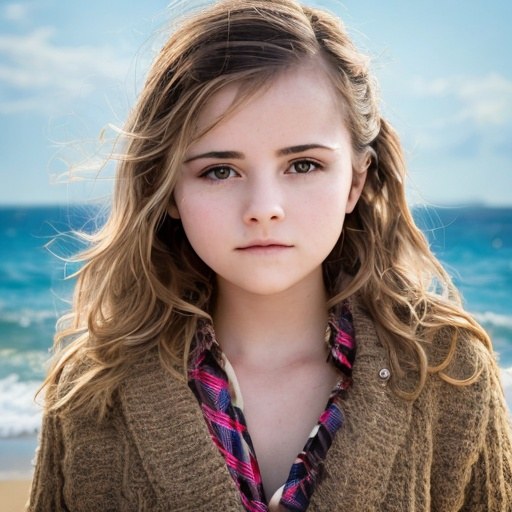}
\includegraphics[width=0.12\textwidth]{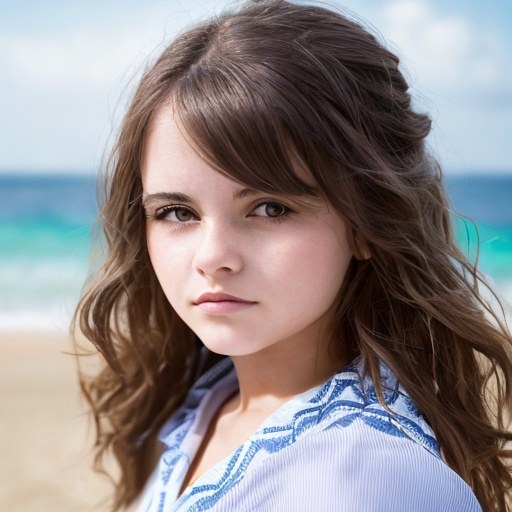}
\includegraphics[width=0.12\textwidth]{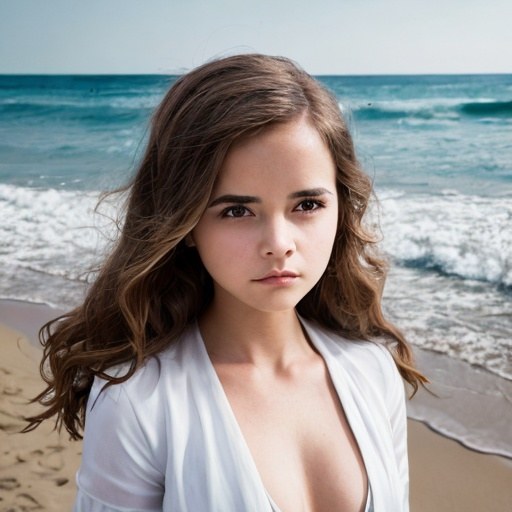} \\
LoRA-Pro AdamW with s=0.7 \par
\includegraphics[width=0.12\textwidth]{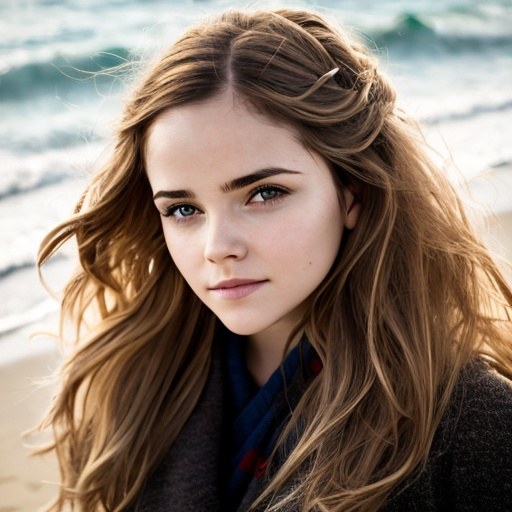}
\includegraphics[width=0.12\textwidth]{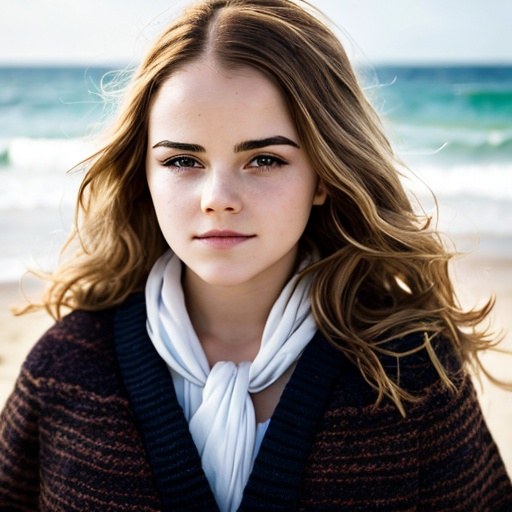}
\includegraphics[width=0.12\textwidth]{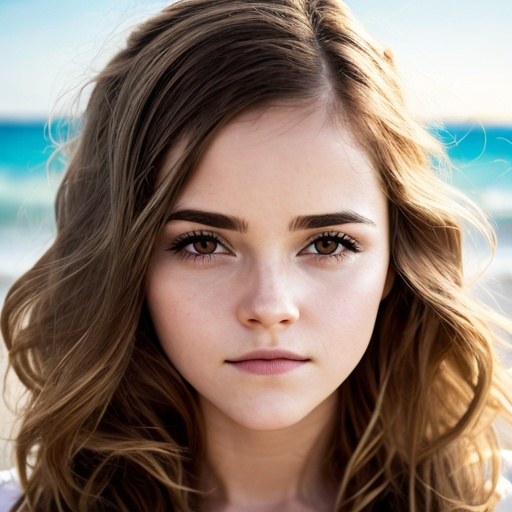}
\includegraphics[width=0.12\textwidth]{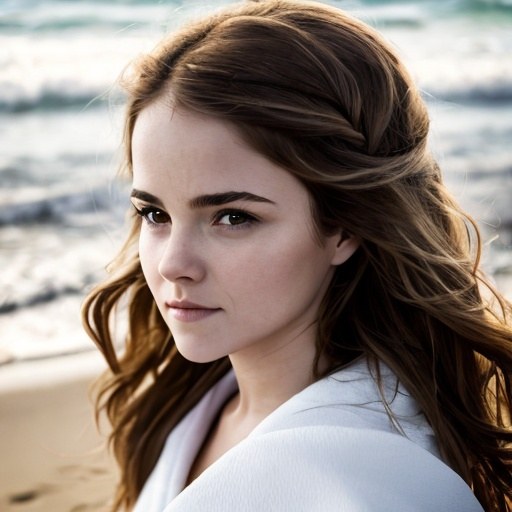}
\includegraphics[width=0.12\textwidth]{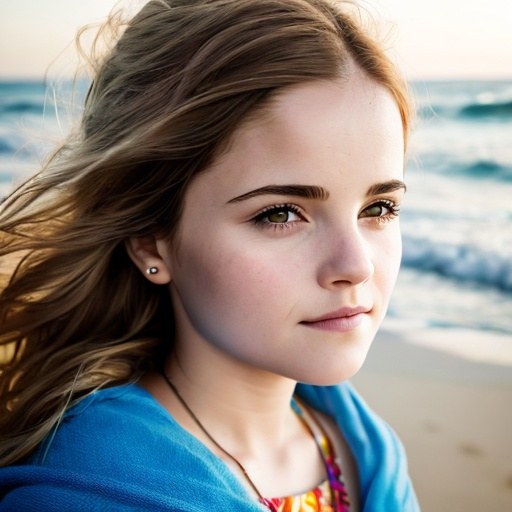} \\
\textbf{FILO-AdafactorM with s=0.7 (ours)} \par
\includegraphics[width=0.12\textwidth]{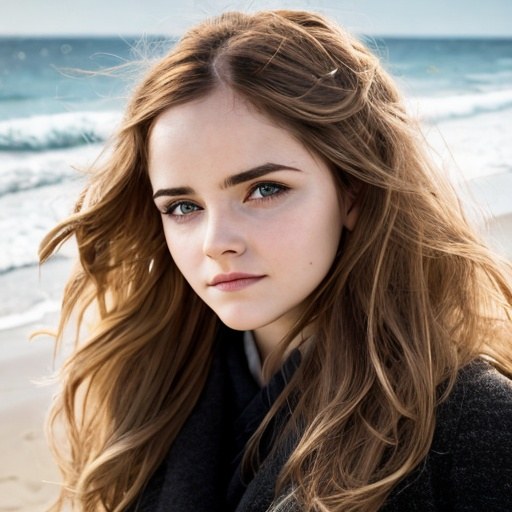}
\includegraphics[width=0.12\textwidth]{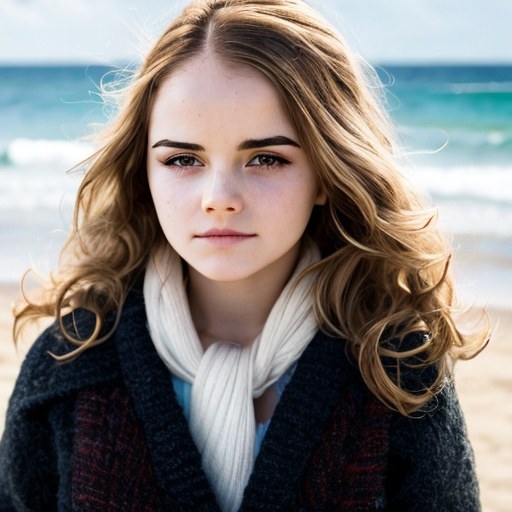}
\includegraphics[width=0.12\textwidth]{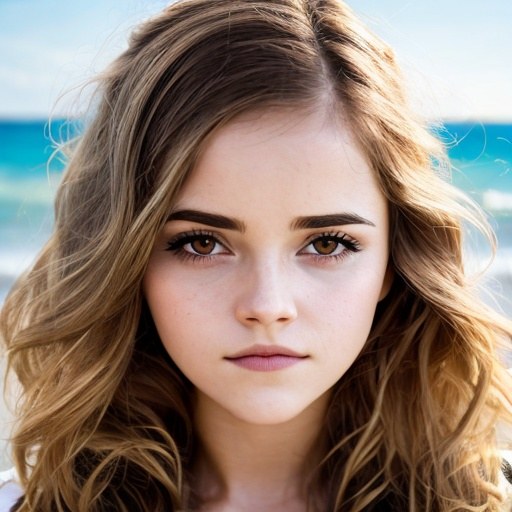}
\includegraphics[width=0.12\textwidth]{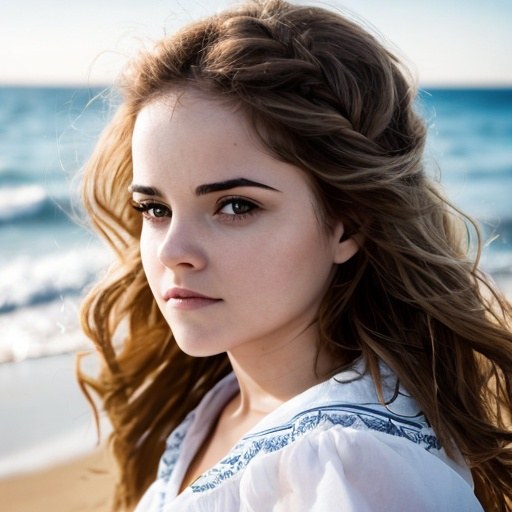}
\includegraphics[width=0.12\textwidth]{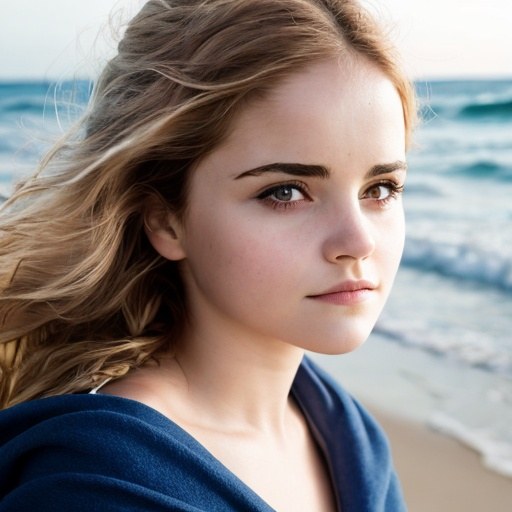}
\caption{Generation results from the prompt ``A photo of Hermione Granger on the beach, small waves, detailed symmetric face, beautiful composition'' using AdamW-based optimizers. All the optimizers apply LoRA scaling factor as 0.7. Results demonstrate that FILO-Adafactor generates higher-quality images for both scaling factors than others, including the face of Hermione Granger and the scene.}\label{fig:hermione_adamw_sea_0.7}
\end{figure}

\end{document}